\documentclass{article}
\pdfoutput=1

\usepackage[square,numbers]{natbib}

\usepackage[final]{neurips_2019}
\usepackage{capt-of}
\usepackage[textsize=tiny]{todonotes}
\usepackage{wrapfig}

\usepackage[utf8]{inputenc} 
\usepackage[T1]{fontenc}    
\usepackage{hyperref}       
\usepackage{url}            
\usepackage{booktabs}       
\usepackage{amsfonts}       
\usepackage{nicefrac}       
\usepackage{microtype}      

\usepackage{graphicx}
\usepackage{subcaption}
\usepackage{booktabs} 

\usepackage{amsmath}
\usepackage{amssymb}
\usepackage{amsthm}
\usepackage{stmaryrd}
\usepackage{xcolor}

\usepackage[english]{babel} 
\usepackage{flexisym}

\usepackage{enumitem}

\newtheorem{defn}{Definition}[section]
\newtheorem{prop}{Proposition}[section]
\newtheorem{thm}{Theorem}[section]

\newtheorem*{prv}{Proof}

\newtheorem{theorem}{Theorem}
\newtheorem{lemma}{Lemma}
\newtheorem{definition}{Definition}
\newtheorem{propo}{Proposition}
\newtheorem{corol}{Corollary}

\usepackage{algorithm}
\usepackage{algorithmic}

\usepackage{etoolbox}
\newtoggle{supplementary}
\toggletrue{supplementary}

\usepackage[toc,page,header]{appendix}
\usepackage{minitoc}

\renewcommand \thepart{}
\renewcommand \partname{}

\title{Comparing distributions: $\ell_1$ geometry improves kernel two-sample testing}

\author{%
  Meyer Scetbon \\
  CREST, ENSAE \&
  Inria, Université Paris-Saclay \\
   \And
   Gaël Varoquaux\\
   Inria, Université Paris-Saclay \\
}

\begin{document}
\doparttoc 
\faketableofcontents 

\maketitle

\begin{abstract}
Are two sets of observations drawn from the same distribution? This
problem is a two-sample test. 
Kernel methods lead to many appealing properties. Indeed state-of-the-art
approaches use the $L^2$ distance between kernel-based
distribution representatives to derive their test statistics. Here, we show that
$L^p$ distances (with $p\geq 1$) between these
distribution representatives give metrics on the space of distributions that are
well-behaved to detect differences between distributions as they
metrize the weak convergence. Moreover, for analytic kernels,
we show that the $L^1$ geometry gives improved testing power for
scalable computational procedures. Specifically, we derive a finite
dimensional approximation of the metric given as the $\ell_1$ norm of a vector which captures differences of expectations of analytic functions evaluated at spatial locations or frequencies (i.e, features). The features can be chosen to
maximize the differences of the distributions and give interpretable
indications of how they differs. Using an $\ell_1$ norm gives better detection
because differences between representatives are dense
as we use analytic kernels (non-zero almost everywhere). The tests are consistent, while
much faster than state-of-the-art quadratic-time kernel-based tests. Experiments
on artificial
and real-world problems demonstrate
improved power/time tradeoff than the state of the art, based on
$\ell_2$ norms, and in some cases, better outright power than even the most
expensive quadratic-time tests. 
\end{abstract}

\faketableofcontents 

\renewcommand \thepart{}
\renewcommand \partname{}

We consider two sample tests:
testing whether two random variables are identically distributed without
assumption on their distributions.
This problem has many applications such as data integration \cite{bio} or automated model
checking \cite{checking}. 
Distances between distributions underlie progress in unsupervised learning
with generative adversarial networks
\cite{li2017mmd,arjovsky2017wasserstein}.
A kernel on the sample space can be used to build the
Maximum Mean Discrepancy (MMD) \cite{MMD,gretton2009fast,MMD-real,energy},
a metric on distribution 
which has the strong propriety of metrizing the weak convergence of
probability measures. It leads to
non-parametric two-sample tests using the reproducing kernel Hilbert space (RKHS) distance
\cite{eric,fromont}, or energy distance \cite{szekely,baringhaus}. 
The MMD has a quadratic computational cost, which may force to use of subsampled
estimates \cite{zaremba,choiceker}. 
\cite{ME} approximate the $L^2$ distance between distribution representatives in the RKHS, to compute in linear time a pseudo metric over the space of distributions. Such approximations are related to random (Fourier)
features, used in kernels algorithms
\cite{randomfeat,randfeat2}.
Distribution representatives can be mean embeddings \cite{freq,Universality} or smooth characteristic
functions \cite{ME, test}.

We first introduce the state of the art on Kernel-based two-sample
testing built from the $L^2$ distance between mean embeddings in
the RKHS. In fact, a wider family of distance is well suited for the
two-sample problem: we show that for any $p\geq 1$, the $L^p$
distance between these distribution representatives is a metric on the
space of Borel probability measures that metrizes their weak convergence. We then define our $\ell_1$-based statistic derived from the $L^1$ geometry and study its asymptotic behavior. We consider the general case
where the number of samples of the two distributions may differ. We show that using the $\ell_1$ norm provides a better testing power. Indeed, test statistics approximate such metrics and are defined as the norm of a $J$-dimensional vector which is the
difference between the two distribution representatives at $J$ locations.
Under the alternative hypothesis $H_1$: $P\neq Q$, the analyticity of the kernel ensures that all the features of this vector are non
zero almost surely. We show that the $\ell_1$ norm captures this dense
difference better than the $\ell_2$ norm and leads to better tests.  
We show also that improvements of Kernel two-sample tests
established with the $\ell_2$ norm \cite{test} hold in the $\ell_1$ case: 
optimizing features and the choice of kernel. We adapt the
construction in the frequency domain as in \cite{ME}. Finally, we show
that on 4 synthetic and 3 real-life problems, our new $\ell_1$-based
tests outperform the state of the art.

\section{Prior art: kernel embeddings for two-sample tests}
Given two samples $X:= \{x_i\}_{i=1}^n$,  $Y:= \{y_i\}_{i=1}^n\subset\mathcal{X}$ independently and identically distributed (i.i.d.) according to two probability measures $P$ and $Q$ on a metric space $(\mathcal{X},d)$ respectively, the goal of a two-sample test is to decide whether $P$ is different from $Q$ on the basis of the samples. 
Kernel methods arise naturally in two-sample testing as they
provide Euclidean norms over the space of probability measures that
metrize the convergence in law. To define such a metric, we need first to introduce the notion of Integral Probability Metric (IPM):
\begin{align}
\text{IPM}[F,P,Q]:=\sup_{f\in F}\biggl(\mathbb{E}_{x\sim
P}\left[f(x)\right]-\mathbb{E}_{y\sim Q}\left[f(y)\right]\biggr)
\end{align}
where $F$ is an arbitrary class of functions. When $F$ is the unit ball $B_k$ in the RKHS $H_k$ associated with a positive definite bounded kernel $k : \mathcal{X} \times \mathcal{X}\rightarrow \mathbb{R}$, the IPM is known as the Maximum Mean Discrepancy (MMD) \cite{MMD}, and it can be shown that the MMD is equal to the RKHS distance between so called mean embeddings \cite{MMD-real},
\begin{equation}
\text{MMD}[P,Q]=\Vert \mu_P-\mu_Q\Vert_{H_k}
\end{equation}
where $\mu_P$ is an embedding of the probability measure $P$ to $H_k$,
\begin{equation}
\mu_P(t):=\int_{\mathbb{R}^d}k(x,t)dP(x) 
\end{equation}
and $\Vert. \Vert_{H_k}$ denotes the norm in the RKHS $H_k$. 
Moreover for kernels said to be \emph{characteristic} \cite{fukumizu2008kernel}, \emph{eg} Gaussian
kernels, $\text{MMD}[P, Q] = 0$ if and only if $P = Q$ \cite{MMD}. In addition, when the kernel is bounded, and $\mathcal{X}$ is a compact Hausdorff space, \cite{weak-cvg} show that the MMD metrizes the weak convergence. Tests between distributions can be designed using an empirical estimation of the MMD.

A drawback of the MMD is the computation
cost of
empirical estimates, these being the sum of two
U-statistics and an empirical average, with a quadratic cost in the sample size.

\cite{ME} study a related expression defined as the $L^2$ distance between mean embeddings of Borel probability measures:
\begin{align}
\label{eq:l2-distance}
d_{L^2,\mu}^2(P,Q):=\int_{t\in\mathbb{R}^d }\biggl|\mu_P(t)-\mu_Q(t)\biggr|^2 d\Gamma(t)
\end{align}
where $\Gamma$ is a Borel probability measure. They estimate the integral (\ref{eq:l2-distance}) with the random variable,
\begin{equation}
d_{\ell_2,\mu,J}^2(P,Q):=\frac{1}{J}\sum_{j=1}^{J}\biggl|\mu_P(T_j)-\mu_Q(T_j)\biggr|^2
\label{eqn:tj_mmd}
\end{equation}
where $\{T_j\}_{j=1}^J$ are sampled i.i.d. from the distribution $\Gamma$. This expression still has desirable metric-like properties, provided that the kernel is \emph{analytic}:
\begin{defn}[Analytic kernel]
A positive definite kernel $k : \mathbb{R}^d
\times\mathbb{R}^d\rightarrow\mathbb{R}$ is \emph{analytic} on its domain if for all $x\in\mathbb{R}^d$, the feature map $k(x, .)$ is an analytic function on $\mathbb{R}^d$. 
\end{defn}
Indeed, for $k$ a definite positive, characteristic, analytic, and bounded
kernel on $\mathbb{R}^d$, \cite{ME} show that
$d_{\ell_2,\mu,J}$
is a random metric\footnote{A random metric is a random process which satisfies all the conditions for a metric ‘almost
surely’ \cite{ME}.} from which consistent two-sample test can be derived. By denoting $\mu_X$ and $\mu_Y$ respectively the empirical mean embeddings of $P$ and $Q$,
\begin{equation*}
\mu_X(T):=\frac{1}{n}\sum_{i=1}^{n} k(x_i,T),\qquad\quad
\mu_Y(T):=\frac{1}{n}\sum_{i=1}^{n} k(y_i,T)
\end{equation*}
\cite{ME} show that for $\{T_j\}_{j=1}^J$ sampled from the distribution $\Gamma$, under the null hypothesis $H_0:$ $P=Q$, as $n\rightarrow\infty$, the following test statistic:
\begin{align}
\label{eq:ME-stat}
\widehat{d}^2_{\ell_2,\mu,J}[X,Y]:=n\sum_{j=1}^{J}\biggl|\mu_X(T_j)-\mu_Y(T_j)\biggr|^2
\end{align}
converges in distribution to a sum of correlated chi-squared variables. Moreover, under the alternative hypothesis $H_1:$ $P\neq Q$, $\widehat{d}^2_{\ell_2,\mu,J}[X,Y]$ can be arbitrarly large as $n\rightarrow \infty$, allowing the test to correctly reject $H_0$. For a fixed level $\alpha$, the test rejects $H_0$ if  $\widehat{d}^2_{\ell_2,\mu,J}[X,Y]$ exceeds a predetermined test threshold, which is given by the $(1 - \alpha)$-quantile of its asymptotic null distribution. As it is very computationally costly to obtain quantiles of this distribution, \cite{ME} normalize the differences between mean embeddings, and consider instead the test statistic
ME[X,Y]:=$\Vert \sqrt{n}\mathbf{\Sigma}_n^{-1/2}\mathbf{S}_n\Vert_2^2$
where $\mathbf{S}_n:=\frac{1}{n}\sum_{i=1}^n \mathbf{z}_i$, $\mathbf{\Sigma}_n :=\frac{1}{n-1} \sum_{i=1}^n(\mathbf{z}_i -\mathbf{S}_n)(\mathbf{z}_i - \mathbf{S}_n)^{T}$, and $\mathbf{z}_i :=(k(x_i,T_j)-k(y_i,T_j))_{j=1}^J\in\mathbb{R}^J$. Under the null hypothesis $H_0$, asymptotically the ME statistic follows $\chi^2(J)$, a chi-squared distribution with J degrees of freedom. Moreover, for $k$ a translation-invariant
kernel, \cite{ME} derive another statistical test, called the SCF test
(for Smooth Characteristic Function), where its statistic $\text{SCF}[X,Y]$ is of the same form as the ME test statistic with a modified $\mathbf{z}_i:= [f(x_i)\sin(x_i^{T} T_j) - f(y_i)\sin(y_i^{T} T_j), f(x_i)\cos(x_i^{T} T_j) - f(y_i)\cos(y_i^{T} T_j)]_{j=1}^J\in\mathbb{R}^{2J}$ where $f$ is the inverse Fourier transform of $k$, and show that under $H_0$, $\text{SCF}[X,Y]$ follows asymptotically $\chi^2(2J)$.

\section{A family of metrics that metrize of the weak convergence}
\cite{ME} build their ME statistic by estimating the
$L^2$ distance between mean embeddings. This metric can be
generalized 
using any $L^p$ distance with $p\geq 1$. These metrics are well
suited for the two-sample problem as they
metrize the weak convergence (see proof in supp. mat. \ref{sec:weak_cvg_ME}):
\begin{thm}
\label{thm:metricl1}
Given $p\geq 1$, $k$ a definite positive, characteristic, continuous, and bounded
kernel on $\mathbb{R}^d$,
$\mu_P$ and $\mu_Q$ the mean embeddings of the Borel probability measures $P$ and
$Q$ respectively, 
the function defined on $\mathcal{M}_{+}^{1}(\mathbb{R}^d)\times \mathcal{M}_{+}^{1}(\mathbb{R}^d)$:
\begin{equation}
\label{eq:l1metric}
d_{L^p,\mu}(P,Q):=\left(\int_{t\in\mathbb{R}^d }\biggl|\mu_P(t)-\mu_Q(t)\biggr|^p d\Gamma(t)\right)^{1/p}
\end{equation}
is a metric on the space of Borel probability measures, for $\Gamma$ a Borel probability measure absolutely continuous with
respect to Lebesgue measure. Moreover a sequence $(\alpha_n)_{n\geq 0}$ of Borel probability measures converges weakly towards $\alpha$ if and only if  $d_{L^p,\mu}(\alpha_n,\alpha)\rightarrow 0$.
\end{thm}
Therefore, as the MMD, these metrics take into account the geometry
of the underlying space and metrize the convergence in law. If we assume in addition that the kernel is analytic, we will show that
deriving test statistics from the $L^1$ distance instead of the $L^2$ distance improves the test power for two-sample testing.

\section{Two-sample testing using the $\ell_1$ norm}

\subsection{A test statistic with simple asymptotic distribution}
 From now, we assume that $k$ is a positive definite, characteristic, analytic, and bounded kernel. 
 
 The statistic presented in eq. \ref{eq:ME-stat} is based on the $\ell_2$ norm of a vector that capture differences between distributions in the RKHS at J locations. We will show that using an
$\ell_1$ norm instead of an $\ell_2$ norm improves the test power
(Proposition \ref{prop:finalement}). It captures better the geometry of
the problem. Indeed, when $P\neq Q$, the differences between
distributions are dense which allow the $\ell_1$ norm to reject better the null hypothesis $H_0$: $P=Q$. 

We now build a consistent statistical test based on an empirical estimation of the $L^1$ metric introduced in eq. \ref{eq:l1metric}:
\begin{align}
\label{l_1:comparison}
\widehat{d}_{\ell_1,\mu,J}[X,Y]:=\sqrt{n}\sum_{j=1}^{J}\biggl|\mu_X(T_j)-\mu_Y(T_j)\biggr|
\end{align}
where  $\{T_j\}_{j=1}^J$ are sampled from the distribution $\Gamma$. We show that under $H_0$, $\widehat{d}_{\ell_1,\mu,J}[X,Y]$ converges in distribution to a sum of correlated Nakagami
variables\footnote{the pdf of the Nakagami distribution of parameters $m\geq \frac{1}{2}$ and $\omega>0$ is  $\forall x\geq 0$,\\ $f(x,m,\omega)=\frac{2m^m }{\Gamma(m)\omega^m}x^{2m-1}\exp(\frac{-m}{\omega}x^2)$ where $\Gamma$ is the Gamma function.} and under 
$H_1$, $\widehat{d}_{\ell_1,\mu,J}[X,Y]$ can be arbitrary large as $n\rightarrow \infty$ (see supp. mat. \ref{sec:naka-coro}). For a fixed level $\alpha$, the test rejects $H_0$ if  $\widehat{d}_{\ell_1,\mu,J}[X,Y]$ exceeds the $(1 - \alpha)$-quantile of its asymptotic null distribution.  We now compare the power of the statistics based respectively on the $\ell_2$ norm (eq. \ref{eq:ME-stat}) and the $\ell_1$ norm (eq. \ref{l_1:comparison}) at the same level $\alpha>0$ and we show that the power of the test using the $\ell_1$ norm is better with
high probability (see supp. mat. \ref{sec:finalement_ME}):
\begin{prop}
\label{prop:finalement}
Let $\alpha\in ]0,1[$, $\gamma>0$ and $J\geq 2$. Let $\{T_j\}_{j=1}^J$
sampled i.i.d. from the distribution $\Gamma$ and let $X:=\{x_i\}_{i=1}^{n}$ and 
$Y:=\{y_i\}_{i=1}^{n}$ i.i.d. samples from $P$ and $Q$ respectively. Let us denote $\delta$ the $(1-\alpha)$-quantile of the asymptotic null distribution of $\widehat{d}_{\ell_1,\mu,J}[X,Y]$ and $\beta$ the $(1-\alpha)$-quantile of the asymptotic null distribution of $\widehat{d}_{\ell_2,\mu,J}^2[X,Y]$. Under the alternative hypothesis, almost surely, there exists $N\geq 1$ such that for all $n\geq N$, with a probability of at least $1-\gamma$ we have:
\begin{align}
\widehat{d}^2_{\ell_2,\mu,J}[X,Y] >\beta \Rightarrow \widehat{d}_{\ell_1,\mu,J}[X,Y]>\delta 
\end{align}
\end{prop}
Therefore, for a fixed level $\alpha$, under the alternative hypothesis, when the number of samples is large enough, with high probability, the $\ell_1$-based test rejects better the null hypothesis. However, even
for fixed $\{T_j\}_{j=1}^J$, computing the quantiles of these distributions
requires a computationally-costly bootstrap or permutation procedure.
Thus we follow a different approach where we allow the number of samples to differ. Let $X:= \{x_i\}_{i=1}^{N_1}$ and $Y:= \{y_i\}_{i=1}^{N_2}$ i.i.d according to respectively $P$ and $Q$. We define for any sequence of $\{T_j\}_{j=1}^J$ in $\mathbb{R}^d$:
\begin{align}
\mathbf{S}_{N_1,N_2}:=
\biggl(\mu_X(T_1)-\mu_Y(T_1),..., \mu_X(T_J)-\mu_Y(T_J)\biggr)
\end{align}
\begin{equation*}
\mathbf{Z}_{X}^{i}:=(k(x_i,T_1),..., k(x_i,T_J)) \in \mathbb{R}^J
\qquad\quad
\mathbf{Z}_{Y}^{j}:=(k(y_j,T_1),..., k(y_j,T_J)) \in \mathbb{R}^J
\end{equation*}
And by denoting:
\begin{align*}
\mathbf{\Sigma}_{N_1} :=
\frac{1}{N_1-1}\sum_{i=1}^{N_1}(\mathbf{Z}_{X}^i-\overline{\mathbf{Z}}_X)(\mathbf{Z}_{X}^i-\overline{\mathbf{Z}}_X)^{T}
&\qquad
\mathbf{\Sigma}_{N_2} :=
\frac{1}{N_2-1}\sum_{j=1}^{N_2}(\mathbf{Z}_{Y}^j-\overline{\mathbf{Z}}_Y)(\mathbf{Z}_{Y}^j-\overline{\mathbf{Z}}_Y)^{T}\\
\mathbf{\Sigma}_{N_1,N_2} :=
&\frac{\mathbf{\Sigma}_{N_1}}{\rho}+\frac{\mathbf{\Sigma}_{N_2}}{1-\rho}
\end{align*}
We can define our new statistic as:
\begin{align}
\label{eq:asym_me}
\text{L1-ME}[X,Y]
:=\left\|\sqrt{t}\mathbf{\Sigma}_{N_1,N_2}^{-\frac{1}{2}}\mathbf{S}_{N_1,N_2}\right\|_{1}
\end{align}
We assume that the number of samples of the distributions $P$ and $Q$ are
of the same order, i.e: let $t=N_1+N_2$, we have:
$\frac{N_1}{t}\rightarrow \rho$ and therefore $\frac{N_2}{t}\rightarrow 1-\rho$ with $\rho \in ]0,1[.$
The computation of the statistic requires inverting a $J\times J$
matrix $\mathbf{\Sigma}_{N_1,N_2}$, but this is fast and numerically
stable: $J$ is typically be small, \emph{eg} less than 10. The next proposition demonstrates the use of this statistic as a consistent two-sample test (see supp. mat. \ref{sec:prop_asymp} for the proof).
\begin{prop}
\label{prop:asympme}
Let $\{T_j\}_{j=1}^J$ sampled i.i.d. from the distribution $\Gamma$ and $X:=\{x_i\}_{i=1}^{N_1}$ and 
$Y:=\{y_i\}_{i=1}^{N_2}$ be i.i.d. samples from $P$ and $Q$ respectively. Under $H_0$, the statistic $\text{L1-ME}[X,Y]$
is almost surely asymptotically distributed as Naka$(\frac{1}{2},1,J)$, a sum of $J$ random variables i.i.d which follow a Nakagami distribution of parameter $m=\frac{1}{2}$ and $\omega=1$.
Finally under $H_1$, almost surely the statistic can be arbitrarily large
as $t \rightarrow \infty$, enabling the test to correctly reject $H_0$.
\end{prop}

\textbf{Statistical test of level $\alpha$:}
Compute
$\Vert\sqrt{t}\mathbf{\Sigma}_{N_1,N_2}^{-\frac{1}{2}}\mathbf{S}_{N_1,N_2}\Vert_{1}$,
choose the threshold $\delta$ corresponding to the $(1-\alpha)$-quantile
of  Naka($\frac{1}{2},1,J$), and reject the null hypothesis
whenever $\Vert\sqrt{t}\mathbf{\Sigma}_{N_1,N_2}^{-\frac{1}{2}}\mathbf{S}_{N_1,N_2}\Vert_{1}$ is larger than $\delta$.

\subsection{Optimizing test locations to improve power}
As in \cite{test}, we can optimize the test locations $\mathcal{V}$ and
kernel parameters (jointly referred to as $\theta$) by maximizing a lower
bound on the test power which offers a simple objective function for fast
parameter tuning. We make the same regularization as in \cite{test} of the test statistic for stability of the matrix inverse, by adding a regularization parameter $\gamma_{N_1,N_2}>0$ which goes to $0$ as $t$ goes to infinity, giving $\text{L1-ME}[X,Y]
:=\|\sqrt{t}(\mathbf{\Sigma}_{N_1,N_2}+\gamma_{N_1,N_2}\mathbf{I})^{-\frac{1}{2}}\mathbf{S}_{N_1,N_2}\|_{1}$ (see proof in supp. mat. \ref{sec:improve}).
\begin{prop}
\label{prop:optpower}
Let $\mathcal{K}$ be a uniformly bounded family of  $k : \mathbb{R}^d
\times\mathbb{R}^d  \rightarrow \mathbb{R}$ measurable kernels (i.e., $\exists$ $K<\infty$ such
that $\sup \limits_{k\in \mathcal{K}} \sup \limits_{(x,y)\in
(\mathbb{R}^{d})^2} |k(x,y)| \leq K$). Let
$\mathcal{V}$ be a collection in which each element is a set of J test
locations. Assume that 
$c:= \sup \limits_{V\in \mathcal{V},k\in \mathcal{K}}\Vert\mathbf{\Sigma}^{-1/2}\Vert < \infty$. Then the test power $\mathbb{P}\left(\widehat{\lambda}_{t} \geq \delta\right)$ of the L1-ME test satisfies $\mathbb{P}\left(\widehat{\lambda}_{t} \geq \delta\right)\geq L(\lambda_{t})$ where:
\begin{align*}
L(\lambda_{t}) =&  1-2\sum
\limits_{k=1}^J\exp\left(-\left(\frac{\lambda_{t}-\delta}{J^2+J}\right)^2\frac{\gamma_{N_{\!1}\!,N_2}N_1N_2}{(N_{\!1}+N_2)^2}\right)\\
 &-2\!\sum \limits_{k,q=1}^J
\!\exp\!\left(-2\frac{\left(\frac{\gamma_{N_{\!1}\!,N_2}}{K_3J^2}\frac{\lambda_{t}-\delta}{(J^2+J)\sqrt{t}}-\frac{J^3K_2}{\sqrt{\gamma_{N_{\!1}\!,N_2}}}-J^4K_1\right)^{\!2}}{K_{\lambda}^2
(N_1 + N_2)\max\left(\frac{8}{\rho
N_1} ,\frac{8}{(1-\rho) N_ 2}\right)^{\!2}
}\right)
\end{align*}
and $K_1,K_2$, $K_3$ and $K_\lambda$, are  positive constants depending on only $K$,
$J$ and $c$. The parameter
$\lambda_{t}:=\Vert\sqrt{t}\mathbf{\Sigma}^{-\frac{1}{2}}\mathbf{S}\Vert_{1}$ is the
population counterpart of
$\widehat{\lambda}_{t}:=\|\sqrt{t}\mathbf({\Sigma}_{N_1,N_2}+\gamma_{N_1,N_2}\mathbf{I})^{-\frac{1}{2}}\mathbf{S}_{N_1,N_2}\|_{1}$ where $\mathbf{S}=\mathbb{E}_{x,y}(S_{N_1,N_2})$ and $\mathbf{\Sigma}=\mathbb{E}_{x,y}(\mathbf{\Sigma}_{N_1,N_2})$. Moreover for large $t$, $L(\lambda_{t})$ is increasing in $\lambda_{t}$.
\end{prop}

Proposition \ref{prop:optpower} suggests that it is sufficient to maximize $\lambda_t$ to maximize a lower bound on the L1-ME test power. The statistic $\lambda_{t}$ for this test depends on a set of test
locations $\mathcal{V}$ and a kernel parameter $\sigma$. We set
$\theta^{*}:=\{\mathcal{V},\sigma \}=\text{arg}\max \limits_{\theta}
\lambda_{t}=\text{arg}\max \limits_{\theta}
\Vert\sqrt{t}\,\mathbf{\Sigma}^{-\frac{1}{2}}\mathbf{S}\Vert_{1}$. As proposed in
\cite{choiceker}, we can maximize a proxy test power
to optimize $\theta$: it does not affect $H_0$ and $H_1$ as long as the data used for parameter tuning and for testing are disjoint. 

\subsection{Using smooth characteristic functions (SCF)}
As the ME statistic, the SCF statistic estimates the $L^2$
distance between well chosen distribution representatives. Here, the representatives of the distributions are the convolution of their characteristic functions and the kernel $k$, assumed translation-invariant. \cite{ME} use
them to detect differences between distributions in the frequency domain. We show that the $L^1$ version (denoted $d_{L^1,\Phi}$)
is a metric on the space of Borel probability measures with integrable
characteristic functions such that if $\alpha_n$ converge weakly towards
$\alpha$, then $d_{L^1,\Phi}(\alpha_n,\alpha)\rightarrow 0$ (see
supp. mat. \ref{sec:weak_cvg_SCF}). Let us introduce the test statistics in the frequency domain respectively based on the $\ell_2$ norm and on the $\ell_1$ norm which lead to consistent tests:
\begin{align}
\widehat{d}^2_{\ell_2,\Phi,J}[X,Y]:= \Vert\sqrt{n} \mathbf{S}_n\Vert_2^2
\quad
\text{and} 
\quad
\widehat{d}_{\ell_1,\Phi,J}[X,Y]:= \Vert\sqrt{n} \mathbf{S}_n\Vert_1
\end{align}
where $\mathbf{S}_n:=\frac{1}{n}\sum_{i=1}^n \mathbf{z}_i$, 
$\mathbf{z}_i:= [f(x_i)\sin(x_i^{T} T_j) - f(y_i)\sin(y_i^{T} T_j), f(x_i)\cos(x_i^{T} T_j) - f(y_i)\cos(y_i^{T} T_j)]_{j=1}^J\in\mathbb{R}^{2J}$ and $f$ is the inverse Fourier transform of $k$. We show that, at the same level $\alpha$, using the $\ell_1$ norm in the frequency domain provides a better power with high probability (see supp. mat. \ref{sec:finalement-SCF}):
\begin{prop}
\label{prop:finalement-SCF}
Let $\alpha\in ]0,1[$, $\gamma>0$ and $J\geq 2$. Let $\{T_j\}_{j=1}^J$
sampled i.i.d. from the distribution $\Gamma$ and let $X:=\{x_i\}_{i=1}^{n}$ and 
$Y:=\{y_i\}_{i=1}^{n}$ i.i.d. samples from $P$ and $Q$ respectively. Let us denote $\delta$ the $(1-\alpha)$-quantile of the asymptotic null distribution of $\widehat{d}_{\ell_1,\Phi,J}[X,Y]$ and $\beta$ the $(1-\alpha)$-quantile of the asymptotic null distribution of $\widehat{d}_{\ell_2,\Phi,J}^2[X,Y]$. Under the alternative hypothesis, almost surely, there exists $N\geq 1$ such that for all $n\geq N$, with a probability of at least $1-\gamma$ we have:
\begin{align}
\widehat{d}^2_{\ell_2,\Phi,J}[X,Y] >\beta \Rightarrow \widehat{d}_{\ell_1,\Phi,J}[X,Y]>\delta 
\end{align}
\end{prop}


We now adapt the construction of the $\text{L1-ME}$ test to the frequency domain to avoid computational issues of the quantiles of the asymptotic null distribution:
\begin{align}
\label{eq:SCF-test}
\text{L1-SCF}[X,Y]:=\Vert\sqrt{t}\,\mathbf{\Sigma}_{N_1,N_2}^{-\frac{1}{2}}\mathbf{S}_{N_1,N_2}\Vert_{1}
\end{align}
with $\mathbf{\Sigma}_{N_1,N_2}$, and $\mathbf{S}_{N_1,N_2}$ defined as
in the L1-ME statistic with new expression for
$\mathbf{Z}_{X}^{i}$ (and $\mathbf{Z}_{Y}^{j}$):
\begin{align*}
\mathbf{Z}_{X}^{i}=&\left(cos\left(T_{1}^Tx_i\right)f(x_i),..., sin\left(T_{J}^Tx_i\right)f(x_i)\right)\in \mathbb{R}^{2J}
\end{align*}
From this statistic,  we build a consistent test. Indeed, an analogue proof of the Proposition \ref{prop:asympme} gives that under $H_0$,  $\text{L1-SCF}[X,Y]$ is a.s. asymptotically distributed as Naka($\frac{1}{2},1,2J$), and under $H_1$, the test statistic can be arbitrarily large as t goes to infinity. Finally an analogue proof of Proposition \ref{prop:optpower} shows that we can optimize the test locations and the kernel parameter to improve the power as well.

\section{Experimental study}

\label{sec:experimental_study}

We now run empirical comparisons of our 
$\ell_1$-based tests to their
$\ell_2$ counterparts, state-of-the-art Kernel-based two-sample tests.
We study both toy and real problems.
We use the isotropic Gaussian kernel class $\mathcal{K}_g$.
We call \textbf{L1-opt-ME} and \textbf{L1-opt-SCF} the tests based
respectively on mean embeddings and smooth characteristic functions
proposed in this paper when optimizing
test locations and the Gaussian width $\sigma$ on
a separate training
set of the same size as the test set.  We denote
also \textbf{L1-grid-ME} and \textbf{L1-grid-SCF} where only the Gaussian
width is optimized by a grid search, and locations are randomly
drawn from a multivariate normal distribution. We write \textbf{ME-full}
and \textbf{SCF-full} for the tests of \cite{test}, also fully optimized according to their criteria. \textbf{MMD-quad}
(quadratic-time) and \textbf{MMD-lin} (linear-time) refer to the
MMD-based tests of \cite{MMD},
where, to ensure a fair comparison, the kernel width is also set
to maximize the test power following
\cite{choiceker}. For \textbf{MMD-quad}, as its null distribution is
an infinite sum of weighted chi-squared variables (no
closed-form quantiles), we approximate the null
distribution with 200 random permutations in each trial.

In all the following experiments, we repeat each problem 500
times. For synthetic problems, we generate new samples from the
specified $P$, $Q$ distributions in each trial. For the first real
problem (Higgs dataset), as the dataset is big enough we use new
samples  from the two distributions for each trial. For the second and
third real problem (Fast food and text datasets), samples
are split randomly into train and test sets in each trial.
In all the simulations we report an empirical estimate of the Type-I
error when $H_0$ hold and of the Type-II error when $H_1$ hold. We set
$\alpha = 0.01$. The code is available at \href{https://github.com/meyerscetbon/L1_test}{https://github.com/meyerscetbon/l1_two_sample_test}.

\textbf{How to realize $\ell_1$-based tests ?}
The asymptotic distributions of the statistics is a sum of i.i.d.
Nakagami distribution. \cite{naka} give a
closed form for the probability density function. As the formula is not
simple, we can also derive an estimate of the CDF (see supp. mat. \ref{sec: cdf}).

\textbf{Optimization} For a fair comparison between our tests and those 
of \cite{test}, we use the same initialization of
the test locations\footnote{\cite{test}:
\href{https://github.com/wittawatj/interpretable-test}{github.com/wittawatj/interpretable-test}}. For the ME-based tests, we initialize the test locations with
realizations from two multivariate normal distributions fitted to samples from $P$ and $Q$ and for the for initialization of the SCF-based tests, we use the standard normal distribution. The regularization parameter is set to $\gamma_{N_1,N_2}=10^{-5}$.
The computation costs for our proposed tests are the same as that of
\cite{test}: with $t$ samples,
optimization is
$\mathcal{O}(J^3 + dJt)$ per gradient ascent iteration
and testing $\mathcal{O}(J^3 + Jt + dJt)$ (see supp. mat. Table \ref{table-time}).

The experiments on synthetic problems mirror those of \cite{test} to make a fair comparison between the prior art and the proposed methods.

\begin{figure*}[!t]
\includegraphics[width=\textwidth]{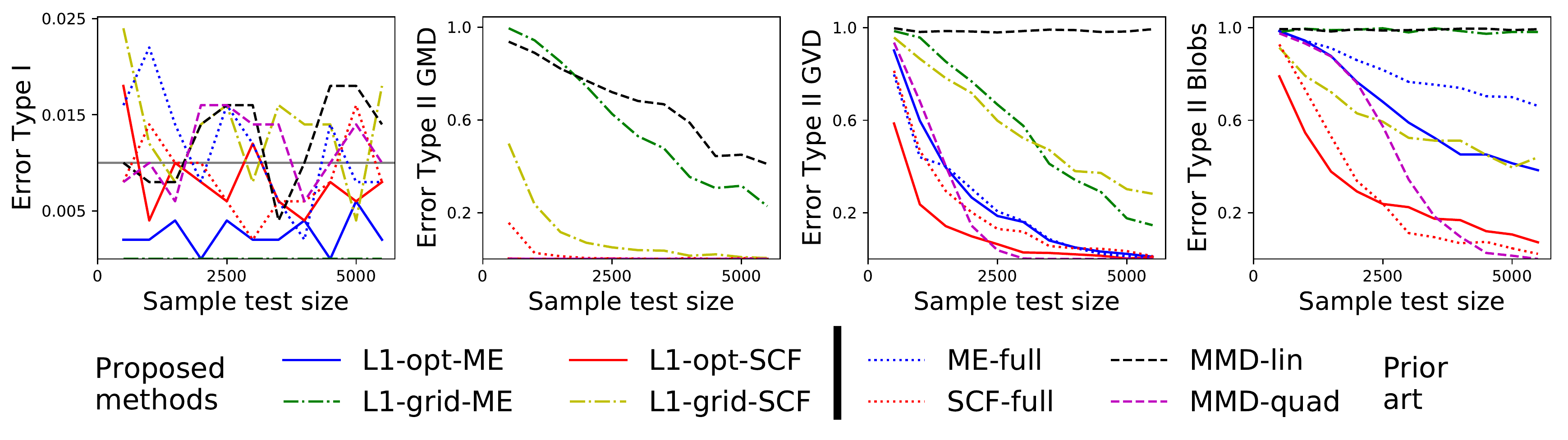}

\caption{Plots of type-I/type-II errors against the test sample
size $n^{te}$ in the four synthetic problems. \label{fig:err_vs_samples}}
\end{figure*}

\begin{wrapfigure}{r}{0.45\textwidth}\vspace*{-3.5ex}
\begin{tabular}{lcc}
Data & $P$ & $Q$ \\ 
\hline
SG & $\mathcal{N}(0,I_d)$ & $\mathcal{N}(0,I_d)$\\
GMD &$\mathcal{N}(0,I_d)$ &$\mathcal{N}((1,0,..,0)^{T},I_d)$ \\
GVD & $\mathcal{N}(0,I_d)$ & $\mathcal{N}(0,diag(2,1,..,1))$\\
Blobs & \multicolumn{2}{c}{Mixture of 16 Gaussians in $\mathbb{R}^2$ as \cite{test}}
\end{tabular}%
\medskip

\hfill%
\begin{minipage}{.7\linewidth}%
\captionof{table}{Synthetic problems. $H_0$ holds only in
SG.\label{tab:synthetic_pb}\vspace*{-.5ex}}%
\end{minipage}
\hfill\vbox{}%

\end{wrapfigure}

\textbf{Test power vs. sample size} 
We consider four synthetic problems: Same Gaussian (SG,
dim$=50$), Gaussian mean difference (GMD, dim$=100$), Gaussian variance
difference (GVD, dim$=30$), and Blobs. Table 
\ref{tab:synthetic_pb} summarizes 
the specifications of
$P$ and $Q$. In the Blobs
problem, $P$ and $Q$ are a mixture of Gaussian distributions
on a $4 \times 4$ grid in $\mathbb{R}^2$. This problem is
challenging as the difference of $P$ and $Q$ is encoded at a much smaller
length scale compared to the global structure as explained in \cite{choiceker}. We set $J = 5$ in this experiment.
\begin{figure*}[b!]
\vspace*{-3ex}\begin{minipage}{.25\linewidth}
\caption{Plots of type-I/type-II error against the dimension in
three synthetic problems: SG (Same Gaussian), GMD (Gaussian Mean
Difference), and GVD (Gaussian Variance Difference).\label{fig:err_vs_dim}}
\end{minipage}%
\hfill%
\begin{minipage}{.72\linewidth}
\includegraphics[width=\textwidth]{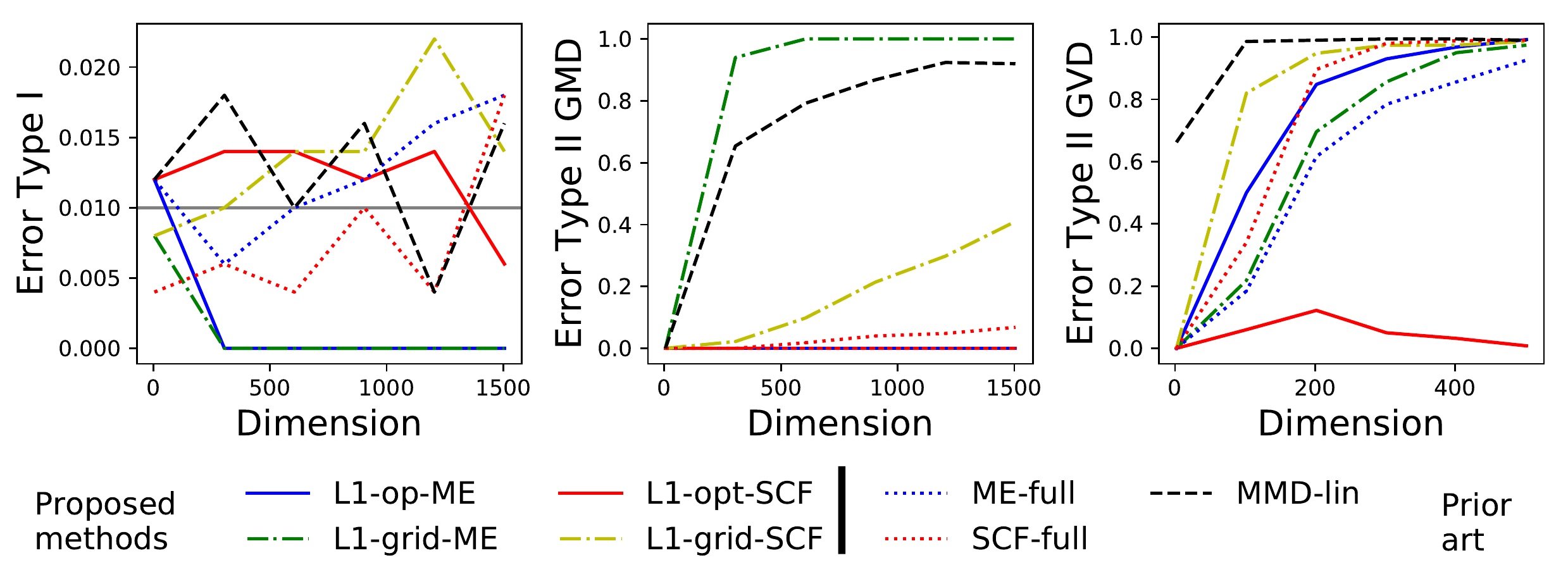}
\end{minipage}%
\end{figure*}

Figure \ref{fig:err_vs_samples} shows type-I error (for SG problem), and
test power (for GMD, GVD and Blobs problem) as a function of test sample size.  
In the SG problem, the type-I error roughly stays at the specified level
$\alpha$ for all tests except the L1-ME tests, which reject the null at a
rate below the specified level $\alpha$. Therefore, here these tests are more conservative.

GMD with 100 dimensions is an easy problem for
\textbf{L1-opt-ME}, \textbf{L1-opt-SCF}, \textbf{ME-full}
\textbf{MMD-quad}, while the \textbf{SCF-full} test requires many samples to achieve optimal test power.
In the GMD, GVD and Blobs cases, \textbf{L1-opt-ME} and
\textbf{L1-opt-SCF} achieve substantially higher test power than
\textbf{L1-grid-ME} and \textbf{L1-grid-SCF}, respectively:
optimizing the test locations brings a clear benefit.
Remarkably \textbf{L1-opt-SCF} consistently outperforms the
quadratic-time \textbf{MMD-quad} up to 2\,500 samples in the GVD case.
SCF variants perform significantly better than ME variants on the Blobs problem, as the difference in $P$ and $Q$ is localized in the frequency domain.
For the same reason, \textbf{L1-opt-SCF} does much better than the
quadratic-time MMD up to 3\,000 samples, as the latter
represents a weighted distance between characteristic functions
integrated across the frequency domain as explained in
\cite{freq}.

We also perform a more difficult GMD problem to distinguish the power of
the proposed tests with the \textbf{ME-full} as all reach maximal power.
\textbf{L1-opt-ME} then performs better than \textbf{ME-full}, its $\ell_2$
counterpart, as it needs less data to achieve good control (see mat. supp. \ref{sec:GMD}).

\textbf{Test power vs. dimension d} On fig \ref{fig:err_vs_dim}, we study how the
dimension of the problem affects type-I error and test power of our
tests. We consider the same synthetic problems: SG, GMD and GVD, we fix
the test sample size to 10000, set $J = 5$, and vary the dimension. Given
that these experiments explore large dimensions and a large number of
samples, computing the \textbf{MMD-quad} was too expensive.

In the SG problem, we observe the \textbf{L1-ME} tests are more conservative as
dimension increases, and the others tests can maintain type-I error at
roughly the specified significance level $\alpha = 0.01$. In the GMD
problem, we note that the tests proposed achieve the maximum test power
without making error of type-II  whatever the dimension is, while the
\textbf{SCF-full} loses power as dimension increases. However, this is
true only with optimization of the test locations as it is shown by the
test power of \textbf{L1-grid-ME} and \textbf{L1-grid-SCF} which drops as
dimension increases. Moreover the performance of \textbf{MMD-lin}
degrades quickly with increasing dimension, as expected from
\cite{ramdas}. Finally in the GVD problem, all tests failed to keep a
good test power as the dimension increases, except the
\textbf{L1-opt-SCF}, which has a very low type-II for all dimensions. These results echo those obtained by \cite{zhu2019interpoint}. Indeed \cite{zhu2019interpoint} study a class
of two sample test statistics based on inter-point distances and they
show benefits of using the $\ell_1$ norm over the Euclidean distance and
the Maximum Mean Discrepancy (MMD) when the dimensionality goes to
infinity.
For this class of test statistics, they characterize
asymptotic power loss w.r.t the dimension and show that the $\ell_1$ norm is
beneficial compared to the $\ell_2$ norm provided that the
summation of discrepancies between marginal univariate distributions is
large enough.

\begin{wrapfigure}{r}{0.45\textwidth}\vspace*{-3.5ex}
        \includegraphics[width=\linewidth]{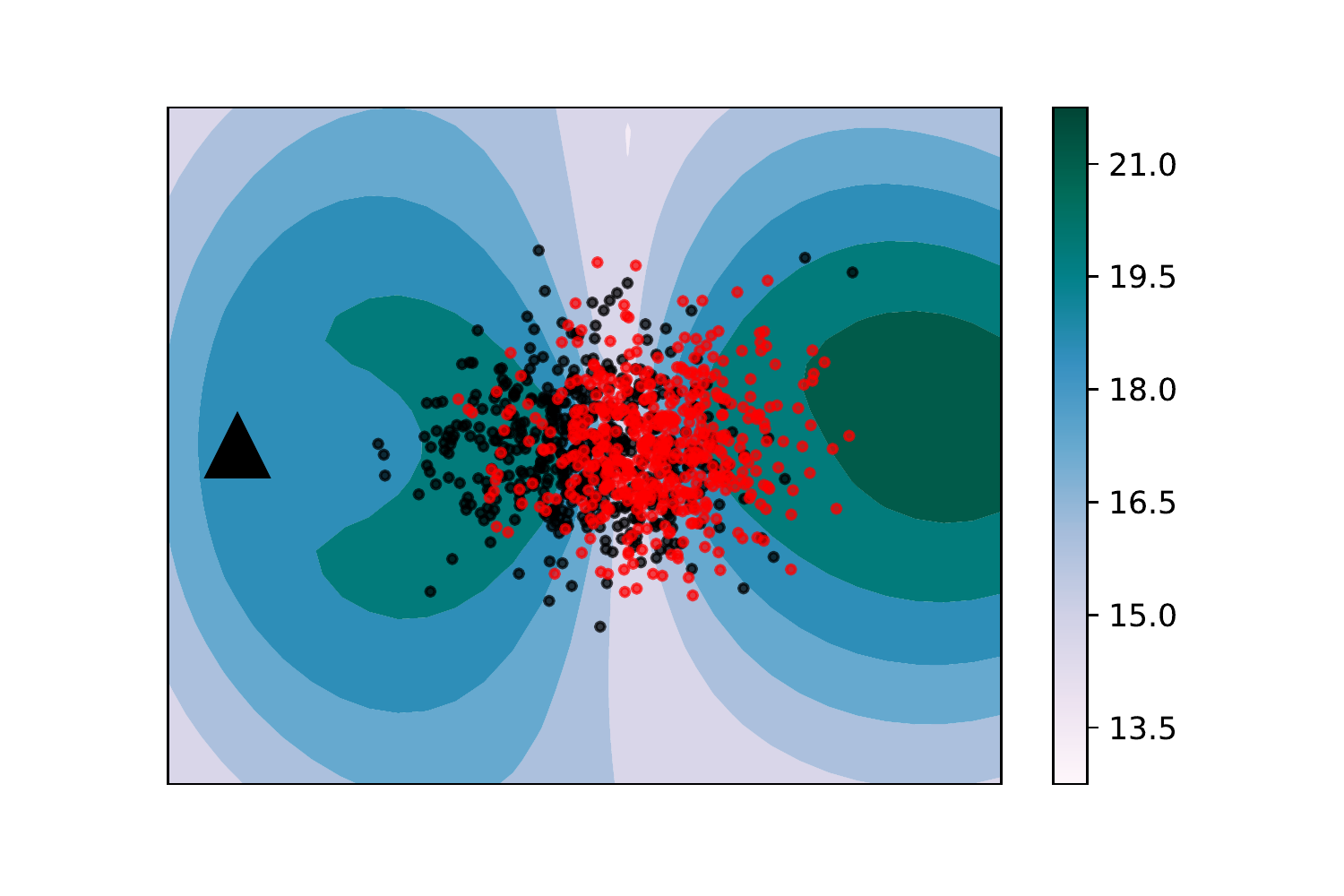}
    \caption{\textbf{Illustrating interpretable features}, replicating
    in the $\ell_1$ case the figure of \cite{test}. A contour plot of $ \widehat{\lambda}_{t}^{tr}\left(T_1,T_2\right)$ as a function of $T_2$,
    when $J=2$,
    and $T_1$ is fixed. The red and black dots represent the samples from
    the $P$ and $Q$ distributions, and the big black triangle the position of
    $T_1$ --complete figure in supp. mat. \ref{sec:info}.
    \label{info2}}
\end{wrapfigure}

\textbf{Informative features} 
\autoref{info2}
we replicate 
the experiment of \cite{test}, showing that the selected locations 
capture multiple modes in the $\ell_1$ case, as in the $\ell_2$ case.
(details in supp. mat. \ref{sec:info}). The figure shows that the objective
function $ \widehat{\lambda}_{t}^{tr}\left(T_1,T_2\right)$ used to
position the second test location $T_2$ has a maximum far from the chosen 
position for the first test location $T_1$.

\begin{figure}[b!]\vspace*{-2ex}
\begin{minipage}{.55\linewidth}
\caption{\textbf{Higgs dataset}: Plots of type-II errors against the test
sample size $n^{te}$.\label{fig:higgs}}
\includegraphics[trim={.29cm .35cm .25cm 7.35cm},clip,width=\linewidth]{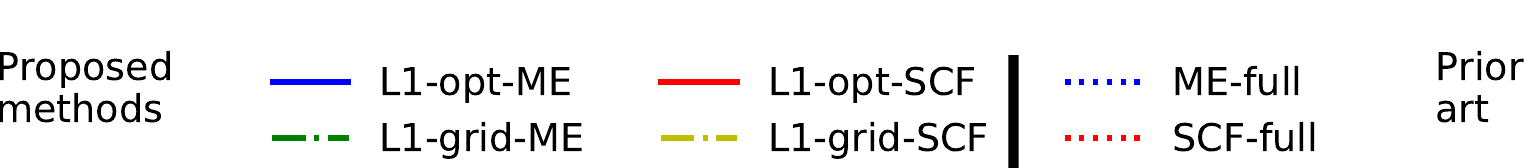}
\end{minipage}\hfill
\begin{minipage}{.4\linewidth}
\includegraphics[trim={.32cm 2.3cm 3.7cm .35cm},clip,width=\linewidth]{plot_higgs_f.pdf}
\end{minipage}
\end{figure}

\textbf{Real Data 1, Higgs}:
The first real problem is the Higgs dataset \cite{higgs_1}
described in \cite{higgs_2}: distinguishing signatures of
Higgs bosons from the background. We 
use a two-sample test on 4 derived
features as in \cite{ME}. 
We compare
for various sample sizes the performance of
the proposed tests with those of
\cite{test}. 
We do not study the $\textbf{MMD-quad}$ test as its computation is too
expensive with 10\,000 samples.
To make the problem harder, we only consider $J=3$
locations.
Fig.\,\ref{fig:higgs} shows a clear benefit of the optimized
$\ell_1$-based tests, in particular for SCF (\textbf{L1-opt-SCF}) compared
to its $\ell_2$ counterpart (\textbf{SCF-full}). Optimizing
the location is important, as \textbf{L1-opt-SCF} and
\textbf{L1-opt-ME} perform much better than their grid versions (which
are comparable to the tests of \cite{ME}).

\begin{table*}[b!]
\centering\small
\begin{tabular}{lccccccc}
Problem &  \hspace*{-3.5ex}\rotatebox{30}{L1-opt-ME}\hspace*{-1.7ex} &
\hspace*{-.5ex}\rotatebox{30}{L1-grid-ME}\hspace*{-1.7ex}  &
\hspace*{-.5ex}\rotatebox{30}{L1-opt-SCF}\hspace*{-1.7ex} &
\hspace*{-.5ex}\rotatebox{30}{L1-grid-SCF}\hspace*{-1.7ex} &
\hspace*{-.5ex}\rotatebox{30}{ME-full}\hspace*{-1.7ex} &
\hspace*{-.5ex}\rotatebox{30}{SCF-full}\hspace*{-1.7ex} &
\hspace*{-.5ex}\rotatebox{30}{MMD-quad}\hspace*{-1ex} \\
\hline
McDo vs Burger King (1141) &   0.112 &    0.426 &    0.428 & 0.960 &    0.170 &     0.094 &     0.184 \\
McDo vs Taco Bell (877)   &   0.554 &    0.624 &    0.710 &     0.834 &    0.684 &     0.638 &     0.666 \\
McDo vs Wendy's (733)   &   0.156 &    0.246 &    0.752 &     0.942 &    0.416 &     0.624 &     0.208 \\
McDo vs Arby's (517)  &   0.000 &    0.004 &    0.006 &     0.468 &    0.004 &     0.012 &     0.004 \\
McDo vs KFC   (429)    &   0.912 &    0.990 &    1.00 &     0.998 &    0.996 &     0.856 &     0.980 \\
\end{tabular}
\caption{\textbf{Fast food dataset:} Type-II errors for
distinguishing the distribution of fast food restaurants. $\alpha= 0.01$.
$J = 3$. The number in brackets denotes the sample size of the
distribution on the right. We consider MMD-quad as the gold standard.}
\label{Tab:exp3result}
\end{table*}
\textbf{Real Data 2, Fastfood:} We use a Kaggle dataset
listing locations of
over 10,000 fast food restaurants across America\footnote{\href{https://www.kaggle.com/datafiniti/fast-food-restaurants}{www.kaggle.com/datafiniti/fast-food-restaurants}}. We consider the 6 most
frequent brands in mainland USA: Mc Donald's, Burger King, Taco Bell,
Wendy's, Arby's and KFC. We benchmark the various two-sample tests to
test whether the spatial distribution
(in $\mathbb{R}^2$) of restaurants differs across brand. This is a non
trivial question, as it depends on marketing strategy of the
brand. We compare the distribution of Mc Donald's restaurants with
others. We also compare the distribution of Mc Donald's restaurants with itself to evaluate the level of the tests (see   supp. mat. Table \ref{Tab:exp3result-level}).
The number of samples differ across the
distributions; hence to perform the tests from \cite{test}, we
randomly subsample the largest distribution. We use 
$J=3$ as the number of locations. 

Table \ref{Tab:exp3result} summarizes type-II errors of the tests. Note
that it is not clear that distributions must differ, as
two brands sometimes compete directly, and target similar
locations. 
We consider the \textbf{MMD-quad} as the gold standard to decide whether distributions differ or not.
The three cases for which there seems to be a difference are
Mc Donald's vs Burger King, Mc Donald's vs Wendy's, and Mc Donalds vs
Arby's.
Overall, we find that the optimized \textbf{L1-opt-ME} agrees best with this gold
standard. 
The Mc Donald's vs Arby's problem seems to be an easy problem, as
all tests reach a maximal power, except for the
\textbf{L1-grid-SCF} test which shows the gain of power brought by the optimization.
In the Mc Donald's vs Wendy's problem the \textbf{L1-opt-ME} test outperforms the $\ell_2$ tests and even the quadratic-time MMD. Finally, all the tests fail to discriminate Mc Donald's vs KFC. The data
provide no evidence that these brands pursue different strategies to
chose locations. 

In the Mc Donald's vs Burger King and Mc Donald's
vs Wendy's problems, the optimized version of the test proposed based on
mean embedding outperform the grid version. This
success implies that the locations learned are each informative, and we
plot them (see supp. mat. Figure \ref{fig:usa}), to 
investigate the interpretability of the \textbf{L1-opt-ME} test.
The figure shows that the procedure narrows on specific regions
of the USA to find differences between distributions of restaurants.

\textbf{Real Data 3, text:} For a
high-dimension problem, we consider the problem of
distinguishing the newsgroups text dataset \cite{Lang95} (details in
supp. Mat. \ref{sec:20group}). Compared to their $\ell_2$ counterpart,
$\ell_1$-optimized tests bring clear benefits and separate all topics of
articles based on their word distribution.

\textbf{Discussion:}
Our theoretical results suggest it is always beneficial for statistical
power to build tests on $\ell_1$ norms rather than $\ell_2$ norm of
differences between kernel distribution representatives (Propositions
\ref{prop:finalement}, \ref{prop:finalement-SCF}). In practice, however,
optimizing test locations with $\ell_1$-norm tests leads to
non-smooth objective functions that are harder to optimize. Our
experiments confirm the theoretical benefit of the $\ell_1$-based
framework. The benefit is particularly pronounced for a large number $J$
of test locations --as the difference between $\ell_1$ and $\ell_2$ norms
increases with dimension (see in supp. mat. Lemmas \ref{lem:betterpow}, \ref{lem:better-power-scf})-- as well as for large dimension of the native
space (\autoref{fig:err_vs_dim}). The benefit of $\ell_1$ distances for
two-sample testing in high dimension has also been reported by
\cite{zhu2019interpoint}, though their framework does not
link to kernel embeddings or to the convergence of probability measures. Further work should consider extending these results to  goodness-of-fit testing, where the $L^1$ geometry was shown empirically to provide excellent performance \cite{huggins2018random}.

\section{Conclusion}
In this paper, we show that statistics derived from the $L^p$
distances between well-chosen distribution representatives are well suited for the two-sample problem as these distances metrize the weak convergence (Theorem \ref{thm:metricl1}). We then compare the power of tests introduced in \cite{ME} and their $\ell_1$ counterparts and we show that $\ell_1$-based statistics have better power with high probability (Propositions \ref{prop:finalement}, \ref{prop:finalement-SCF}). As with state-of-the-art
Euclidean approaches, the framework leads to
tractable computations and learns interpretable locations of where the
distributions differ.
Empirically, on all 4 synthetic and 3 real problems investigated,
the $\ell_1$ geometry gives clear benefits compared to the Euclidean 
geometry. The $L^1$ distance is known to be well suited for densities, to control differences or estimation \cite{devroye1985nonparametric}. It
is also beneficial for kernel embeddings of distributions.

\paragraph*{Acknowledgments}

This work was funded by the DirtyDATA ANR grant (ANR-17-CE23-0018).
We also would like to thank Zoltán Szabó from École Polytechnique for
crucial suggestions, and acknowledge hardware donations from NVIDIA
Corporation.

\bibliography{biblio}

\begin{thebibliography}{34}
\providecommand{\natexlab}[1]{#1}
\providecommand{\url}[1]{\texttt{#1}}
\expandafter\ifx\csname urlstyle\endcsname\relax
  \providecommand{\doi}[1]{doi: #1}\else
  \providecommand{\doi}{doi: \begingroup \urlstyle{rm}\Url}\fi

\bibitem[Arjovsky et~al.(2017)Arjovsky, Chintala, and
  Bottou]{arjovsky2017wasserstein}
M.~Arjovsky, S.~Chintala, and L.~Bottou.
\newblock Wasserstein generative adversarial networks.
\newblock In \emph{International Conference on Machine Learning}, pages
  214--223, 2017.

\bibitem[Baldi et~al.(2014)Baldi, Sadowski, and Whiteson]{higgs_2}
P.~Baldi, P.~Sadowski, and D.~Whiteson.
\newblock Searching for exotic particles in high-energy physics with deep
  learning.
\newblock \emph{Nature communications}, 5:\penalty0 4308, 2014.

\bibitem[Baringhaus and Franz(2004)]{baringhaus}
L.~Baringhaus and C.~Franz.
\newblock On a new multivariate two-sample test.
\newblock \emph{Journal of multivariate analysis}, 88\penalty0 (1):\penalty0
  190--206, 2004.

\bibitem[Borgwardt et~al.(2006)Borgwardt, Gretton, Rasch, Kriegel,
  Sch{\"o}lkopf, and Smola]{bio}
K.~M. Borgwardt, A.~Gretton, M.~J. Rasch, H.-P. Kriegel, B.~Sch{\"o}lkopf, and
  A.~J. Smola.
\newblock Integrating structured biological data by kernel maximum mean
  discrepancy.
\newblock \emph{Bioinformatics}, 22\penalty0 (14):\penalty0 e49--e57, 2006.

\bibitem[Chwialkowski et~al.(2015)Chwialkowski, Ramdas, Sejdinovic, and
  Gretton]{ME}
K.~P. Chwialkowski, A.~Ramdas, D.~Sejdinovic, and A.~Gretton.
\newblock Fast two-sample testing with analytic representations of probability
  measures.
\newblock In \emph{Advances in Neural Information Processing Systems}, pages
  1981--1989, 2015.

\bibitem[Cucker and Smale(2002)]{cucker}
F.~Cucker and S.~Smale.
\newblock On the mathematical foundations of learning.
\newblock \emph{Bulletin of the American mathematical society}, 39\penalty0
  (1):\penalty0 1--49, 2002.

\bibitem[Devroye and Gy{\"o}rfi(1985)]{devroye1985nonparametric}
L.~Devroye and L.~Gy{\"o}rfi.
\newblock Nonparametric density estimation: The l1 view, 1985.

\bibitem[Dharmawansa et~al.(2007)Dharmawansa, Rajatheva, and Ahmed]{naka}
P.~Dharmawansa, N.~Rajatheva, and K.~Ahmed.
\newblock On the distribution of the sum of nakagami-$ m $ random variables.
\newblock \emph{IEEE transactions on communications}, 55\penalty0 (7):\penalty0
  1407--1416, 2007.

\bibitem[Fromont et~al.(2012)Fromont, Lerasle, Reynaud-Bouret, et~al.]{fromont}
M.~Fromont, M.~Lerasle, P.~Reynaud-Bouret, et~al.
\newblock Kernels based tests with non-asymptotic bootstrap approaches for
  two-sample problems.
\newblock In \emph{Conference on Learning Theory}, page~23, 2012.

\bibitem[Fukumizu et~al.(2008)Fukumizu, Gretton, Sun, and
  Sch{\"o}lkopf]{fukumizu2008kernel}
K.~Fukumizu, A.~Gretton, X.~Sun, and B.~Sch{\"o}lkopf.
\newblock Kernel measures of conditional dependence.
\newblock In \emph{Advances in neural information processing systems}, pages
  489--496, 2008.

\bibitem[Gretton et~al.(2007)Gretton, Borgwardt, Rasch, Sch{\"o}lkopf, and
  Smola]{MMD}
A.~Gretton, K.~M. Borgwardt, M.~Rasch, B.~Sch{\"o}lkopf, and A.~J. Smola.
\newblock A kernel method for the two-sample-problem.
\newblock In \emph{Advances in neural information processing systems}, pages
  513--520, 2007.

\bibitem[Gretton et~al.(2009)Gretton, Fukumizu, Harchaoui, and
  Sriperumbudur]{gretton2009fast}
A.~Gretton, K.~Fukumizu, Z.~Harchaoui, and B.~K. Sriperumbudur.
\newblock A fast, consistent kernel two-sample test.
\newblock In \emph{Advances in neural information processing systems}, pages
  673--681, 2009.

\bibitem[Gretton et~al.(2012{\natexlab{a}})Gretton, Borgwardt, Rasch,
  Sch{\"o}lkopf, and Smola]{MMD-real}
A.~Gretton, K.~M. Borgwardt, M.~J. Rasch, B.~Sch{\"o}lkopf, and A.~Smola.
\newblock A kernel two-sample test.
\newblock \emph{Journal of Machine Learning Research}, 13\penalty0
  (Mar):\penalty0 723--773, 2012{\natexlab{a}}.

\bibitem[Gretton et~al.(2012{\natexlab{b}})Gretton, Sejdinovic, Strathmann,
  Balakrishnan, Pontil, Fukumizu, and Sriperumbudur]{choiceker}
A.~Gretton, D.~Sejdinovic, H.~Strathmann, S.~Balakrishnan, M.~Pontil,
  K.~Fukumizu, and B.~K. Sriperumbudur.
\newblock Optimal kernel choice for large-scale two-sample tests.
\newblock In \emph{Advances in neural information processing systems}, page
  1205, 2012{\natexlab{b}}.

\bibitem[Harchaoui et~al.(2008)Harchaoui, Moulines, and Bach]{eric}
Z.~Harchaoui, E.~Moulines, and F.~R. Bach.
\newblock Testing for homogeneity with kernel fisher discriminant analysis.
\newblock In \emph{Advances in Neural Information Processing Systems}, page
  609, 2008.

\bibitem[Huggins and Mackey(2018)]{huggins2018random}
J.~Huggins and L.~Mackey.
\newblock Random feature stein discrepancies.
\newblock In \emph{Advances in Neural Information Processing Systems}, pages
  1899--1909, 2018.

\bibitem[Jitkrittum et~al.(2016)Jitkrittum, Szab{\'o}, Chwialkowski, and
  Gretton]{test}
W.~Jitkrittum, Z.~Szab{\'o}, K.~P. Chwialkowski, and A.~Gretton.
\newblock Interpretable distribution features with maximum testing power.
\newblock In \emph{Advances in Neural Information Processing Systems}, pages
  181--189, 2016.

\bibitem[Lang(1995)]{Lang95}
K.~Lang.
\newblock Newsweeder: Learning to filter netnews.
\newblock In \emph{Proceedings of the Twelfth International Conference on
  Machine Learning}, pages 331--339, 1995.

\bibitem[Le et~al.(2013)Le, Sarl{\'o}s, and Smola]{randfeat2}
Q.~Le, T.~Sarl{\'o}s, and A.~Smola.
\newblock Fastfood-computing hilbert space expansions in loglinear time.
\newblock In \emph{International Conference on Machine Learning}, pages
  244--252, 2013.

\bibitem[Li et~al.(2017)Li, Chang, Cheng, Yang, and P{\'o}czos]{li2017mmd}
C.-L. Li, W.-C. Chang, Y.~Cheng, Y.~Yang, and B.~P{\'o}czos.
\newblock Mmd gan: Towards deeper understanding of moment matching network.
\newblock In \emph{Advances in Neural Information Processing Systems}, pages
  2203--2213, 2017.

\bibitem[Lichman et~al.(2013)]{higgs_1}
M.~Lichman et~al.
\newblock {UCI} machine learning repository, 2013.

\bibitem[Lloyd and Ghahramani(2015)]{checking}
J.~R. Lloyd and Z.~Ghahramani.
\newblock Statistical model criticism using kernel two sample tests.
\newblock In \emph{Advances in Neural Information Processing Systems}, pages
  829--837, 2015.

\bibitem[Paszke et~al.(2017)Paszke, Gross, Chintala, Chanan, Yang, DeVito, Lin,
  Desmaison, Antiga, and Lerer]{paszke2017automatic}
A.~Paszke, S.~Gross, S.~Chintala, G.~Chanan, E.~Yang, Z.~DeVito, Z.~Lin,
  A.~Desmaison, L.~Antiga, and A.~Lerer.
\newblock Automatic differentiation in pytorch.
\newblock 2017.

\bibitem[Rahimi and Recht(2008)]{randomfeat}
A.~Rahimi and B.~Recht.
\newblock Random features for large-scale kernel machines.
\newblock In \emph{Advances in neural information processing systems}, pages
  1177--1184, 2008.

\bibitem[Ramdas et~al.(2015)Ramdas, Reddi, P{\'o}czos, Singh, and
  Wasserman]{ramdas}
A.~Ramdas, S.~J. Reddi, B.~P{\'o}czos, A.~Singh, and L.~A. Wasserman.
\newblock On the decreasing power of kernel and distance based nonparametric
  hypothesis tests in high dimensions.
\newblock In \emph{AAAI}, pages 3571--3577, 2015.

\bibitem[Sejdinovic et~al.(2013)Sejdinovic, Sriperumbudur, Gretton, and
  Fukumizu]{energy}
D.~Sejdinovic, B.~Sriperumbudur, A.~Gretton, and K.~Fukumizu.
\newblock Equivalence of distance-based and rkhs-based statistics in hypothesis
  testing.
\newblock \emph{The Annals of Statistics}, pages 2263--2291, 2013.

\bibitem[Simon(2005)]{simon}
B.~Simon.
\newblock \emph{Trace ideals and their applications}.
\newblock Number 120. Am. Math. Soc., 2005.

\bibitem[Simon-Gabriel and Sch{\"o}lkopf(2016)]{weak-cvg}
C.-J. Simon-Gabriel and B.~Sch{\"o}lkopf.
\newblock Kernel distribution embeddings: Universal kernels, characteristic
  kernels and kernel metrics on distributions.
\newblock \emph{arXiv preprint arXiv:1604.05251}, 2016.

\bibitem[Sriperumbudur et~al.(2010)Sriperumbudur, Gretton, Fukumizu,
  Sch{\"o}lkopf, and Lanckriet]{freq}
B.~K. Sriperumbudur, A.~Gretton, K.~Fukumizu, B.~Sch{\"o}lkopf, and G.~R.
  Lanckriet.
\newblock Hilbert space embeddings and metrics on probability measures.
\newblock \emph{Journal of Machine Learning Research}, 11:\penalty0 1517, 2010.

\bibitem[Sriperumbudur et~al.(2011)Sriperumbudur, Fukumizu, and
  Lanckriet]{Universality}
B.~K. Sriperumbudur, K.~Fukumizu, and G.~R. Lanckriet.
\newblock Universality, characteristic kernels and rkhs embedding of measures.
\newblock \emph{Journal of Machine Learning Research}, 12:\penalty0 2389, 2011.

\bibitem[Steinwart(2001)]{steinwart2001influence}
I.~Steinwart.
\newblock On the influence of the kernel on the consistency of support vector
  machines.
\newblock \emph{Journal of machine learning research}, 2\penalty0
  (Nov):\penalty0 67--93, 2001.

\bibitem[Sz{\'e}kely and Rizzo(2004)]{szekely}
G.~J. Sz{\'e}kely and M.~L. Rizzo.
\newblock Testing for equal distributions in high dimension.
\newblock \emph{InterStat}, 5\penalty0 (16.10):\penalty0 1249--1272, 2004.

\bibitem[Zaremba et~al.(2013)Zaremba, Gretton, and Blaschko]{zaremba}
W.~Zaremba, A.~Gretton, and M.~Blaschko.
\newblock B-test: A non-parametric, low variance kernel two-sample test.
\newblock In \emph{Advances in neural information processing systems}, pages
  755--763, 2013.

\bibitem[Zhu and Shao(2019)]{zhu2019interpoint}
C.~Zhu and X.~Shao.
\newblock Interpoint distance based two sample tests in high dimension.
\newblock \emph{arXiv preprint arXiv:1902.07279}, 2019.

\end{thebibliography}

\iftoggle{supplementary}{
\clearpage
\appendix

\onecolumn

\bigskip

\renewcommand \thepart{}
\renewcommand \partname{}

\part{Supplementary materials} 
\parttoc 

\section{A family of metrics that metrize of the weak convergence}
\subsection{Distances between Mean Embeddings}
\label{sec:weak_cvg_ME}
\begin{theorem}
Given $p\geq 1$, $k$ a definite positive, characteristic, continuous, and bounded
kernel on $\mathbb{R}^d$,
$\mu_P$ and $\mu_Q$ the mean embeddings of the Borel probability measures $P$ and
$Q$ respectively, 
the function defined on $\mathcal{M}_{+}^{1}(\mathbb{R}^d)\times \mathcal{M}_{+}^{1}(\mathbb{R}^d)$:
\begin{equation}
d_{L^p,\mu}(P,Q):=\left(\int_{t\in\mathbb{R}^d }\biggl|\mu_P(t)-\mu_Q(t)\biggr|^p d\Gamma(t)\right)^{1/p}
\end{equation}
is a metric on the space of Borel probability measures, for $\Gamma$ a Borel probability measure absolutely continuous with
respect to Lebesgue measure. Moreover a sequence $(\alpha_n)_{n\geq 0}$ of Borel probability measures converges weakly towards $\alpha$ if and only if  $d_{L^p,\mu}(\alpha_n,\alpha)\rightarrow 0$.
\end{theorem}

\begin{prv}
First, let us prove that for any $p\geq 1$ $d_{L^p,\mu}$ is metric of on the space of Borel probability measures. Let $p\geq 1$, we have:
$$|\mu_P\left(t\right)-\mu_Q\left(t\right)|^p=|\langle \mu_P-\mu_Q,k_t\rangle|^p$$
Therefore:
$$|\mu_P\left(t\right)-\mu_Q\left(t\right)|^p\leq ||\mu_P-\mu_Q||_H^p(k\left(t,t\right))^{p/2}$$
But as $k$ is bounded, and $\Gamma$ is finite, $d_{L^p,\mu}$ is well defined on $\mathcal{M}_{+}^{1}\left(\mathcal{X}\right)^2$.
Let us prove now that if $P\neq Q$ then $d_{L^p,\mu}\left(P,Q\right)>0$.
\begin{definition}\cite{fukumizu2008kernel}
A kernel is characteristic if the mapping $P\in \mathcal{M}_{+}^{1}\left(\mathcal{X}\right)\rightarrow \mu_P\in H_k$ is injective, where $H_k$ is the RKHS associated with $k$.
\end{definition}
\begin{lemma}\cite{steinwart2001influence}
\label{lemma:continue-sten}
If k is a continuous kernel on a metric space then every feature maps associated with the kernel are continuous.
\end{lemma}
Let $P$ and $Q$ two Borel distributions such that $P\neq Q$.
Since the mapping $p\rightarrow \mu_P$ is injective, there must exists at least one point $o$ where $\mu_P-\mu_Q$ is non-zero. By continuity of $\mu_P-\mu_Q$, there exists a ball around $o$ in which  $\mu_P-\mu_Q$ is non-zero. Then $d_{L^p,\mu}\left(P,Q\right)>0$. Finally all the other proprieties of a metric are clearly verified by this function.

Let us now show that $d_{L^1,\mu}$ metrize the weak convergence. For that purpose, we first show that this metric has an IPM formulation:
\begin{lemma}l
\label{lemma:ipm}
We denote by $\mathcal{T}_{k}$ the integral operator on $L_{2}^{d\Gamma}(\mathbb{R}^d)$ associated to the positive definite, characteristic, continuous, and bounded kernel $k$ defined as:
$$\begin{array}{ccccc}
\mathcal{T}_{k} & : &  L_{2}^{d\Gamma}(\mathbb{R}^d) & \to &  L_{2}^{d\Gamma}(\mathbb{R}^d)\\
 & & f &\to & \int_{\mathbb{R}^d} k(x,.)f(x)d\Gamma(x)
\end{array}$$
By denoting $B_{\infty}^{d\Gamma}$ the unit ball of $L_{\infty}^{d\Gamma}(\mathbb{R}^d)$, we have that:
\begin{align*}
d_{L^1,\mu}&(P,Q)=\displaystyle\sup_{f\in \mathcal{T}_k(B_{\infty}^{d\Gamma})}
\biggl(\mathbb{E}_{P}[f(X)]-\mathbb{E}_{Q}[f(Y)]\biggr)
\end{align*}
\end{lemma}

\begin{prv}
We have:
\begin{align*}
d_{L^1,\mu}\left(P,Q\right)=&\int_{x\in \mathbb{R}^d}|\mu_P\left(x\right)-\mu_Q\left(x\right)|d\Gamma(x)\\
=&\int_{x\in \mathbb{R}^d}|\langle\mu_P-\mu_Q,k_x\rangle_H|d\Gamma(x)\\
=&\int_{x\in\{v:\mu_P\left(v\right)\geq\mu_Q\left(v\right)\}}\langle\mu_P-\mu_Q,k_x\rangle_H d\Gamma(x)
-\int_{x\in\{v:\mu_P\left(v\right)<\mu_Q\left(v\right)\}}\langle\mu_P-\mu_Q,k_x\rangle_H d\Gamma(x)\\
=&\langle\mu_P-\mu_Q,\int_{x\in\{v:\mu_P\left(v\right)\geq\mu_Q\left(v\right)\}}k_x d\Gamma(x)
-\int_{x\in\{v:\mu_P\left(v\right)<\mu_Q\left(v\right)\}}k_x d\Gamma(x)\rangle_H
\end{align*}
Then:
$$d_{L^1,\mu}\left(P,Q\right)=\langle\mu_P-\mu_Q,f\rangle_H$$
with
\begin{align}
f=\int_{t\in \mathbb{R}^d}g_t d\Gamma(t)\quad\text{where}\quad
g_t=\left\{
              \begin{array}{ll}
              k_t \text{ if  } t\in\{x: \mu_P\left(x\right)\geq\mu_Q\left(x\right)\}\\
                \ -k_t  \text{ otherwise.}
                \end{array}\right.
\end{align}
Therefore, $f \in T_k(B_{\infty}^{d\Gamma})\subset H_k$ and we have:
\begin{align*}
d_{L^1,\mu}\left(P,Q\right)=\mathbb{E}_{P}\left(f\left(X\right)\right)-\mathbb{E}_{Q}\left(f\left(Y\right)\right)\\
\end{align*}
\medbreak
Now, let f be an element of $\mathcal{T}_k(B_{\infty}^{d\Gamma})\subset H_k$.
Therefore there exists $g\in B_{\infty}^{d\Gamma}$ such that $f=\mathcal{T}_k(g)$ and we have then:
\begin{align*}
\mathbb{E}_{P}\left(f\left(X\right)\right)-\mathbb{E}_{Q}\left(f\left(Y\right)\right)=&\langle \mu_P-\mu_Q,f\rangle\\
=&\langle \mu_P-\mu_Q,\int_{t\in \mathbb{R}^d}g(t) k_t d\Gamma(t)\rangle\\
=&\int_{t\in \mathbb{R}^d}g(t) \langle\mu_P-\mu_Q,k_t\rangle d\Gamma(t)\\
=&\int_{t\in\mathbb{R}^d}g(t) \left(\mu_P\left(t\right)-\mu_Q\left(t\right)\right) d\Gamma(t)\\
\leq & \int_{t\in \mathbb{R}^d}|\mu_P\left(t\right)-\mu_Q\left(t\right)|d\Gamma(t)
\end{align*}
Therefore we have:
\begin{align*}
d_{L^1,\mu}&(P,Q)=\displaystyle\sup_{f\in \mathcal{T}_k(B_{\infty}^{d\Gamma})}
\biggl(\mathbb{E}_{P}[f(X)]-\mathbb{E}_{Q}[f(Y)]\biggr)
\end{align*}
\end{prv}
From this IPM formulation we now show that $d_{L^1,\mu}$ metrize the weak convergence.
First, as the kernel $k$ is assumed to be continuous, then $\mathcal{T}_k(B_{\infty}^{d\Gamma})\subset H_k\subset C^0(\mathbb{R}^d)$, the set of continuous functions. Therefore, thanks to the IPM formulation of the metric, the weak convergence of a sequence of distributions $(\alpha_n)_{n\geq 0}$ towards a distribution $\alpha$ implies the convergence according to the $d_{L^1,\mu}$-distance.
Conversely let $\alpha\in\mathcal{M}_{+}^{1}\left(\mathcal{X}\right)$ and let us assume that $(\alpha_n)_{n\geq 0}$ is a sequence of Borel probability measures such that $d_{L^1,\mu}(\alpha_n,\alpha)\rightarrow 0$. Since $\int\limits_{x\in\mathbb{R}^d} k(x, x)d\Gamma(x)$ is finite, $T_k$ is self-adjoint, positive semi-definite and trace-class \cite{simon}. It has at most countably many positive eigenvalues $(\lambda_m)_{m\geq 0}$ and corresponding orthonormal eigenfunctions $(e_m)_{m\geq 0}$. Then the Mercer theorem \cite{cucker} gives that $(\lambda_m^{1/2}e_m)_{m\geq 0}$ is an orthonormal basis of $H_k$.
Let us denote $C=\sup\limits_{x\in\mathbb{R}^d}\sqrt{k(x,x)}$ And $V_m=\frac{\lambda_m^{1/2}e_m}{C}$. Therefore we have:
\begin{align*}
\Vert V_m\Vert_{\infty,d\Gamma}&\leq \frac{\Vert \lambda_m^{1/2}e_m\Vert_{H_k}}{C}C\leq 1
\end{align*}
Therefore, thanks to Lemma \ref{lemma:ipm}, for all $m\geq 0$, we have:
\begin{align*}
\langle \mu_{\alpha_n}-\mu_\alpha,T_k(V_m)\rangle_{H_k}\rightarrow 0
\end{align*}
Now, we want to show that for every $f\in H_k$, $\langle\mu_{\alpha_n}-\mu_\alpha,f\rangle_{H_k}\rightarrow 0$. Let us consider $f\in H_k$. As $(\lambda_m^{1/2}e_m)_{m\geq 0}$ is an orthonormal basis of $H_k$, we have:
\begin{align*}
f=\sum_{m\geq 0}\langle f,\lambda_m^{1/2}e_m\rangle_{H_k}\lambda_m^{1/2}e_m
\end{align*}
Therefore if we define for every $m\geq 0$:
\begin{align*}
f_m:=\sum_{i=0}^m \langle f,\lambda_i^{1/2}e_i\rangle_{H_k}\lambda_i^{1/2}e_i
\end{align*}
We have that:
\begin{align*}
\Vert f_m-f\Vert_{H_k}\rightarrow 0
\end{align*}
Therefore let $\epsilon>0$, and $K$ such that:
\begin{align*}
\Vert f_K-f\Vert_{H_k}\leq \epsilon
\end{align*}
First we remarks that:
\begin{align*}
\langle \mu_{\alpha_n}-\mu_\alpha,f_K\rangle&=\sum_{i=0}^K \langle f,\lambda_i^{1/2} e_i\rangle \langle \mu_{\alpha_n}-\mu_\alpha, \lambda_i^{1/2} e_i\rangle\\
&=\sum_{i=0}^K \langle f,\lambda_i^{1/2} e_i\rangle C  \langle \mu_{\alpha_n}-\mu_\alpha, V_i\rangle\\
&=\sum_{i=0}^K \langle f,\lambda_i^{1/2} e_i\rangle \frac{C}{\lambda_i}  \langle \mu_{\alpha_n}-\mu_\alpha, T_k(V_i)\rangle
\end{align*}
Indeed the last equality hold as all the eigenvalues are positives.
Finally we have that:
\begin{align*}
\langle \mu_{\alpha_n}-\mu_\alpha,f_K\rangle_{H_k}\rightarrow 0 \text{ as $n$ goes to infinity.}
\end{align*}
Let $N$, such that for $n\geq N$:
\begin{align*}
\langle \mu_{\alpha_n}-\mu_\alpha,f_K\rangle_{H_k}\leq \epsilon 
\end{align*}
Therefore we have for all $n\geq N$:
\begin{align*}
\langle \mu_{\alpha_n}-\mu_\alpha,f\rangle&=\langle \mu_{\alpha_n}-\mu_\alpha,f_K\rangle+\langle \mu_{\alpha_n}-\mu_\alpha,f-f_K\rangle\\
&\leq \epsilon + \Vert \mu_{\alpha_n}-\mu_\alpha\Vert_{H_k}\Vert f-f_K\Vert_{H_k}\\
&\leq \epsilon + \Vert \mu_{\alpha_n}-\mu_\alpha\Vert_{H_k}\epsilon
\end{align*}
Finally as $k$ is bounded, we have that:
\begin{align*}
 \Vert \mu_{\alpha_n}-\mu_\alpha\Vert_{H_k}\leq 2\sup_{x,t}\sqrt{k(x,t)}
\end{align*}
Finally we have that for every $f\in H_k$:
\begin{align*}
\langle \mu_{\alpha_n}-\mu_\alpha,f\rangle\rightarrow 0
\end{align*}
Therefore for any $f\in B_{H_k}$, the unit ball of the RKHS, we have:
\begin{align*}
\langle \mu_{\alpha_n}-\mu_\alpha,f\rangle\rightarrow 0
\end{align*}
And then:
\begin{align*}
\text{MMD}[\alpha_n, \alpha]\rightarrow 0
\end{align*}
Moreover we have the following theorem:
\begin{theorem}{(\cite{weak-cvg})}
 A bounded kernel over a locally compact Hausdorff space $\mathcal{X}$ metrizes the weak convergence of probability measures iff it is continuous and characteristic.
\end{theorem}
Therefore $\alpha_n$ converge weakly towards $\alpha$ and  $d_{L^1,\mu}$ metrize the weak convergence.
Moreover thanks to Hölder's inequality we have that for any $p\geq 1$:
\begin{align}
d_{L^1,\mu}(P,Q)\leq d_{L^p,\mu}(P,Q)
\end{align}
Moreover as the kernel $k$ is bounded we have also:
\begin{align}
d_{L^p,\mu}(P,Q)^p&\leq \Vert \mu_P-\mu_Q\Vert_{\infty}^{p-1}d_{L^1,\mu}(P,Q)\\
&\leq (2C^2)^{p-1}d_{L^1,\mu}(P,Q)
\end{align}
Therefore for any $p\geq 1$ $d_{L^p,\mu}$ metrizes the weak convergence.
\end{prv}
\subsection{Distances between Smooth Characteristic Functions}
\label{sec:weak_cvg_SCF}
\begin{definition}{\cite{ME}}
\label{def:smooth}
Let $k:\mathbb{R}^d\rightarrow \mathbb{R}$ be a translation-invariant kernel i.e., $k(x-y)$ defines a positive definite kernel for x and y, $P$ a Borel probability measure and $\psi_P(t):=\mathbb{E}_{x}\left(\exp(ix^{T}t)\right)$ be the characteristic function of $P$. A smooth characteristic function $\Phi_P$ is defined as:
\begin{align}
\Phi_P(v):=\int_{\mathbb{R}^d}\psi_P(t)k(v-t)dt
\end{align}
\end{definition}
\begin{lemma}\cite{ME}
\label{lem:smooth-rkhs}
If $k$ is a continuous, integrable and translation invariant kernel with an inverse Fourier transform strictly greater then zero an and $P$ has integrable characteristic function, then the mapping:
\begin{align}
\Gamma:P\rightarrow \Phi_P
\end{align}
is injective and $\Phi_P$ is an element of the RKHS $H_k$ associated with $k$.
\end{lemma}
\begin{theorem}
Given $p\geq 1$, $k$ a translation invariant with an inverse Fourier transform strictly greater then zero, continuous, and integrable
kernel on $\mathbb{R}^d$,
$\Phi_P$ and $\Phi_Q$ the  smooth characteristic functions of the Borel probability measures with integrable characteristic functions $P$ and
$Q$ respectively, 
the following function:
\begin{equation}
d_{L^p,\Phi}(P,Q):=\left(\int_{t\in\mathbb{R}^d }\biggl|\Phi_P(t)-\Phi_Q(t)\biggr|^p d\Gamma(t)\right)^{1/p}
\end{equation}
where $\Gamma$ a Borel probability measure absolutely continuous with
respect to Lebesgue measure, is a metric on the space of Borel probability measures with integrable characteristic functions. Moreover if a sequence $(\alpha_n)_{n\geq 0}$ of Borel probability measures with integrable characteristic functions converges weakly towards $\alpha$ then  $d_{L^1,\mu}(\alpha_n,\alpha)\rightarrow 0$.
\end{theorem}

\begin{prv}
Let $p\geq 1$. First, as $\psi_P$ and $\psi_Q$  live in $H_k$, the RKHS associated with $k$, we have:
$$|\Phi_P\left(t\right)-\Phi_Q\left(t\right)|^p\leq ||\Phi_P-\Phi_Q||_H^p\text{ } k(0)^{p/2}$$
Let us prove now that if $P\neq Q$ then $d(P,Q)>0$. 
Thanks to Lemma \ref{lemma:continue-sten}, $\Phi_P$ and $\Phi_Q$ are continuous. Since the mapping $P\rightarrow \Phi_P$ is injective, there must exists at least one point $o$ where $\Phi_P-\Phi_Q$ is non-zero. By continuity of $\Phi_P-\Phi_Q$, there exists a ball around $o$ in which $\Phi_P-\Phi_Q$ is non-zero. Then $d_{L^p,\Phi}(P,Q)>0$. Moreover, all the other proprieties of a metric are clearly verified by this function. Let us now show that $d_{L^1,\Phi}$ admits a IPM formulation:
\begin{lemma}
\label{lemma:ipm_SCF}
Let $\mathcal{T}_{k}$ be the integral operator on $L_{2}^{d\Gamma}(\mathbb{R}^d)$ associated with the kernel $k$.
By denoting $B_{\infty}^{d\Gamma}$ the unit ball of $L_{\infty}^{d\Gamma}(\mathbb{R}^d)$, we have that:
\begin{align*}
d_{L^1,\Phi}&(P,Q)=\displaystyle\sup_{f\in \mathcal{L}\left(\mathcal{T}_k(B_{\infty}^{d\Gamma})\right)}
\biggl(\mathbb{E}_{P}[f(X)]-\mathbb{E}_{Q}[f(Y)]\biggr)
\end{align*}
where:
\begin{align}
\mathcal{L}(f)(x):=\int_{t\in\mathbb{R}^d} \exp(it^{T}x)f(t)dt
\end{align}
\end{lemma}
\begin{prv}
Let $P$ and $Q$ be Borel probability measures with integrable characteristic functions. As $\Phi_P$ and $\Phi_Q$ live in the RKHS associated with $k$, we obtain, as in the proof of Theorem \ref{thm:metricl1}, that:
\begin{align*}
d_{L^1,\Phi}&(P,Q)=\langle \Phi_P-\Phi_Q,f\rangle
\end{align*}
with
\begin{align*}
f=\int_{t\in \mathbb{R}^d}g_t d\Gamma(t)\quad\text{where}\quad
g_t=\left\{
              \begin{array}{ll}
              k_t \text{ if  } t\in\{x: \Phi_P\left(x\right)\geq\Phi_Q\left(x\right)\}\\
                \ -k_t  \text{ otherwise.}
                \end{array}\right.
\end{align*}
Therefore $f\in \mathcal{T}_k(B_{\infty}^{d\Gamma})$ and we have:
\begin{align*}
d_{L^1,\Phi}(P,Q)&=\int_{\mathbb{R}^d}\psi_P(t)f(t)dt-\int_{\mathbb{R}^d}\psi_Q(t)f(t)dt\\
&=\int_{t\in\mathbb{R}^d}\int_{\epsilon\in\mathbb{R}^d}\exp(i\epsilon^{T}t)f(t)dP(\epsilon)dt-\int_{t\in\mathbb{R}^d}\int_{\epsilon\in\mathbb{R}^d}\exp(i\epsilon^{T}t)f(t)dQ(\epsilon)dt
\end{align*}
Let us now show that for any $g\in B_{\infty}^{d\Gamma}$, $\mathcal{T}_k(g)$ is integrable (w.r.t the Lebesgue measure):
\begin{align*}
\int_{x\in\mathbb{R}^d}|\mathcal{T}_k(g)(x)|dx\leq \int_{x\in\mathbb{R}^d}\int_{t\in\mathbb{R}^d}|k(x,t)g(t)|d\Gamma(t)dx
\end{align*}
But as $k$ is translation-invariant we have:
\begin{align*}
\int_{t\in\mathbb{R}^d}\int_{x\in\mathbb{R}^d}|k(x,t)g(t)|d\Gamma(t)dx&=
\int_{x\in\mathbb{R}^d}\left(\int_{u\in\mathbb{R}^d}|k(u)|du\right)|g(t)| d\Gamma(t)\\
&= \int_{u\in\mathbb{R}^d}|k(u)|du \int_{x\in\mathbb{R}^d}|g(t)| d\Gamma(t)
\end{align*}
And as $k$ is integrable, and $g\in B_{\infty}^{d\Gamma}$, we can apply the Fubini–Tonelli theorem, and $\mathcal{T}_k(g)$ is integrable.

Therefore for any Borel probability measure $P$ with integrable characteristic function, $\int_{x\in\mathbb{R}^d}\int_{\epsilon\in\mathbb{R}^d}|f(t)|dP(\epsilon)dt<\infty$ and by Fubini–Tonelli theorem, we can rewrite $d_{L^1,\Phi}(P,Q)$ as:
\begin{align*}
d_{L^1,\Phi}(P,Q)=\int_{\epsilon\in\mathbb{R}^d}\left(\int_{t\in\mathbb{R}^d}\exp(i\epsilon^{T}t)f(t)dt\right)dP(\epsilon)-\int_{\epsilon\in\mathbb{R}^d}\left(\int_{t\in\mathbb{R}^d}\exp(i\epsilon^{T}t)f(t)dt\right)dQ(\epsilon)
\end{align*}
Therefore we have:
\begin{align*}
d_{L^1,\Phi}(P,Q)&=\int_{\epsilon\in\mathbb{R}^d}\mathcal{L}(f)(\epsilon)dP(\epsilon)-\int_{\epsilon\in\mathbb{R}^d}\mathcal{L}(f)(\epsilon)dQ(\epsilon)\\
&=\mathbb{E}_P(\mathcal{L}(f)(X))-\mathbb{E}_Q(\mathcal{L}(f)(Y))
\end{align*}
Let now $g$ be an abritary function in $B_{\infty}^{d\Gamma}$. Then we have:
\begin{align*}
\mathbb{E}_P(L(\mathcal{T}_k(g))(X))-\mathbb{E}_Q(L(\mathcal{T}_k(g))(Y))&=
\int_{\epsilon\in\mathbb{R}^d}\mathcal{L}(\mathcal{T}_k(g))(\epsilon)dP(\epsilon)-\int_{\epsilon\in\mathbb{R}^d}\mathcal{L}(\mathcal{T}_k(g))(\epsilon)dQ(\epsilon)
\end{align*}
But we have that:
\begin{align*}
\int_{\epsilon\in\mathbb{R}^d}\mathcal{L}(\mathcal{T}_k(g))(\epsilon)dP(\epsilon)&=\int_{\epsilon\in\mathbb{R}^d}\left(\int_{t\in\mathbb{R}^d}\exp(i\epsilon^{T}t)\mathcal{T}_k(g)(t)dt\right)dP(\epsilon)\\
&=\int_{t\in\mathbb{R}^d}\left(\int_{\epsilon\in\mathbb{R}^d}\exp(i\epsilon^{T}t)dP(\epsilon)\right)\mathcal{T}_k(g)(t)dt\\
&=\int_{\mathbb{R}^d}\psi_P(t)\mathcal{T}_k(g)(t)dt\\
&=\langle \Phi_P, \mathcal{T}_k(g)\rangle
\end{align*}
Finally we have:
\begin{align*}
\mathbb{E}_P(\mathcal{L}(\mathcal{T}_k(g))(X))-\mathbb{E}_Q(\mathcal{L}(\mathcal{T}_k(g))(Y))&=\langle \Phi_P-\Phi_Q, \mathcal{T}_k(g)\rangle\\
&=\int_{\mathbb{R}^d}g(t)(\Phi_P(t)-\Phi_Q(t))d\Gamma(t)\\
&\leq \int_{\mathbb{R}^d}|\Phi_P(t)-\Phi_Q(t)|d\Gamma(t)
\end{align*}
The results follows.
\end{prv}
Therefore thanks to the IPM formulation of the $d_{L^1,\Phi}$-distance,  we deduce that for all $p\geq 1$, if $\alpha_n$ converge weakly towards $\alpha$, then $d_{L^1,\Phi}(\alpha_n,\alpha)\rightarrow 0$. Indeed, we have shown that $\mathcal{T}_k(B_{\infty}^{d\Gamma})\subset L^1(\mathbb{R}^d)$, therefore $\mathcal{L}(\mathcal{T}_k(B_{\infty}^{d\Gamma}))\subset C^0(\mathbb{R}^d)$, and the result follows.
\end{prv}

\section{Two-sample testing using the $\ell_1$ norm}
\subsection{$\ell_1$-based random metric with mean embeddings}
\label{sec:randme}
\begin{definition}
\label{def:random-l1}
Let k be a kernel. For any $J > 0$, we define:
\begin{equation*}
d_{\ell_1,\mu,J}:=\left\{d_{\ell_1,\mu,J}\left[P,Q\right]=\frac{1}{J}\sum_{j=1}^{J}|\mu_P\left(T_j\right)-\mu_Q\left(T_j\right)|:
\quad
P,Q\in\mathcal{M}_{+}^{1}\left(\mathbb{R}^d\right)\right\}
\end{equation*}
with $\{T_j\}_{j=1}^{J}$ sampled independently from the distribution $\Gamma$.
\end{definition}
\begin{theorem}
\label{th:1}
Let k be a bounded, analytic, and characteristic kernel. 
Then for any $J > 0$, $d_{\ell_1,\mu,J}$ is a random metric on the space of Borel probability measures.
\end{theorem}

\begin{prv}
To prove this theorem we have first to introduce the fact that analytic functions are ’well behaved’.
\begin{lemma}
\label{lemma:1}
Let $\mu$ be absolutely continuous measure on $\mathbb{R}^d$ (wrt. the Lebesgue measure). Non-zero, analytic function f can be zero at most at the set of measure 0, with respect to the measure $\mu$.
\end{lemma}
\begin{prv}
If f is zero at the set with a limit point then it is zero everywhere. Therefore f can be zero at most at a set A without a limit point, which by definition is a discrete set $($distance between any two points in $A$ is greater then some $\epsilon > 0)$. Discrete sets have zero Lebesgue measure (as a countable union of points with zero measure). Since $\mu$ is absolutely continuous then $\mu(A)$ is zero as well.
\end{prv}

Let us now show how to build a random metric based on the $\ell_1$ norm.
\begin{lemma}
\label{lemma:2}
Let $\Lambda$ be an injective mapping from the space of the Borel probability measures into a space of analytic functions on $\mathbb{R}^d$. Define
$$d_{\Lambda,J}\left[P,Q\right]:=\frac{1}{J}\sum_{j=1}^{J}|\Lambda P\left(T_j\right)-\Lambda Q\left(T_j\right)|$$ with $\{T_j\}_{j=1}^{J}$ sampled independently from the distribution $\Gamma$.
\medbreak
Then $d_{\Lambda,J}$ is a random metric.
\end{lemma}

\begin{prv}
Let $\Lambda P$ and $\Lambda Q$ be images of measures P and Q respectively. We want to apply Lemma \ref{lemma:1} to the analytic function $f = \Lambda P-\Lambda Q$, with the measure $\Gamma$, to see that if $P\neq Q$ then $f \neq 0$ a.s. To do so, we need to show that $P\neq Q$ implies that f is non-zero. Since mapping to $\Gamma$ is injective, there must exists at least one point $o$ where f is non-zero. By continuity of f, there exists a ball around $o$ in which f is non-zero.
\medbreak
We have shown that $P\neq Q$ implies f is almost everywhere non zero which in turn implies that $d_{\Lambda,J} \left(P, Q\right) > 0$ a.s. If $P =Q$ then $f=0$ and $d_{\Lambda,J}\left(P, Q\right)=0$.
\medbreak
By the construction $d_{\Lambda,J}$ is clearly symmetric and satisfies the triangle inequality.
\end{prv}
Before proving the theorem we need to introduce a Lemma:
\begin{lemma}\cite{ME}
\label{lemma:analytic}
If k is a bounded, analytic kernel on $\mathbb{R}^d\times\mathbb{R}^d$, then all functions in the RKHS $H$ associated with this kernel are analytic.
\end{lemma}
Since k is characteristic the mapping $\Lambda : P \rightarrow \mu_P$ is injective. Since k is a bounded, analytic kernel on $\mathbb{R}^d\times\mathbb{R}^d$,the Lemma \ref{lemma:analytic} guarantees that $\mu_P$ is analytic, hence the image of $\Lambda$ is a subset of analytic functions. Therefore, we can use Lemma \ref{lemma:2} to see that $d_{\Lambda,J}\left[P,Q\right] = d_{\ell_1,\mu,J} \left[P, Q\right]$ is a random metric and this concludes the proof.
\end{prv}

\subsection{A first test with finite-sample control}
Let us now build a statistic based on an estimation of the random metric introduced in eq.\ref{eq:l1metric}. Let $X=\{x_1,...,x_{N_1}\}$\ and $Y=\{y_1,...,y_{N_2}\}\subset \mathbb{R}^d$ i.i.d.\ two
samples drawn respectively from the Borel probability measures $P$ and $Q$. From these samples we define their empirical mean embeddings $\mu_X$ and $\mu_Y$:
\begin{equation*}
\mu_X(T):=\frac{1}{N_1}\sum_{i=1}^{N_1} k(x_i,T),\qquad\quad
\mu_Y(T):=\frac{1}{N_2}\sum_{i=1}^{N_2} k(y_i,T)
\end{equation*}
And we define:
\begin{align}
\mathbf{S}_{N_1,N_2}:=
\biggl(\mu_X(T_1)-\mu_Y(T_1),..., \mu_X(T_J)-\mu_Y(T_J)\biggr)
\end{align}
with $\{T_j\}_{j=1}^J$ sampled independently from the distribution
$\Gamma$. Finally we define a first statistic:
\begin{align}
\label{eq:stat_finit}
d_{\ell_1,\mu,J}[X,Y]:=\frac{1}{J}\Vert {S}_{N_1,N_2}\Vert_1
\end{align}


We now derive a control of the statistic:
\begin{propo}
\label{prop:finit-control-prop}
With $K$ such that $\sup\limits_{x,y\in\mathbb{R}^d}|k(x,y)|\leq
\frac{K}{2}$,
\begin{align*}
\mathbb{P}_{X,Y}\left(\biggl|d_{\ell_1,\mu,J}[X,Y]-d_{\ell_1,\mu,J}[P,Q]\biggr|>t\right) 
\leq 2J\exp\left(\frac{-t^2 N_1 N_2}{2K^2(N_1+N_2)}\right)
\end{align*}
\end{propo}

\begin{prv}
We have:
\begin{equation*}
|d_{\ell_1,\mu,J}[X,Y]-d_{\ell_1,\mu,J}[P,Q]|\leq \frac{1}{J}\sum_{j=1}^J\bigg||\mu_X\left(T_j\right)-\mu_Y\left(T_j\right)|
-|\mu_p\left(T_j\right)-\mu_p\left(T_j\right)|\bigg|
\end{equation*}
Then:
\begin{equation*}
|d_{\ell_1,\mu,J}[X,Y]-d_{\ell_1,\mu,J}[P,Q]|\leq \frac{1}{J}\sum_{j=1}^J\bigg|\left(\mu_X\left(T_j\right)-\mu_Y\left(T_j\right)\right)
-\mathbb{E}_{X,Y\sim p,q}\left(\mu_X\left(T_j\right)-\mu_Y\left(T_j\right)\right)\bigg|
\end{equation*}
Let us now consider the upper bound of the difference. By applying a union bound we have:
\begin{align*}
\mathbb{P}\biggl(\frac{1}{J}\sum_{j=1}^J\bigg|\left(\mu_X\left(T_j\right)-\mu_Y(T_j)\right)
-\mathbb{E}_{X,Y}\left(\mu_X\left(T_j\right)-\mu_Y\left(T_j\right)\right)\bigg|\geq
t\biggr)
\qquad\qquad
\\
\qquad\qquad\leq \sum_{j=1}^J \mathbb{P}_{X,Y}\left(\frac{1}{J}\bigg|\left(\mu_X\left(T_j\right)-\mu_Y\left(T_j\right)\right)-\mathbb{E}_{X,Y}\left(..\right)\bigg|\geq \frac{t}{J}\right)
\end{align*}
Then by applying Hoeffding's inequality on each term of the sum of the right term of the inequality, we have:
\begin{equation*}
\mathbb{P}_{X,Y}\left(\frac{1}{J}\bigg|\left(\mu_X\left(T_j\right)-\mu_Y\left(T_j\right)\right)-\mathbb{E}_{X,Y}\left(..\right)\bigg|\geq \frac{t}{J}\right)
\leq 2\exp\left(-\frac{t^2 N_1N_2}{2K^2\left(N_1+N_2\right)}\right)
\end{equation*}
Finally we have:
\begin{equation*}
\mathbb{P}_{X,Y}\left(\bigg|d_{\ell_1,\mu,J}[X,Y]-d_{\ell_1,\mu,J}[P,Q]\bigg|\geq
t\right)
\leq 2J\exp\left(-\frac{t^2 N_1N_2}{2K^2\left(N_1+N_2\right)}\right)
\end{equation*}

\end{prv}

\begin{corol}
The hypothesis test associated with the statistic $d_{\ell_1,\mu,J}[X,Y]$ of level $\alpha$ for the null hypothesis $P=Q$, that is for $d_{\ell_1,\mu,J}[P,Q]=0$ almost surely, has almost surely the acceptance region: 
$$d_{\ell_1,\mu,J}[X,Y]<K\sqrt{\frac{N_1+N_2}{N_1N_2}}\sqrt{2\log\left(\frac{J}{\alpha}\right)}$$
Moreover, the test is consistent almost surely.
\end{corol}
\begin{proof}
Let us note the probability space of random variables $\{T_j\}_{j=1}^J$ as $\left(\Omega,\mathcal{F},P\right)$.

Let $\omega \in \Omega$ such that $d_{\ell_1,\mu,J}^{\omega}[P,Q]=0$. Then we have thanks to Proposition \ref{prop:finitcontrol} that:
$$d_{\ell_1,\mu,J}^{\omega}[X,Y]<K\sqrt{\frac{N_1+N_2}{N_1N_2}}\sqrt{2\log\left(\frac{J}{\alpha}\right)}$$
with a probability at last of $1-\alpha$.
\medbreak
By assuming the null hypothesis $P=Q$, we have thanks to Theorem \ref{th:1} that $d_{\ell_1,\mu,J}[P,Q]=0$ a.s., then the result above hold a.s.
\medbreak
Moreover the statistic converges in probability to its population value a.s which give us the consistency of the test a.s.
\end{proof}
We now show that, under the alternative hypothesis, the statistic captures dense differences between distributions with high probability:
\begin{corol}
Let $\gamma >0$, then under the alternative hypothesis, almost surely there exist $\Delta>0$ such that for all $N_1,N_2\geq 1$:
\begin{align*}
\mathbb{P}_{X,Y}&\left(\forall j\in \llbracket 1,J \rrbracket,  \frac{|\mu_X(T_j)-\mu_Y(T_j)|}{J}\geq \frac{\Delta}{J}-\omega_{N_1,N_2}\right) 
\geq 1-\gamma\notag\\
&\text{where }\omega_{N_1,N_2}=\frac{1}{J}\sqrt{\log(\frac{J^2}{\gamma})\frac{2K^2(N_1+N_2)}{N_1N_2}}\notag
\end{align*}
\end{corol}
\begin{proof}
Let $\Delta$ be the minimum of $\mu_p-\mu_q$ over the set of locations $\{T_j\}_{j=1}^J$. Thanks to the analycity of the kernel we have that under the alternative hypothesis, $\mu_p-\mu_q$ is non zero everywhere almost surely. Therefore $\Delta>0$ almost surely. Moreover by applying Proposition \ref{prop:finit-control-prop} for each $T_j$ we obtain that for all $N_1,N_2\geq 0$:
\begin{align*}
\mathbb{P}_{X,Y}&\left(\frac{|\mu_X(T_j)-\mu_Y(T_j)|}{J}\geq \frac{\Delta}{J}-\omega_{N_1,N_2}\right) 
\geq 1-\frac{\gamma}{J}\notag\\
&\text{where }\omega_{N_1,N_2}=\frac{1}{J}\sqrt{\log(\frac{J^2}{\gamma})\frac{2K^2(N_1+N_2)}{N_1N_2}}\notag
\end{align*}
Finally by applying an union bound, the result follows.
\end{proof}

\section{A test statistic with simple asymptotic distribution}
\subsection{Asymptotic distribution of $\widehat{d}_{\ell_1,\mu,J}[X,Y]$}
\label{sec:naka-coro}
\begin{propo}
\label{prop:naka-coro}
Let $\{T_j\}_{j=1}^J$ sampled independently from the distribution $\Gamma$ and $X:=\{x_i\}_{i=1}^{n}$ and 
$Y:=\{y_i\}_{i=1}^{n}$ be i.i.d. samples from $P$ and $Q$ respectively. Under $H_0$, the statistic $\widehat{d}_{\ell_1,\mu,J}[X,Y]$
is almost surely asymptotically distributed as a sum of $J$ correlated  Nakagami variables.
Finally under $H_1$, almost surely the statistic can be arbitrarily large as $n \rightarrow \infty$, allowing the test to correctly reject $H_0$.
\end{propo}

\begin{prv}
Let us note the probability space of random variables $\{T_j\}_{j=1}^J$ as $\left(\Omega,\mathcal{F},P\right)$.
Let $\omega \in \Omega$ such that $d_{\ell_1,\mu,J}^{\omega}[P,Q]=0$ (see Definition \ref{def:random-l1}) and let us define:
$$\mathbf{z}_{i}^{\omega}:=\left(k\left(x_i,T_1\left(\omega\right)\right)- k\left(y_j,T_1\left(\omega\right)\right),..., k\left(x_i,T_J\left(\omega\right)\right)-k\left(y_j,T_J\left(\omega\right)\right)\right)\in\mathbb{R}^J$$
Therefore we can define:
\begin{align*}
\mathbf{S}_n:=\frac{1}{n}\sum_{i=1}^n \mathbf{z}_{i}^{\omega}
\end{align*}
By applying the Central-Limit Thoerem, we have:
\begin{align*}
\sqrt{n}\mathbf{S}_n\rightarrow \mathcal{N}\left(0, \mathbf{\Sigma}^{\omega}\right)\quad\quad \text{with} \quad\quad \mathbf{\Sigma}^{\omega}:=\text{Cov}(\mathbf{z}^{\omega})
\end{align*}

Therefore $\widehat{d}_{\ell_1,\mu,J}^{\omega}[X,Y]=\Vert\sqrt{n}\,\mathbf{S}_n^{\omega}\Vert_1 $ converges to a sum of correlated Nakagami variables. But under, the null hypothesis $P=Q$, we have thanks to Theorem \ref{th:1} that $d_{\ell_1,\mu,J}[P,Q]=0$ a.s., then a.s. $\widehat{d}_{\ell_1,\mu,J}$ converges to a sum of correlated Nakagami variables. Let's now consider an $\omega$ such that $d_{\ell_1,\mu,J}^{\omega}[P,Q]>0$. Since  $\mathbf{S}_{n}^{\omega}$ converges in probability to the vector
$\mathbf{S}^{\omega}=\mathbb{E}\left(\mathbf{z}^{\omega}\right)\neq 0$, then we have:
\begin{equation*}
\mathbb{P}\left(
\left\|\sqrt{n}\mathbf{S}_{n}^{\omega}\right\|_{1}>r\right)
=\mathbb{P}\left(
\left\|\mathbf{S}_{n}^{\omega}\right\|_{1}-\frac{r}{\sqrt{n}}>0\right)
\end{equation*}
And as $\frac{r}{\sqrt{t}}\rightarrow 0$ as $t\rightarrow \infty$, we have finally:
$$\mathbb{P}\left(
\left\|\sqrt{n}\mathbf{S}_{n}^{\omega}\right\|_{1}>r\right)\rightarrow
1 \quad\text{as}\quad t\rightarrow \infty.$$
\medbreak
Finally, under $H_1$, $d_{\ell_1,\mu,J}[P,Q] > 0$ almost surely and the statistic can be arbitrarily large as $n \rightarrow \infty$ almost surely.
\end{prv}

\subsection{Proof of Proposition \ref{prop:finalement}}
\label{sec:finalement_ME}
\begin{propo}
Let $\alpha\in ]0,1[$, $\gamma>0$ and $J\geq 2$. Let $\{T_j\}_{j=1}^J$
sampled i.i.d. from the distribution $\Gamma$ and let $X:=\{x_i\}_{i=1}^{n}$ and 
$Y:=\{y_i\}_{i=1}^{n}$ i.i.d. samples from $P$ and $Q$ respectively. Let us denote $\delta$ the $(1-\alpha)$-quantile of the asymptotic null distribution of $\widehat{d}_{\ell_1,\mu,J}[X,Y]$ and $\beta$ the $(1-\alpha)$-quantile of the asymptotic null distribution of $\widehat{d}_{\ell_2,\mu,J}^2[X,Y]$. Under the alternative hypothesis, almost surely, there exists $N\geq 1$ such that for all $n\geq N$, with a probability of at least $1-\gamma$ we have:
\begin{align}
\widehat{d}^2_{\ell_2,\mu,J}[X,Y] >\beta \Rightarrow \widehat{d}_{\ell_1,\mu,J}[X,Y]>\delta 
\end{align}
\end{propo}

\begin{prv}
First we remarks that:
\begin{align*}
\widehat{d}^2_{\ell_2,\mu,J}[X,Y]=\Vert \sqrt{n} \mathbf{S}_n\Vert_2^2
\end{align*}
and
\begin{align*}
\widehat{d}_{\ell_1,\mu,J}[X,Y]=\Vert \sqrt{n} \mathbf{S}_n\Vert_1
\end{align*}
where $\mathbf{S}_n:=\frac{1}{n}\sum_{i=1}^n \mathbf{z}_{i}^{\omega}$ and $\mathbf{z}_{i}:=\left(k\left(x_i,T_1\left(\omega\right)\right)- k\left(y_j,T_1\left(\omega\right)\right),..., k\left(x_i,T_J\left(\omega\right)\right)-k\left(y_j,T_J\left(\omega\right)\right)\right)$.
Let us now introduce the following Lemma:
\begin{lemma}
\label{lem:betterpow} 
Let $\mathbf{x}$ a random vector $\in \mathbb{R}^J$ with $J\geq 2$,  $\mathbf{z}:= \min\limits_{j\in[|1,J|]}|x_j|$, $\epsilon>0$ and $\gamma>0$. If
\begin{align*}
\mathbb{P}(\mathbf{z}\geq\epsilon)\geq 1-\gamma
\end{align*}
we have with a probability of
at least $1-\gamma$ that, $\forall t_1\geq t_2\geq 0$,
if $\epsilon\geq \sqrt{\frac{t_1^2-t_2^2}{J(J-1)}}$, then
$$\Vert \mathbf{x}\Vert_2>t_2\Rightarrow \Vert \mathbf{x}\Vert_1>t_1.$$
\end{lemma}

\begin{prv}
First we remarks that:
\begin{align*}
\epsilon>\sqrt{\frac{t_1^2-t_2^2}{J\left(J-1\right)}} \Rightarrow & J\left(J-1\right)\epsilon>t_1^2-t_2^2\\
\Rightarrow & t_2^2>t_1^2-J\left(J-1\right)\epsilon^2
\end{align*}
Therefore, we have:
\begin{align*}
\|\mathbf{x}\|_2\geq t_2 \Rightarrow & \|\mathbf{x}\|_{2}^2+J\left(J-1\right)\epsilon^2\geq t_1^2\\
\Rightarrow & \sqrt{\|\mathbf{x}\|_{2}^2+J\left(J-1\right)\epsilon^2}\geq t_1
\end{align*}
But we have that:
$$\|\mathbf{x}\|_{1}^2=\sum_{i=1}^J |\mathbf{x}_i|^2 +\sum_{i\neq
j}|\mathbf{x}_i||\mathbf{x}_j|$$
Therefore we have with a probability of 1-$\gamma$ that:
$$\|\mathbf{x}\|_{1}^2\geq \|\mathbf{x}\|_{2}^2+J\left(J-1\right)\epsilon^2$$
And:
$$\|\mathbf{x}\|_2\geq t_2\Rightarrow \|\mathbf{x}\|_1\geq t_1$$
\end{prv}
Moreover by denoting $\delta$ the $(1-\alpha)$-quantile of the asymptotic null distribution of $\widehat{d}_{\ell_1,\mu,J}[X,Y]$ and $\beta$ the $(1-\alpha)$-quantile of the asymptotic null distribution of $\widehat{d}_{\ell_2,\mu,J}^2[X,Y]$ we have that $\delta\geq\sqrt{\beta}$: 
\begin{lemma}
\label{lemma:quantile}
Let $\mathbf{x}$ be a random vector in $\mathbb{R}^J$, $\delta$  the $(1-\alpha)$-quantile of $\Vert \mathbf{x} \Vert_1$ and $\beta$ the $(1-\alpha)$-quantile of $\Vert \mathbf{x} \Vert_2$. We have then:
\begin{align}
\delta\geq \beta \geq 0\text{.}
\end{align}
\end{lemma}
\begin{prv}
The results is a direct consequence of the domination of the $\ell_1$ norm: $$\|\mathbf{x}\|_1\geq \|\mathbf{x}\|_2$$
\end{prv}
Indeed, under $H_0$, we have shown that (see proof Proposition \ref{prop:naka-coro}):
\begin{align*}
\sqrt{n}\mathbf{S}_n\rightarrow \mathcal{N}\left(0, \mathbf{\Sigma}^{\omega}\right)\quad\quad \text{with} \quad\quad \mathbf{\Sigma}:=\text{Cov}(\mathbf{z})
\end{align*}
Therefore by applying the Lemma \ref{lemma:quantile}  to $\mathbf{x}$ which  follows $\mathcal{N}\left(0, \mathbf{\Sigma}^{\omega}\right)$, we obtain that $\delta\geq\sqrt{\beta}$. Now, To show the result we only need to show that the assumption of the Lemma \ref{lem:betterpow} is sastified for the random vector $\mathbf{x}:=\sqrt{n}\mathbf{S}_n$, $t_1=\delta$ and $t_2=\sqrt{\beta}$, i.e. for $\epsilon=\sqrt{\frac{\delta^2-\beta}{J(J-1)}}$ under the alternative hypothesis.
Under $H_1: P\neq Q$, we have that $\mathbf{S}_n$ converge in probability to $\mathbf{S}:=\mathbb{E}_{(x,y)\sim (P,Q)}(\mathbf{S}_{n})$. Then by continuity of the application:
$$\phi_j:x:=(x_j)_{j=1}^J\mathbb{R}^J\rightarrow |x_j|$$, we have that for all $j\in[|1,J|]$,
$|(\mathbf{S}_n)_j|$ converges in probability towards $\mathbf{S}_j$, the $j$-th coordinate of $\mathbf{S}$. Since $\mathbf{S}=(\mu_P(T_j)-\mu_Q(T_j))_{j=1}^J$, thanks to the analycity of the kernel $k$, the Lemma \ref{lemma:analytic} guarantees the analycity of $\mu_P-\mu_Q$. And thanks to the injectivity of the mean embedding function, $\mu_P-\mu_Q$ is a  non-zero function, therefore thanks to Lemma \ref{lemma:1} $\mu_P-\mu_Q$ is non zero almost everywhere. Moreover the $(T_j)_{j=1}^J$ are independent, therefore the coordinates of $\mathbf{S}$ are almost surely all nonzero. Then we have
then for all $j\in\llbracket 1,J \rrbracket$:
\begin{equation*}
\mathbb{P}\left(
\bigg|(\sqrt{n}\mathbf{S}_{n})_j\bigg|>\epsilon\right)
=\mathbb{P}\left(
\bigg|(\mathbf{S}_{n})_j\bigg|-\frac{\epsilon}{\sqrt{n}}>0\right)
\end{equation*}
And as $\frac{\epsilon}{\sqrt{n}}\rightarrow 0$ as $n\rightarrow \infty$, we
have finally almost surely for all $j\in\llbracket 1,J \rrbracket$:
\begin{align*}
\mathbb{P}_{X,Y}\left(\bigg|(\sqrt{n}\mathbf{S}_{n})_{j}\bigg|\geq \epsilon\right)\rightarrow 1 \text{  as  }n\rightarrow \infty
\end{align*}
Therefore almost surely there exist $N\geq 1$ such that for all $n \geq N$ and for all $j\in\llbracket 1,J \rrbracket$:
\begin{align*}
\mathbb{P}_{X,Y}\left(\bigg|(\sqrt{n}\mathbf{S}_{n})_{j}\bigg|\geq \epsilon\right)\geq 1-\frac{\gamma}{J}
\end{align*}
Finally by applying a union bound we obtain that almost surely, for all $n \geq N$:
\begin{align*}
\mathbb{P}_{X,Y}\left(\forall j\in[|1,J|]\text{,  }\bigg|(\sqrt{ n}\mathbf{S}_{n})_{j}\bigg|\geq \epsilon\right)\geq 1-\gamma
\end{align*}
Therefore by applying Lemma \ref{lem:betterpow}, we obtain that, almost surely, for all $n \geq N$, with a probability of at least $1-\gamma$:
\begin{align*}
\Vert \sqrt{n}\mathbf{S}_n\Vert_2 >\sqrt{\beta} \Rightarrow \Vert \sqrt{n}\mathbf{S}_n\Vert_1>\delta.
\end{align*}

\end{prv}

\subsection{Proof of the Proposition \ref{prop:asympme}}
\label{sec:prop_asymp}
\begin{propo}
Let $\{T_j\}_{j=1}^J$ sampled independently from the distribution $\Gamma$ and $X:=\{x_i\}_{i=1}^{N_1}$ and 
$Y:=\{y_i\}_{i=1}^{N_2}$ be i.i.d. samples from $P$ and $Q$ respectively. Under $H_0$, the statistic $\text{L1-ME}[X,Y]$
is almost surely asymptotically distributed as Naka$(\frac{1}{2},1,J)$, a sum of $J$ random variables i.i.d which follow a Nakagami distribution of parameter $m=\frac{1}{2}$ and $\omega=1$.
Finally under $H_1$, almost surely the statistic can be arbitrarily large as $t \rightarrow \infty$, allowing the test to correctly reject $H_0$.
\end{propo}

\begin{prv}
Let us note the probability space of random variables $\{T_j\}_{j=1}^J$ as $\left(\Omega,\mathcal{F},P\right)$.
Let $\omega \in \Omega$ such that $d_{\ell_1,\mu,J}^{\omega}[P,Q]=0$ (see Definition \ref{def:random-l1}).
Let us denote:
$$\mathbf{Z}_{X}^{i,\omega}:=\left(k\left(x_i,T_1\left(\omega\right)\right),..., k\left(x_i,T_J\left(\omega\right)\right)\right)
\quad\quad
\mathbf{Z}_{Y}^{j,\omega}:=\left(k\left(y_j,T_1\left(\omega\right)\right),..., k\left(y_j,T_J\left(\omega\right)\right)\right),$$ 
$$\mathbf{S}_{N_1,N_2}^{\omega}:=
\frac{1}{N_1}\sum_{i=1}^{N_1}\mathbf{Z}_{X}^{i,\omega}-\frac{1}{N_2}\sum_{j=1}^{N_2}
\mathbf{Z}_{Y}^{j,\omega}.$$
As $d_{\ell_1,\mu,J}^{\omega}[P,Q]=0$ then for all $j$,
$\mu_p\left(T_j\left(\omega\right)\right)
=\mu_q\left(T_j\left(\omega\right)\right)$, which implies that
$\mathbb{E}\left(\mathbf{Z}_{X}^{i,\omega}\right) =
\mathbb{E}\left(\mathbf{Z}_{Y}^{j,\omega}\right)$.
Therefore, by applying the Central-Limit Theorem, we have:
\begin{equation*}
\sqrt{t}\,\mathbf{S}_{N_1,N_2}^{\omega}\longrightarrow\mathcal{N}\left(0,\frac{\mathbf{\Sigma}_{1}^{\omega}}{\rho}\right)-\mathcal{N}\left(0,\frac{\mathbf{\Sigma}_{2}^{\omega}}{1-\rho}\right)
\qquad\text{with}\quad\mathbf{\Sigma}_1=\text{Cov}\left(\mathbf{Z}_X^{\omega}\right)
\quad\text{and}\quad\mathbf{\Sigma}_2=\text{Cov}\left(\mathbf{Z}_Y^{\omega}\right)
\end{equation*}
As $\mathbf{Z}_X^{\omega}$ and $\mathbf{Z}_Y^{\omega}$ are independent, we have then that:
$$\sqrt{t}\,\mathbf{S}_{N_1,N_2}^{\omega}\longrightarrow\mathcal{N}\left(0,\mathbf{\Sigma}^{\omega}\right)
\qquad
\text{with}\quad\mathbf{\Sigma}^{\omega}=\frac{\mathbf{\Sigma}_1^{\omega}}{\rho}+\frac{\mathbf{\Sigma}_2^{\omega}}{1-\rho}$$
And by Slutsky's theorem we deduce that:
$$\sqrt{t}\left(\mathbf{\Sigma}_{N_1,N_2}^{\omega}\right)^{-\frac{1}{2}}\mathbf{S}_{N_1,N_2}^{\omega}\longrightarrow\mathcal{N}\left(0,\mathbf{I}\right)$$
So by noting,
$\sqrt{t}\left(\mathbf{\Sigma}_{N_1,N_2}^{\omega}\right)^{-\frac{1}{2}}\mathbf{S}_{N_1,N_2}^{\omega}=\left(W_{N_1,N_2}^{1,\omega},...,W_{N_1,N_2}^{J,\omega}\right)$, we have that for each coordinate:
$$\left(W_{N_1,N_2}^{j,\omega}\right)\longrightarrow \mathbf{S}_{j}^{\omega}$$
where $\left(\mathbf{S}_{j}^{\omega}\right)$ are i.i.d and follow a standard normal distribution.
Therefore by considering the $\ell_1$ norm of the statistic we have that:
$$||\sqrt{t}\left(\mathbf{\Sigma}_{N_1,N_2}^{\omega}\right)^{-\frac{1}{2}}\mathbf{S}_{N_1,N_2}^{\omega}||_{1}\longrightarrow\sum_{j=1}^{J}
|\mathbf{S}_{j}^{\omega}|$$ where $\left(\mathbf{S}_j^{\omega}\right)$ are independent and
$\mathbf{S}_{j}^{\omega}\sim Naka\left(\frac{1}{2},1\right)$.
And by assuming the null hypothesis $P=Q$, we have thanks to Theorem \ref{th:1} that $d_{\ell_1,\mu,J}[P,Q]=0$ a.s., then the result above hold a.s.
Moreover, let's consider an $\omega$ such that $d_{\ell_1,\mu,J}^{\omega}[P,Q]>0$. 
First we need show that
$\left(\mathbf{\Sigma}_{N_1,N_2}^{\omega}\right)^{-\frac{1}{2}}$
converges in probability to the positive definite matrix
$\left(\mathbf{\Sigma}^{\omega}\right)^{-\frac{1}{2}}$. 
For that we need to prove the following:
\begin{lemma}
\label{lemma:semi-inv}
The function $h\left(\mathbf{X}\right)=\mathbf{X}^{-\frac{1}{2}}$ is well defined on
$\mathcal{S}_J^{++}\left(R\right)$ and is continuous.
\end{lemma}
\begin{prv}
First we observe that h is the composition of two function which are:
\begin{itemize}
\item $h_1\left(\mathbf{X}\right)=\mathbf{X}^{-1}$ which is well defined and continuous on
$\mathcal{S}_J^{++}\left(R\right)$
\item $h_2\left(\mathbf{X}\right)=\mathbf{X}^{\frac{1}{2}}$ which is well defined on
$\mathcal{S}_J^{+}\left(R\right)$ because each matrix of
$\mathcal{S}_J^{+}\left(R\right)$ admits a unique square root matrix on
$\mathcal{S}_J^{+}\left(R\right)$, so the result hold on $\mathcal{S}_J^{++}\left(R\right)$.
\end{itemize}
Let us prove now the continuity of $h_2$.
Let $\left(\mathbf{U}_n\right)$ a sequence in
$\mathcal{S}_n^{++}\left(R\right)$ such that $\mathbf{U}_n\rightarrow
\mathbf{U}$ and 
let us prove that $h_2\left(\mathbf{U}_n\right)\rightarrow
h_2\left(\mathbf{U}\right)$ to prove the continuity of $h_2$.
As $\left(\mathbf{U}_n\right)$ converges, then $\left(\mathbf{U}_n\right)$ is bounded, and we have:
$$|||\mathbf{U}_n|||\leq K \Longrightarrow
|||h_2\left(\mathbf{U}_n\right)|||=\sqrt{|||\mathbf{U}_n|||}\leq \sqrt{K}$$
Then $\left(h_2\left(\mathbf{U}_n\right)\right)$ is bounded.
Let us show now that:
$\forall \mathbf{A}$ s.t $\exists$ $\phi$ strictly increasing and
$h_2\left(\mathbf{U}_{\phi\left(n\right)}\right)\rightarrow \mathbf{A}$ we have
$\mathbf{A}=h_2\left(\mathbf{U}\right)$.
Let $\mathbf{A}$ defined as above.
Then $\exists$ $\phi$ strictly increasing such that
$h_2\left(\mathbf{U}_{\phi\left(n\right)}\right)\rightarrow \mathbf{A}$.
As $\mathcal{S}_n^{+}\left(R\right)$ is closed,
$\mathbf{A}\in\mathcal{S}_n^{+}\left(R\right)$, and by continuity of
$\mathbf{M}\rightarrow \mathbf{M}^2$ we have also that $\mathbf{U}_{\phi\left(n\right)}\rightarrow
\mathbf{A}^2$.
And as $\mathbf{U}_{n}\rightarrow \mathbf{U}$, we have
$\mathbf{A}^2=\mathbf{U}$. And by uniqueness, we have finally: 
$$h_2\left(\mathbf{U}\right)=\mathbf{A}.$$ 
So $h_2$ est continuous, and that conclude the proof.
\end{prv}

Then each entry of the matrix $\mathbf{\Sigma}_{N_1,N_2}^{\omega}$
converges to the matrix $\mathbf{\Sigma}^{\omega}$, hence entires of the
matrix $\left(\mathbf{\Sigma}^{\omega}\right)^{-\frac{1}{2}}$, given by a
continuous function of the entries of $\mathbf{\Sigma}^{\omega}$, are
limit of the sequence $\left(\mathbf{\Sigma}_{N_1,N_2}^{\omega}\right)^{-\frac{1}{2}}$.
\medbreak
Similarly $\mathbf{S}_{N_1,N_2}^{\omega}$ converges in probability to the vector
$\mathbf{S}^{\omega}=\mathbb{E}\left(\mathbf{Z}^{1,\omega}\right)-\mathbb{E}\left(\mathbf{Z}^{2,\omega}\right)\neq
0$ . Since
$\|\left(\mathbf{\Sigma}^{\omega}\right)^{-\frac{1}{2}}\mathbf{S}^{\omega}\|_{1}=
\mathbf{A}_{\omega}> 0$ (indeed  $\left(\mathbf{\Sigma}^{\omega}\right)^{-\frac{1}{2}}$ is
positive definite), then
$\|\left(\mathbf{\Sigma}_{N_1,N_2}^{\omega}\right)^{-\frac{1}{2}}\mathbf{S}_{N_1,N_2}^{\omega}\|_{1}$, being a continuous function of the entries
of $\mathbf{S}_{N_1,N_2}^{\omega}$ and
$\left(\mathbf{\Sigma}_{N_1,N_2}^{\omega}\right)^{-\frac{1}{2}}$,
converges to $\mathbf{A}_{\omega}$.
Then 
\begin{equation*}
\mathbb{P}\left(
\left\|\sqrt{t}(\mathbf{\Sigma}_{N_1,N_2}^{\omega})^{-\frac{1}{2}}\mathbf{S}_{N_1,N_2}^{\omega}\right\|_{1}>r\right)
=\mathbb{P}\left(
\left\|\left(\mathbf{\Sigma}_{N_1,N_2}^{\omega}\right)^{-\frac{1}{2}}\mathbf{S}_{N_1,N_2}^{\omega}\right\|_{1}-\frac{r}{\sqrt{t}}>0\right)
\end{equation*}
And as $\frac{r}{\sqrt{t}}\rightarrow 0$ as $t\rightarrow \infty$, we have finally:
$$\mathbb{P}\left(
\left\|\sqrt{t}\left(\mathbf{\Sigma}_{N_1,N_2}^{\omega}\right)^{-\frac{1}{2}}\mathbf{S}_{N_1,N_2}^{\omega}\right\|_{1}>r\right)\rightarrow
1 \quad\text{as}\quad t\rightarrow \infty.$$
\medbreak
Finally, since $d_{\ell_1,\mu,J}[P,Q] > 0$ almost surely then
$\mathbb{E}\left(\mathbf{Z}^{1,\omega}\right)-\mathbb{E}\left(\mathbf{Z}^{2,\omega}\right)\neq 0$  for almost all $\omega \in \Omega_1$, therefore under $H_1$, the statistic can be arbitrarily large as $t \rightarrow \infty$ almost surely.
\end{prv}

\section{Optimizing test locations to improve power}
\subsection{Proof of Proposition \ref{prop:optpower}}
\label{sec:improve}
\begin{prop}
Let $\mathcal{K}$ be a uniformly bounded family of  $k : \mathbb{R}^d
\times\mathbb{R}^d  \rightarrow \mathbb{R}$ measurable kernels (i.e., $\exists$ $K<\infty$ such
that $\sup \limits_{k\in \mathcal{K}} \sup \limits_{(x,y)\in
(\mathbb{R}^{d})^2} |k(x,y)| \leq K$). Let
$\mathcal{V}$ be a collection in which each element is a set of J test
locations. Assume that 
$c:= \sup \limits_{V\in \mathcal{V},k\in \mathcal{K}}\Vert\mathbf{\Sigma}^{-1/2}\Vert < \infty$. Then the test power $\mathbb{P}\left(\widehat{\lambda}_{t} \geq \delta\right)$ of the L1-ME test satisfies $\mathbb{P}\left(\widehat{\lambda}_{t} \geq \delta\right)\geq L(\lambda_{t})$ where:
\begin{align*}
L(\lambda_{t}) &=  1-2\sum
\limits_{k=1}^J\exp\left(-\left(\frac{\lambda_{t}-\delta}{J^2+J}\right)^2\frac{\gamma_{N_{\!1}\!,N_2}N_1N_2}{(N_{\!1}+N_2)^2}\right)\\
 &-2\!\sum \limits_{k,q=1}^J
\!\exp\!\left(-2\frac{\left(\frac{\gamma_{N_{\!1}\!,N_2}}{K_3J^2}\frac{\lambda_{t}-\delta}{(J^2+J)\sqrt{t}}-\frac{J^3K_2}{\sqrt{\gamma_{N_{\!1}\!,N_2}}}-J^4K_1\right)^{\!2}}{K_{\lambda}^2
(N_1 + N_2)\max\left(\frac{8}{\rho
N_1} ,\frac{8}{(1-\rho) N_ 2}\right)^{\!2}
}\right)
\end{align*}
and $K_1,K_2$, $K_3$ and $K_\lambda$, are  positive constants depending on only $K$,
$J$ and $c$. The parameter
$\lambda_{t}:=\Vert\sqrt{t}\mathbf{\Sigma}^{-\frac{1}{2}}\mathbf{S}\Vert_{1}$ is the
population counterpart of
$\widehat{\lambda}_{t}:=\|\sqrt{t}\mathbf({\Sigma}_{N_1,N_2}+\gamma_{N_1,N_2}\mathbf{I})^{-\frac{1}{2}}\mathbf{S}_{N_1,N_2}\|_{1}$ where $\mathbf{S}=\mathbb{E}_{x,y}(S_{N_1,N_2})$ and $\mathbf{\Sigma}=\mathbb{E}_{x,y}(\mathbf{\Sigma}_{N_1,N_2})$. Moreover for large $t$, $L(\lambda_{t})$ is increasing in $\lambda_{t}$.
\end{prop}

\begin{prv}
We will first find an upper bound of $|\widehat{\lambda} _{t}-\lambda_{t}|$, 
then we will compute a lower bound of $\mathbb{P}\left(\widehat{\lambda} _{t}>\delta\right)$. To simplify the notation In the following, we denote:
\begin{align}
\mathbf{\Sigma}_{N_1,N_2}:=&\frac{\mathbf{\Sigma}_{N_1}}{\rho}+\frac{\mathbf{\Sigma}_{N_2}}{1-\rho}+\gamma_{N_1,N_2}\mathbf{I}
\end{align}
such that $\widehat{\lambda}_{t}:=\|\sqrt{t}\mathbf({\Sigma}_{N_1,N_2})^{-\frac{1}{2}}\mathbf{S}_{N_1,N_2}\|_{1}$.
We have:
$$|\widehat{\lambda} _{N_1,N_2}-\lambda_{t}|=\left|\sqrt{t}\left(
\|\mathbf{\Sigma}_{N_1,N_2}^{-\frac{1}{2}}\mathbf{S}_{N_1,N_2}\|_{1}-\|\mathbf{\Sigma}^{-\frac{1}{2}}\mathbf{S}\|_{1}\right)\right|$$
Then we have:
\begin{align*}
\left|
\|\mathbf{\Sigma}_{N_1,N_2}^{-\frac{1}{2}}\mathbf{S}_{N_1,N_2}\|_{1}-\|\mathbf{\Sigma}^{-\frac{1}{2}}\mathbf{S}\|_{1}\right|
&\leq
\|\mathbf{\Sigma}_{N_1,N_2}^{-\frac{1}{2}}\mathbf{S}_{N_1,N_2}-\mathbf{\Sigma}^{-\frac{1}{2}}\mathbf{S}\|_{1} \\
&\leq
\|\mathbf{\Sigma}_{N_1,N_2}^{-\frac{1}{2}}\mathbf{S}_{N_1,N_2}-\mathbf{\Sigma}_{N_1,N_2}^{-\frac{1}{2}}\mathbf{S}+\mathbf{\Sigma}_{N_1,N_2}^{-\frac{1}{2}}\mathbf{S}-\mathbf{\Sigma}^{-\frac{1}{2}}\mathbf{S}\|_{1}\\
&\leq
\|\mathbf{\Sigma}_{N_1,N_2}^{-\frac{1}{2}}\left(\mathbf{S}_{N_1,N_2}-\mathbf{S}\right)\|_{1}+\|\left(\mathbf{\Sigma}_{N_1,N_2}^{-\frac{1}{2}}-\mathbf{\Sigma}^{-\frac{1}{2}}\right)\mathbf{S}\|_{1}
\end{align*}
Let us now consider the first term on the right side of the inequality:
$$\|\mathbf{\Sigma}_{N_1,N_2}^{-\frac{1}{2}}\left(\mathbf{S}_{N_1,N_2}-\mathbf{S}\right)\|_{1}=\sum_{j=1}^J|\mathbf{\Sigma}_{N_1,N_2}^{-\frac{1}{2}}\left(\mathbf{S}_{N_1,N_2}-\mathbf{S}\right)|_j$$
But since $\mathbf{\Sigma}_{N_1,N_2}$ is symmetric definite positive, we can write:
$$\mathbf{\Sigma}_{N_1,N_2}=\mathbf{U} \mathbf{D} \mathbf{U}^T$$ where
$\mathbf{U}$ is orthogonal and $\mathbf{D}=\text{diag}\left(\lambda_i\right)$ with $\lambda_i>0$.
So:
$$\mathbf{\Sigma}_{N_1,N_2}^{-\frac{1}{2}}=\mathbf{U}
\mathbf{D}^{-\frac{1}{2}} \mathbf{U}^T$$

But the regularization of
$\mathbf{\Sigma}_{N_1,N_2}=(\frac{\mathbf{\Sigma}_{N_1}}{\rho}+\frac{\mathbf{\Sigma}_{N_2}}{1-\rho}+\gamma_{N_1,N_2}
\mathbf{I}$) ensure that $\lambda_i\geq \gamma_{N_1,N_2}$.
Thus $\lambda_{i}^{-\frac{1}{2}}\leq \gamma_{N_1,N_2}^{-\frac{1}{2}}$, and we have now:
$$\left|[\mathbf{\Sigma}_{N_1,N_2}^{-\frac{1}{2}}]_{i,j}\right|=\left|\sum_{j=1}^J
\lambda_{j}^{-\frac{1}{2}} \left(\mathbf{U}_k\right)_i\left(\mathbf{U}_k\right)_j\right|$$
where $\mathbf{U}=[\mathbf{U}_1,...,\mathbf{U}_J]$ and $\|\mathbf{U}_k\|_2=1$.
And finally:
$$\left|[\mathbf{\Sigma}_{N_1,N_2}^{-\frac{1}{2}}]_{i,j}\right|\leq \frac{J}{\sqrt{\gamma_{N_1,N_2}}}$$
Now we have:
\begin{align*}
\left\|\mathbf{\Sigma}_{N_1,N_2}^{-\frac{1}{2}}\left(\mathbf{S}_{N_1,N_2}-\mathbf{S}\right)\right\|_{1}
&\leq\sum_{j=1}^J\left|\sum_{k=1}^J[\mathbf{\Sigma}_{N_1,N_2}^{-\frac{1}{2}}]_{j,k}\left(\mathbf{S}_{N_1,N_2}-\mathbf{S}\right)_k\right| \\
& \leq
\frac{J^2}{\sqrt{\gamma_{N_1,N_2}}}\sum_{k=1}^J|\left(\mathbf{S}_{N_1,N_2}-\mathbf{S}\right)_k|\\
& \leq \frac{J^2}{\sqrt{\gamma_{N_1,N_2}}} \sum_{k=1}^J \left|\mu_{X}\left(T_k\right)-\mu_{Y}\left(T_k\right)
-\mathbb{E}\left(\mu_{X}\left(T_k\right)-\mu_{Y}\left(T_k\right)\right)\right|
\end{align*}
Let us note
$\frac{\mathbf{\Sigma}_{N_1}}{\rho}+\frac{\mathbf{\Sigma}_{N_2}}{1-\rho}=\mathbf{M}_{N_1,N_2}$
and consider the second term of the inequality:
\begin{align*}
   \mathbf{\Sigma}_{N_1,N_2}^{-\frac{1}{2}}-\mathbf{\Sigma}^{-\frac{1}{2}}
= &
\left(\mathbf{M}_{N_1,N_2}+\gamma_{N_1,N_2}\mathbf{I}\right)^{-\frac{1}{2}}-\mathbf{\Sigma}^{-\frac{1}{2}}\\
    =  &\left[\left(\mathbf{M}_{N_1,N_2}+\gamma_{N_1,N_2}\mathbf{I}\right)^{-\frac{1}{2}}
   -\left(\mathbf{\Sigma}+\gamma_{N_1,N_2}\mathbf{I}\right)^{-\frac{1}{2}}\right]
   +\left[\left(\mathbf{\Sigma}+\gamma_{N_1,N_2}\mathbf{I}\right)^{-\frac{1}{2}}-\mathbf{\Sigma} \right] \\
    = & \left(1\right)+\left(2\right)
\end{align*}
Let us first consider $(1)$:
\begin{align*}
(1)=&\mathbf{\Sigma}_{N_1,N_2}^{-\frac{1}{2}}\left(\left(\mathbf{\Sigma}+\gamma_{N_1,N_2}\mathbf{I}\right)^{\frac{1}{2}}-
\left(\mathbf{M}_{N_1,N_2}+\gamma_{N_1,N_2}\mathbf{I}\right)^{\frac{1}{2}}\right)\left(\mathbf{\Sigma}+\gamma_{N_1,N_2}\mathbf{I}\right)^{-\frac{1}{2}}\\
=&\mathbf{\Sigma}_{N_1,N_2}^{-\frac{1}{2}}\left[\left(\mathbb{E}\left(\mathbf{M}_{N_1,N_2}+\gamma_{N_1,N_2}\mathbf{I}\right)\right)^{\frac{1}{2}}-
\left(\mathbf{M}_{N_1,N_2}+\gamma_{N_1,N_2}\mathbf{I}\right)^{\frac{1}{2}}\right]\left(\mathbb{E}\left(\mathbf{M}_{N_1,N_2}+\gamma_{N_1,N_2}\mathbf{I}\right)\right)^{\frac{1}{2}}\\
=&\mathbf{\Sigma}_{N_1,N_2}^{-\frac{1}{2}}\left[\left(\mathbb{E}\left(\mathbf{\Sigma}_{N_1,N_2}\right)\right)^{\frac{1}{2}}-\mathbf{\Sigma}_{N_1,N_2}^{\frac{1}{2}}\right]\left(\mathbb{E}\left(\mathbf{\Sigma}_{N_1,N_2}\right)\right)^{-\frac{1}{2}}\\
=&\mathbf{\Sigma}_{N_1,N_2}^{-\frac{1}{2}}\left[\left(\mathbb{E}\left(\mathbf{\Sigma}_{N_1,N_2}^{\frac{1}{2}}\right)\right)-\mathbf{\Sigma}_{N_1,N_2}^{\frac{1}{2}}\right]\left(\mathbb{E}\left(\mathbf{\Sigma}_{N_1,N_2}\right)\right)^{\frac{1}{2}}+
\mathbf{\Sigma}_{N_1,N_2}^{-\frac{1}{2}}\left[\left(\mathbb{E}\left(\mathbf{\Sigma}_{N_1,N_2}\right)\right)^{\frac{1}{2}}-\mathbb{E}\left(\mathbf{\Sigma}_{N_1,N_2}^{\frac{1}{2}}\right)\right]\left(\mathbb{E}\left(\mathbf{\Sigma}_{N_1,N_2}\right)\right)^{-\frac{1}{2}}
\end{align*}
And we have for $(2)$:
$$\left(2\right)=\left(\mathbf{\Sigma}+\gamma_{N_1,N_2}\mathbf{I}\right)^{-\frac{1}{2}}\left(\mathbf{\Sigma}^{\frac{1}{2}}-\left(\mathbf{\Sigma}+\gamma_{N_1,N_2}\mathbf{I}\right)^{\frac{1}{2}}\right)\mathbf{\Sigma}^{-\frac{1}{2}}$$
Thus we have:
{\small
\begin{align*}
\left\|\left(\mathbf{\Sigma}_{N_1,N_2}^{-\frac{1}{2}}-\mathbf{\Sigma}^{-\frac{1}{2}}\right)\mathbf{S}\right\|_{1} 
\leq& \;
\left\|\mathbf{\Sigma}_{N_1,N_2}^{-\frac{1}{2}}\left[\left(\mathbb{E}\left(\mathbf{\Sigma}_{N_1,N_2}^{\frac{1}{2}}\right)\right)-\mathbf{\Sigma}_{N_1,N_2}^{\frac{1}{2}}\right]\left(\mathbb{E}\left(\mathbf{\Sigma}_{N_1,N_2}\right)\right)^{-\frac{1}{2}}\mathbf{S}\right\|_{1} \\
&+\left\|\mathbf{\Sigma}_{N_1,N_2}^{-\frac{1}{2}}\left[\left(\mathbb{E}\left(\mathbf{\Sigma}_{N_1,N_2}\right)\right)^{\frac{1}{2}}-\mathbb{E}\left(\mathbf{\Sigma}_{N_1,N_2}^{\frac{1}{2}}\right)\right]\left(\mathbb{E}\left(\mathbf{\Sigma}_{N_1,N_2}\right)\right)^{-\frac{1}{2}}\mathbf{S}\right\|_{1}
\\
&+\left\|\left(\mathbf{\Sigma}+\gamma_{N_1,N_2}\mathbf{I}\right)^{-\frac{1}{2}}\left(\mathbf{\Sigma}^{\frac{1}{2}}-\left(\mathbf{\Sigma}+\gamma_{N_1,N_2}\mathbf{I}\right)^{\frac{1}{2}}\right)\mathbf{\Sigma}^{-\frac{1}{2}}\mathbf{S}\right\|_{1}
\end{align*}
}
But we know that $|\mathbf{\Sigma}_{N_1,N_2}^{-\frac{1}{2}}|_{i,j}\leq \frac{J}{\sqrt{\gamma_{N_1,N_2}}}$ and by the same reasoning we have also that $|\left(\mathbf{\Sigma}+\gamma_{N_1,N_2}\mathbf{I}\right)^{-\frac{1}{2}}_{i,j}|\leq \frac{J}{\sqrt{\gamma_{N_1,N_2}}}$.
By noting:
\begin{align*}
\mathbf{K}_1=& \displaystyle\sup_{k\in \llbracket 1,J \rrbracket} |[\mathbf{\Sigma}^{-\frac{1}{2}}\mathbf{S}]_k|\\
\mathbf{K}_2=& \displaystyle\sup_{k\in \llbracket 1,J \rrbracket}
\left|\left[\left(\mathbb{E}\left(\mathbf{\Sigma}_{N_1,N_2}\right)\right)^{\frac{1}{2}}-\mathbb{E}\left(\mathbf{\Sigma}_{N_1,N_2}^{\frac{1}{2}}\right)
\left(\mathbb{E}\left(\mathbf{\Sigma}_{N_1,N_2}\right)\right)^{-\frac{1}{2}}\mathbf{S}\right]_k\right|\\
\mathbf{K}_3=& \displaystyle\sup_{k\in \llbracket 1,J \rrbracket}
|[\left(\mathbb{E}\left(\mathbf{\Sigma}_{N_1,N_2}\right)\right)^{-\frac{1}{2}}\mathbf{S}]_k|
\end{align*}
All these constants are independent from $N_1,N_2,\left(x_i\right)$ and $\left(y_j\right)$. We have finally:
\begin{align*}
\left\|(\mathbf{\Sigma}_{N_1,N_2}^{-\frac{1}{2}}-\mathbf{\Sigma}^{-\frac{1}{2}})\mathbf{S}\right\|_{1}
&\leq
\sum_{j=1}^J\sum_{q=1}^J\sum_{k=1}^J\left|\left(\mathbb{E}\left(\mathbf{\Sigma}_{N_1,N_2}^{\frac{1}{2}}\right)\right)_{q,k}
-\left(\mathbf{\Sigma}_{N_1,N_2}^{\frac{1}{2}}\right)_{q,k}\right|\frac{\mathbf{K}_3J}{\sqrt{\gamma_{N_1,N_2}}}+J^4
\mathbf{K}_1+\frac{J^3 \mathbf{K}_2}{\sqrt{\gamma_{N_1,N_2}}}\\
&\leq
\left[\sum\limits_{q,k=1}^J\left|\left(\mathbb{E}\left(\mathbf{\Sigma}_{N_1,N_2}^{\frac{1}{2}}\right)\right)_{q,k}
-\left(\mathbf{\Sigma}_{N_1,N_2}^{\frac{1}{2}}\right)_{q,k}\right|\right]\frac{\mathbf{K}_3J^2}{\sqrt{\gamma_{N_1,N_2}}}+J^4
\mathbf{K}_1+\frac{J^3 \mathbf{K}_2}{\sqrt{\gamma_{N_1,N_2}}}
\end{align*}
And by applying a union bound on all the terms that compose the upper bound of $|\widehat{\lambda} _{N_1,N_2}-\lambda_{t}|$ we have thus:
\begin{align*}
\mathbb{P}\left(\left|\widehat{\lambda}_{t}-\lambda_{t}\right|\leq \alpha\right)&\geq
\sum_{k=1}^J\mathbb{P}\left(\sqrt{t}\frac{J^2}{\sqrt{\gamma_{N_1,N_2}}}|\mu_X\left(T_k\right)-\mu_Y\left(T_k\right)-
\mathbb{E}\left(\mu_X\left(T_k\right)-\mu_Y\left(T_k\right)\right)|\leq
\frac{\alpha}{J+J^2}\right) \\
&+\sum\limits_{q,k=1}^J\mathbb{P}\left(\sqrt{t}\left(\left(\left|\left(\mathbb{E}\left(\mathbf{\Sigma}_{N_1,N_2}^{\frac{1}{2}}\right)\right)_{q,k}-
\left(\mathbf{\Sigma}_{N_1,N_2}^{\frac{1}{2}}\right)_{q,k}\right|\right)\frac{\mathbf{K}_3J^2}{\sqrt{\gamma_{N_1,N_2}}}+J^4
\mathbf{K}_1+
\frac{J^3 \mathbf{K}_2}{\sqrt{\gamma_{N_1,N_2}}}\right)\leq \frac{\alpha}{J^2+J}\right)\\
&-\left(J^2+J-1\right)
\end{align*}
As $\mu_X\left(T\right)-\mu_Y\left(T\right)=\sum_{k=1}^{t}Z_i$ where $Z_k$ are independent and:
\begin{itemize}
\item $\forall i\leq N_1, Z_i=\frac{k\left(x_i,T\right)}{N_1}$, so $|Z_i|\leq \frac{K}{N_1}$
\item $\forall N_1<i\leq N_2, Z_i=-\frac{k\left(y_i,T\right)}{N_1}$ so $|Z_i|\leq \frac{K}{N_2}$
\end{itemize}
We have thanks to Hoeffding's inequality that $\forall k \in \llbracket 1,J \rrbracket:$
\begin{align*}
\mathbb{P}\left(\sqrt{t}\frac{J^2}{\sqrt{\gamma_{N_1,N_2}}}\left|\mu_X\left(T_k\right)-\mu_Y\left(T_k\right)-
\mathbb{E}\left(\mu_X\left(T_k\right)-\mu_Y\left(T_k\right)\right)\right|\leq
\frac{\alpha}{J+J^2}\right)\\
\geq 1-2\exp\left(-\left(\frac{\alpha}{J^2+J}\right)^2\frac{\gamma_{N_1,N_2}N_1N_2}{K^2\left(N_1+N_2\right)^2}\right)
\end{align*}
Moreover $\forall k,q \in \llbracket 1,J \rrbracket:$
\begin{multline*}
\mathbb{P}\left[\sqrt{t}\left(\left|\left(\mathbb{E}\left(\mathbf{\Sigma}_{N_1,N_2}^{\frac{1}{2}}\right)\right)_{q,k}-
\left(\mathbf{\Sigma}_{N_1,N_2}^{\frac{1}{2}}\right)_{q,k}\right|\frac{\mathbf{K}_3J^2}{\sqrt{\gamma_{N_1,N_2}}}
+\frac{J^4}{\mathbf{K}_1}+\frac{J^3K_2}{\sqrt{\gamma_{N_1,N_2}}}\right)\leq
\frac{\alpha}{J^2+J}\right] \\
=\mathbb{P}\left[\left|\left(\mathbf{\Sigma}_{N_1,N_2}^{\frac{1}{2}}\right)_{k,q}-\mathbb{E}\left(\mathbf{\Sigma}_{N_1,N_2}^{\frac{1}{2}}\right)_{k,q}\right|
\leq
\frac{\gamma_{N_1,N_2}}{\mathbf{K}_3J^2}\left[\frac{\alpha}{\left(J^2+J\right)\sqrt{t}}-\left(\frac{J^3
\mathbf{K}_2}{\sqrt{\gamma_{N_1,N_2}}}+J^4 \mathbf{K}_1\right)\right]\right]
\end{multline*}
Let define
$F\left(x_1,...,x_{N_1},y_1,...,y_{N_2}\right):=\mathbf{\Sigma}_{N_1,N_2}$
and
$F_{k,q}\left(x_1,...,x_{N_1},y_1,...,y_{N_2}\right):=\left(\mathbf{\Sigma}_{N_1,N_2}\right)_{k,q}$
We can see easily that $\forall \left(x_i\right),\left(y_i\right), x,x',y,y'$:
\begin{equation*}
\biggl|F_{k,q}(x_1,..,x,..,x_{N_1},y_1,..,y_{N_2})-
F_{k,q}(x_1,..,x',..,x_{N_1},y_1,..,y_{N_2})\biggr|\leq \frac{8}{\rho N_ 1}
\end{equation*}
and
\begin{equation*}
\biggl|F_{k,q}(x_1,..,x_{N_1},y_1,..y,..,y_{N_2})-
F_{k,q}(x_1,..,x_{N_1},y_1,..,y',..,y_{N_2})\biggr|\leq \frac{8}{\left(1-\rho\right) N_ 2}
\end{equation*}
Let $g\left(\mathbf{X}\right)=\mathbf{X}^{\frac{1}{2}}$ defined on $\mathbf{S}_J^{++}\left(R\right)$ and takes values in  $\mathbf{S}_J^{++}\left(R\right)$.
This fuction is well defined because each matrix of $\mathbf{S}_J^{++}\left(R\right)$ admits a unique square root matrix on $\mathbf{S}_J^{++}\left(R\right)$. Moreover The result hold on $\mathbf{S}_J^{+}\left(R\right)$.
\begin{lemma}
$g$ is locally Lipschitz continuous on  $\mathcal{S}_J^{++}\left(R\right)$ which means that:
\begin{equation*}
\forall N>0,\;\forall \mathbf{X}, \mathbf{Y} \in
B\left(0,N\right)\subset\mathcal{S}_J^{++}\left(\mathbb{R}\right),
\quad
\exists K_N
/\left\|g\left(\mathbf{X}\right)-g\left(\mathbf{Y}\right)\right\|\leq K_N
\|\mathbf{X}-\mathbf{Y}\|
\end{equation*}
\end{lemma}
\begin{prv}
Let us first prove that $g$ is $C^{\infty}$. First thanks to Lemma \ref{lemma:semi-inv} $g$ is continuous on $\mathbf{S}_J^{++}\left(R\right)$. Let us show now that $g$ is $C^{\infty}$ on this space. We know that $\mathbf{M}\rightarrow \mathbf{M}^2$ induces a bijection
from $\mathcal{S}_n^{++}\left(R\right)$ on itself where the inverse is g.
To prove then that g is $C^{\infty}$, thanks to the inverse function
theorem, we just have to show that
$\mathbf{D}_{\mathbf{U}_0}\left(\mathbf{M}\rightarrow
\mathbf{M}^2\right)$ is invertible for every $\mathbf{U}_0 \in \mathcal{S}_n^{++}\left(R\right)$. Let $\mathbf{U}_0 \in \mathcal{S}_n^{++}\left(R\right)$.
And let's consider the differential defined on
$\mathcal{S}_n\left(R\right)$ in $\mathcal{S}_n\left(R\right)$ which is a
linear application and which associates $\mathbf{H}$ to $\mathbf{U}_0
\mathbf{H}+ \mathbf{H} \mathbf{U}_0$. If we prove the injectivity of this function we will have its invertibility as $\mathcal{S}_n\left(R\right)$ is a finite dimensional space. Let $\mathbf{H}\in\mathcal{S}_n\left(R\right)$ such that $\mathbf{U}_0
\mathbf{H}+ \mathbf{H} \mathbf{U}_0=0$ and $\mathbf{x}$ an eigenvector of
$\mathbf{U}_0$ associated with the eigenvalue $\lambda$ which is strictly
positive as $\mathbf{U}_0$ is definite positive. We have:
$$\mathbf{U}_0 \mathbf{H} \mathbf{x}=-\mathbf{H} \mathbf{U}_0 \mathbf{x}
= -\lambda \mathbf{H} \mathbf{x}$$
As $-\lambda<0$ it is not an eigenvalue of $\mathbf{U}_0$ and then
$\mathbf{H} \mathbf{x}=0$. This is true for all the eigenvectors of
$\mathbf{U}_0$, then $\mathbf{H}=0$ and the differential is injective, so
$g$ is $C^{\infty}$ on $\mathcal{S}_n^{++}\left(R\right)$. Finally by applying the Mean value theorem, we have that g is locally Lipschitz continuous.
\end{prv}
We also remark that $||F\left(x_i,y_j\right)||=\displaystyle\max_{i,j\in
\llbracket 1,J \rrbracket}|\left(\mathbf{\Sigma}_{N_1,N_2}\right)_{i,j}|\leq \lambda $ (because the
Gaussian kernel is bounded) with $\lambda$ independent from $N_1,N_2,\left(x_i\right)$ and $\left(y_j\right)$. Then by taking the following norm $\|\mathbf{M}\|=\displaystyle\max_{i,j\in \llbracket 1,J
\rrbracket} \mathbf{M}_{i,j}$ we have:
\begin{multline*}
\biggl\|g\left(F\left(x_1,..,x,..,x_{N_1},y_1,..,y_{N_2}\right)\right)-g\left(F\left(x_1,..,x',..,x_{N_1},y_1,..,y_{N_2}\right)\right)\biggr\|\\
\leq
K_{\lambda}\biggl\|F\left(x_1,..,x,..,x_{N_1},y_1,..,y_{N_2}\right)-F\left(x_1,..,x',..,x_{N_1},y_1,..,y_{N_2}\right)\biggr\|
\end{multline*}
And:
\begin{equation*}
\biggl\|F(x_1,..,x,..,x_{N_1},y_1,..,y_{N_2})-
F(x_1,..,x',..,x_{N_1},y_1,..,y_{N_2})\biggr\|
\leq \max\left(\frac{8}{\rho N_ 1} ,\frac{8}{\left(1-\rho\right) N_ 2}\right)
\end{equation*}
Then $\forall k,q\in \llbracket 1,J \rrbracket:$
$$\biggl|\mathbf{\Sigma}_{N_1,N_2}^{\frac{1}{2}}\left(x\right)-\mathbf{\Sigma}_{N_1,N_2}^{\frac{1}{2}}(x')
\biggr|\leq  K_{\lambda} \max\left(\frac{8}{\rho N_ 1} ,\frac{8}{\left(1-\rho\right) N_ 2}\right)$$
And thanks to the McDiarmid inequality we have:
\begin{multline*}
\mathbb{P}\left(\left|\left(\mathbf{\Sigma}_{N_1,N_2}^{\frac{1}{2}}\right)_{k,q}-\mathbb{E}\left(\mathbf{\Sigma}_{N_1,N_2}^{\frac{1}{2}}\right)_{k,q}\right|
\leq
\frac{\gamma_{N_1,N_2}}{K_3J^2}\left[\frac{\alpha}{\left(J^2+J\right)\sqrt{t}}-\left(\frac{J^3K_2}{\sqrt{\gamma_{N_1,N_2}}}+J^4K_1\right)\right]\right)\\
\geq 1
-2\exp\left(-2\frac{\left(\frac{\gamma_{N_1,N_2}}{K_3J^2}(\frac{\alpha}{\left(J^2+J\right)\sqrt{t}}-\frac{J^3K_2}{\sqrt{\gamma_{N_1,N_2}}}-J^4K_1)\right)^2}{K_{\lambda}^2
\left(N_1 + N_2\right)\max\left(\frac{8}{\rho
N_1} ,\frac{8}{\left(1-\rho\right) N_ 2}\right)^2
}\right)
\end{multline*}
Then we have:
\begin{multline*}
\mathbb{P}\left(|\widehat{\lambda}_{N_1,N_2}-\lambda_{t}|\leq \alpha\right)\geq 1-
2\sum \limits_{k=1}^J\exp\left(-\left(\frac{\alpha}{J^2+J}\right)^2\frac{\gamma_{N_1,N_2}N_1N_2}{\left(N_1+N_2\right)^2}\right)\\
-2\sum \limits_{k,q=1}^J \exp\left(-2\frac{\left(\frac{\gamma_{N_1,N_2}}{K_3J^2}(\frac{\alpha}{\left(J^2+J\right)\sqrt{t}}-\frac{J^3K_2}{\sqrt{\gamma_{N_1,N_2}}}-J^4K_1)\right)^2}{K_{\lambda}^2
\left(N_1 + N_2\right)\max\left(\frac{8}{\rho
N_1} ,\frac{8}{\left(1-\rho\right) N_ 2}\right)^2
}\right)
\end{multline*}
And finally, by taking $\alpha=\lambda_t-\delta$ we have the result.
\end{prv}

\section{Using smooth characteristic functions (SCF)}
\label{sec:randscf}
\begin{theorem}
Let $k$ be an analytic, integrable kernel with an inverse Fourier transform strictly greater than zero. For any $J > 0$, we define:
\begin{equation*}
d_{\Phi,J}=\left\{d_{\Phi,J}\left[p,q\right]=\frac{1}{J}\sum_{j=1}^{J}|\Phi_p\left(T_j\right)-\Phi_q\left(T_j\right)|
:\quad
p,q\in\mathcal{M}_{+}^{1}\left(\mathbb{R}^d\right),\Phi_p,\Phi_q\in
L_1\left(\mathbb{R}^d\right)\right\}
\end{equation*}
\medbreak
Then for any $J > 0$, $d_{\Phi,J}$ is a random metric on the space of Borel probability measures with integrable characteristic functions.
\end{theorem}

\begin{prv}
Since $k$ is an analytic, integrable kernel with an inverse Fourier transform strictly greater then zero then by the Lemma \ref{lem:smooth-rkhs} the mapping $\Lambda : P\rightarrow \Phi_P$ is injective and $\Lambda\left(P\right)$ is an element of the RKHS associated with $k$. The Lemma \ref{lemma:analytic} shows that $\Phi_P$ is analytic. Therefore we can use Lemma \ref{lemma:2} to see that $d_{\Lambda,J}\left(P, Q\right)= d_{\Phi,J}\left(P,Q\right)$ is a random metric. This concludes the proof of the Theorem.
\end{prv}

\subsection{Proof of Proposition \ref{prop:finalement-SCF}}
\label{sec:finalement-SCF}
\begin{propo}
Let $\alpha\in ]0,1[$, $\gamma>0$ and $J\geq 2$. Let $\{T_j\}_{j=1}^J$
sampled i.i.d. from the distribution $\Gamma$ and let $X:=\{x_i\}_{i=1}^{n}$ and 
$Y:=\{y_i\}_{i=1}^{n}$ i.i.d. samples from $P$ and $Q$ respectively. Let us denote $\delta$ the $(1-\alpha)$-quantile of the asymptotic null distribution of $\widehat{d}_{\ell_1,\Phi,J}[X,Y]$ and $\beta$ the $(1-\alpha)$-quantile of the asymptotic null distribution of $\widehat{d}_{\ell_2,\Phi,J}^2[X,Y]$. Under the alternative hypothesis, almost surely, there exists $N\geq 1$ such that for all $n\geq N$, with a probability of at least $1-\gamma$ we have:
\begin{align}
\widehat{d}^2_{\ell_2,\Phi,J}[X,Y] >\beta \Rightarrow \widehat{d}_{\ell_1,\Phi,J}[X,Y]>\delta 
\end{align}
\end{propo}

\begin{prv}
Let us first introduce the following Lemma:
\begin{lemma}
\label{lem:better-power-scf}
Let $\mathbf{x}$ a random vector $\in \mathbb{C}^J$ with $J\geq 2$,  $\epsilon>0$, $\gamma>0$ and $\mathbf{z}:= \min\limits_{j\in[|1,J|]}|\text{Re}(x_j)|+|\text{Im}(x_j)|$  where $\text{Im}$ and  $\text{Re}$ are respectively the imaginary and real part functions. Moreover let denote $\mathbf{X}:=(\text{Im}(x_j),\text{Re}(x_j))_{j=1}^J\in\mathbb{R}^{2J}$. If
\begin{align*}
\mathbb{P}(\mathbf{z}\geq\epsilon)\geq 1-\gamma
\end{align*}
we have with a probability of
at least $1-\gamma$ that, $\forall t_1\geq t_2\geq 0$,
if $\epsilon\geq \sqrt{\frac{t_1^2-t_2^2}{J(J-1)}}$, then
$$\Vert \mathbf{X}\Vert_2 >t_2\Rightarrow \Vert \mathbf{X}\Vert_1\geq t_1.$$
\end{lemma}

\begin{prv}
First we remarks that:
\begin{align*}
\epsilon>\sqrt{\frac{t_1^2-t_2^2}{J\left(J-1\right)}} \Rightarrow & J\left(J-1\right)\epsilon>t_1^2-t_2^2\\
\Rightarrow & t_2^2>t_1^2-J\left(J-1\right)\epsilon^2
\end{align*}
Therefore, we have:
\begin{align*}
\|\mathbf{X}\|_2\geq t_2 \Rightarrow & \|\mathbf{X}\|_{2}^2+J\left(J-1\right)\epsilon^2\geq t_1^2\\
\Rightarrow & \sqrt{\|\mathbf{X}\|_{2}^2+J\left(J-1\right)\epsilon^2}\geq t_1
\end{align*}
But we have that:
$$\|\mathbf{X}\|_{1}^2=\|\mathbf{X}\|_{2}^2+\sum_{i\neq
j}\left(|\text{Im}(x_i)|+|\text{Re}(x_i)|\right)\left(|\text{Im}(x_j)|+|\text{Re}(x_j|\right)$$
Therefore we have with a probability of 1-$\gamma$ that:
$$\|\mathbf{X}\|_{1}^2\geq \|\mathbf{X}\|_{2}^2+J\left(J-1\right)\epsilon^2$$
And:
$$\|\mathbf{X}\|_2\geq t_2\Rightarrow \|\mathbf{X}\|_1\geq t_1$$
\end{prv}
Moreover by denoting $\delta$ the $(1-\alpha)$-quantile of the asymptotic null distribution of $\widehat{d}_{\ell_1,\Phi,J}[X,Y]$ and $\beta$ the $(1-\alpha)$-quantile of the asymptotic null distribution of $\widehat{d}_{\ell_2,\Phi,J}^2[X,Y]$ thanks to Lemma \ref{lemma:quantile}, we have that $\delta\geq\sqrt{\beta}$.  Therefore to show the result we only need to show that the assumption of the Lemma \ref{lem:better-power-scf} is sastified for the random vector $\mathbf{X}:=\sqrt{n}\mathbf{S}_n\in\mathbb{R}^{2J}$, $t_1=\delta$ and $t_2=\sqrt{\beta}$, i.e. for $\epsilon=\sqrt{\frac{\delta^2-\beta}{J(J-1)}}$ under the alternative hypothesis. Under $H_1: P\neq Q$, we have $\mathbf{S}_{n}$ converges in probability to the vector
$\mathbf{S}$ where $\mathbf{\Sigma}:=\mathbb{E}_{(x,y)\sim
(p,q)}(\Sigma_n)$
and  $\mathbf{S}:=\mathbb{E}_{(x,y)\sim (p,q)}(\mathbf{S}_{n})$. Moreover we have $\mathbf{S}=(\text{Im}(\Phi_P(T_j)-\Phi_Q(T_j)),\text{Re}(\Phi_P(T_j)-\Phi_Q(T_j)))_{j=1}^{J}\in\mathbb{R}^{2J}$. Indeed, according to the Definition \ref{def:smooth}, we have for all $j\in[|1,J|]$:
\begin{align*}
\phi_P(T_j)&:=\int_{\epsilon\in\mathbb{R}^d}\psi_P(\epsilon)k(T_j-\epsilon)d\epsilon\\
&=\int_{\epsilon\in\mathbb{R}^d}\int_{x\in\mathbb{R}^d}\exp({ix^{T}\epsilon})k(T_j-\epsilon) dP(x)d\epsilon\\
&=\int_{x\in\mathbb{R}^d} \left(\int_{\epsilon\in\mathbb{R}^d} \exp({ix^{T}(\epsilon-T_j)})k(\epsilon-T_j)\right)\exp({ix^{T}T_j})d\epsilon dP(x)\\
&=\int_{x\in\mathbb{R}^d}f(x)\exp({ix^{T}T_j})dP(x)
\end{align*}
and all these equalities hold as $k$ is integrable. Lemma \ref{lem:smooth-rkhs} guarantees the injectivity of the function $\Gamma:P\rightarrow \Phi_P$, and as $P\neq Q$, therefore $\Phi_P-\Phi_Q$ is a non-zero function.  Moreover $\Phi_P$ and $\Phi_Q$ live in the RKHS $H_k$ associated with $k$. Therefore thanks to Lemma \ref{lemma:analytic}, $\Phi_P-\Phi_Q$ is analytic.
Therefore thanks to Lemma \ref{lemma:1}, $\Phi_P-\Phi_Q$ is almost surely non zero. Moreover the $(T_j)_{j=1}^J$ are independent, therefore almost surely $(|\Phi_P(T_j)-\Phi_Q(T_j)|)_j$ are all non zero, and then
$\left(\left|\text{Im}(\Phi_P(T_j)-\Phi_Q(T_j)))\right|+\left|\text{Re}(\Phi_P(T_j)-\Phi_Q(T_j))\right|\right)_j$ are all non zero. Then by continuity of the functions defined for all $k\in [|1,J|]$ by:
\begin{align}
\phi_k:x:=(x_j^1,x_j^2)_{j=1}^J\in\mathbb{R}^{2J}\rightarrow |x_k^1|+|x_k^1|
\end{align}
We have that for all $k\in[|1,J|]$, $\phi_k(\mathbf{S}_n)$ converge in probability towards $\phi_k(\mathbf{S})$, which are almost surely all non zeros. Then for all $k\in\llbracket 1,J \rrbracket$ we have:
\begin{equation*}
\mathbb{P}\left(
\bigg|(\sqrt{n}\phi_k(\mathbf{S}_n)\bigg|>\epsilon\right)
=\mathbb{P}_{X,Y}\left(
\bigg|(\phi_k(\mathbf{S}_n)\bigg|-\frac{\epsilon}{\sqrt{n}}>0\right)
\end{equation*}
And as $\frac{\epsilon}{\sqrt{n}}\rightarrow 0$ as $n\rightarrow \infty$, we
have finally almost surely for all $k\in\llbracket 1,J \rrbracket$:
\begin{align*}
\mathbb{P}_{X,Y}\left(\bigg|(\sqrt{n}\phi_k(\mathbf{S}_n)\bigg|\geq \epsilon\right)\rightarrow 1 \text{  as  }n\rightarrow \infty
\end{align*}
Therefore almost surely there exist $N\geq 1$ such that for all $n \geq N$ and for all $k\in\llbracket 1,J \rrbracket$:
\begin{align*}
\mathbb{P}_{X,Y}\left(\bigg|(\sqrt{n}\phi_k(\mathbf{S}_n)\bigg|\geq \epsilon\right)\geq 1-\frac{\gamma}{J}
\end{align*}
Finally by applying a union bound we obtain that almost surely, for all $n \geq N$:
\begin{align*}
\mathbb{P}_{X,Y}\left(\forall k\in[|1,J|]\text{,  }\bigg|(\sqrt{n}\phi_k(\mathbf{S}_n)\bigg|\geq \epsilon\right)\geq 1-\gamma
\end{align*}
Therefore by applying Lemma \ref{lem:betterpow}, we obtain that, almost surely, for all $n \geq N$, with a probability of at least $1-\gamma$:
\begin{align*}
\widehat{d}_{\ell_2,\Phi,J}[X,Y] >\sqrt{\beta} \Rightarrow \widehat{d}_{\ell_1,\Phi,J}[X,Y]>\delta 
\end{align*}

\end{prv}

\section{Experiments}

\subsection{Realization of the $\ell_1$-based tests} 
\label{sec: cdf}
Indeed to realize these tests, we need to compute the $1-\alpha$ quantile of the $Nake\left(\frac{1}{2},1,J\right)$. To do so we need to obtain the cumulative distribution function ( CDF ) of the sum of $J$ Nakagami i.i.d. But as we do not have a closed form of this distribution, we need to estimate this CDF by considering the empirical distribution function. Indeed to generate samples from $Nake\left(\frac{1}{2},1,J\right)$, it is sufficient to generate samples from multivariate normal distribution $\mathcal{N}\left(0,I_{d_{J}}\right)$, and to sum the absolute values of the $J$ coordinates of theses vectors.

Moreover, we have the following result:
\begin{theorem}{(\textbf{Dvoretzky–Kiefer–Wolfowitz inequality})}
Let $x_1$,...,$x_n$ be real-valued independent and identically distributed random variables with cumulative distribution function $F\left(.\right)$
Let $F_n$ denote the associated empirical distribution function defined by:
$$F_n\left(t\right)=\frac{1}{n}\sum_{i=1}^n\mathbf{1}_{x_i\leq t}$$
Then we have $\forall \epsilon>0$:
$$\mathbb{P}\left(||F_n-F||_{\infty}>\epsilon\right)\leq 2e^{-2n\epsilon^{2}}$$
\end{theorem}
Finally we have, $F\left(x\right)-\epsilon\leq F_n\left(x\right)\leq F\left(x\right)+\epsilon$ with a probability of $1-\delta$ where $\epsilon=\sqrt{\frac{ln\left(\frac{2}{\delta}\right)}{2n}}$.

Then with a probability of $99\%$, and by taking $n=100\,000$ samples i.i.d of the $Naka\left(\frac{1}{2},1,J\right)$, we can estimate the CDF with an error of $\epsilon\leq 0.0051$,
which is less than $\alpha=0.01$.

\textbf{Optimization:} The lower bounds that we optimize to perform
$\textbf{L1-opt-ME}$ and $\textbf{L1-opt-SCF}$ are non-convex, as in the
prior art \cite{test}. However, the use of the $\ell_1$-norm makes optimization even harder, as it is no longer a smooth. Moreover we need to differentiate through the inverse square root matrix operation which can lead is some cases to degenerate matrices during the gradient ascent. Therefore to avoid this, we decide to check at each step the convergence of the inverse square root matrix operation. Further work should consider dedicated optimization algorithms.

Table \ref{table-time} gives the run times of the different optimized tests on the Blobs problem when the test sample size is $n^{te}=1e6$.

\begin{table}[!h]
\centering\small%
\begin{tabular}{c |  c c c c}
 & \textbf{L1-opt-ME} & \textbf{ME-full} & \textbf{L1-opt-SCF} & \textbf{SCF-full}\\ 
\hline
 Run Time (s) & 164.23 & 157.97 & 599.77 & 579.42\\ 
    \end{tabular}
\caption{Run times of the optimized tests when $n^{te}=1e6$ and $J=2$ for the blobs problem.\label{table-time}}
\end{table}

\textbf{Software implementation}: as the expression of the optimization
objective is rather complicated, we use the automatic differentiation of
pytorch \cite{paszke2017automatic}, to compute its gradient, and then
proceed with a gradient ascent where the step size after $t$ iterations is the inverse of the euclidean  norm of the gradient times $\sqrt{t}$. The specific code can be found at
\url{https://github.com/meyerscetbon/l1_two_sample_test}.

\subsection{Experimental verification of the Propostion \ref{prop:finalement}}
To show the validity of the Proposition \ref{prop:finalement}
experimentally, we examine the behavior of the unormalized $\ell_2$ and
$\ell_1$ based tests respectively defined in eq. \ref{eq:ME-stat} and
\ref{l_1:comparison}. In Figure \ref{fig:unor-nsample}, we compare the
unormalized tests on the GVD problem where we increase the test sample
size with $d=100$ and $J=2$. Here the locations are chosen at random and
are sampled from a standard normal distribution. Moreover here
$\alpha=0.05$. Compared to the normalized tests studied in
\autoref{sec:experimental_study} is that here we no longer have a direct access to the quantiles of the asymptotic null distribution. Being a problem where we can generated the data ourselves, we have therefore estimate the quantiles of our interest. Moreover, when comparing the $\ell_2$ and $\ell_1$ tests, we sample at random the locations and we evaluate the two statistics at the same locations.
We see that as the test sample size increases, the $\ell_1$-based tests rejects better the null distribution.

\begin{figure}[!h]
\begin{minipage}{0.4\textwidth}
\caption{Plot of type-II error against the test sample size $n^{te}$ in the GVD toy problem: $P=\mathcal{N}\left(0,I_d\right)$ and 
$Q=\mathcal{N}\left(0,\text{diag}(2,1,...,1)\right)$ with $d=100$\label{fig:unor-nsample}.}
\end{minipage}%
\hfill%
\begin{minipage}{0.55\textwidth}
\includegraphics[width=\textwidth]{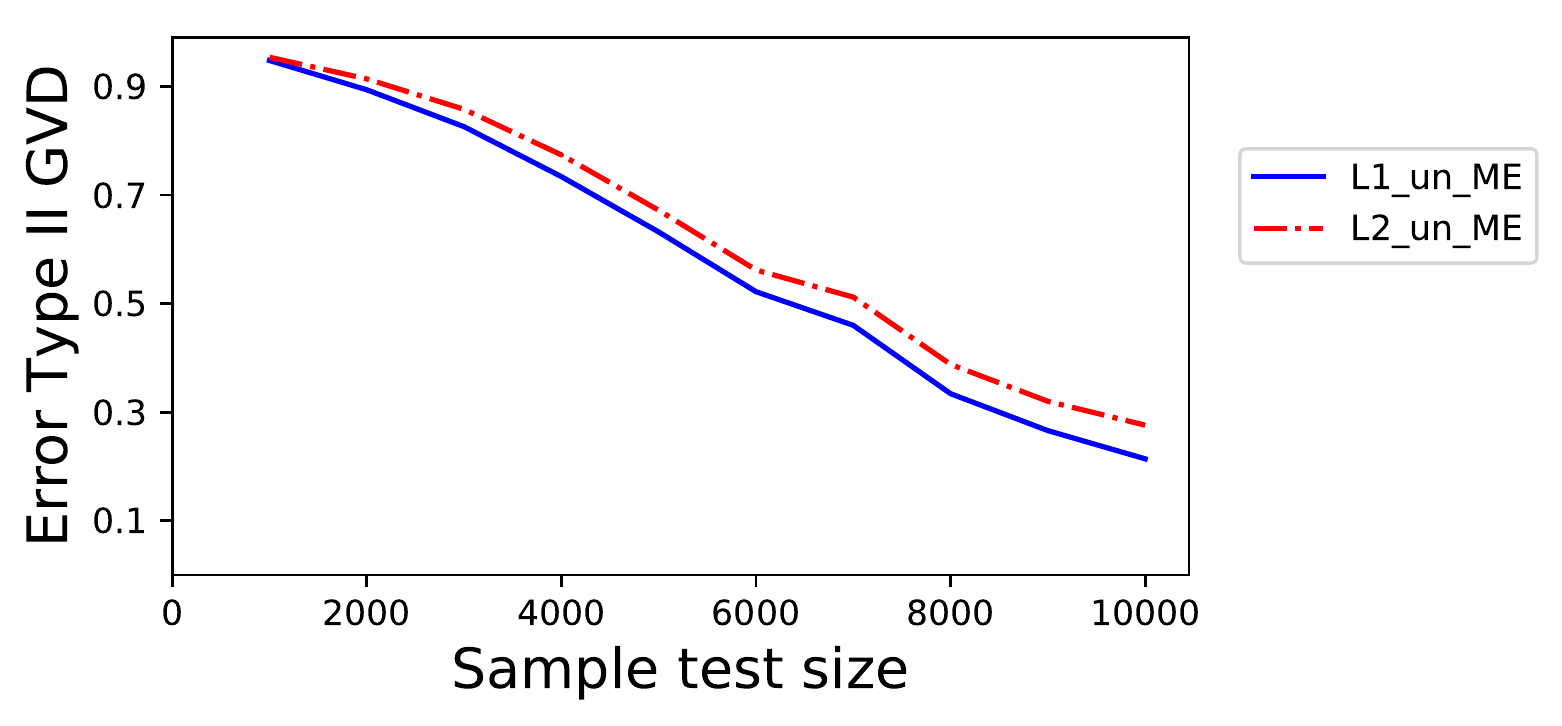}
\end{minipage}%
\end{figure}

\subsection{Experiments on a more difficult problem}
\label{sec:GMD}
\begin{figure}[!h]
\begin{minipage}{0.45\textwidth}
\caption{Plot of type-II error against the test sample size $n^{te}$ in the following toy problem: $P=\mathcal{N}\left(0,I_d\right)$ and 
$Q=\mathcal{N}\left(\left(0.3,0,..,0\right)^{T},I_d\right)$ with $d=100$\label{fig:GMD}}
\end{minipage}%
\hfill%
\begin{minipage}{0.5\textwidth}
\includegraphics[width=\textwidth]{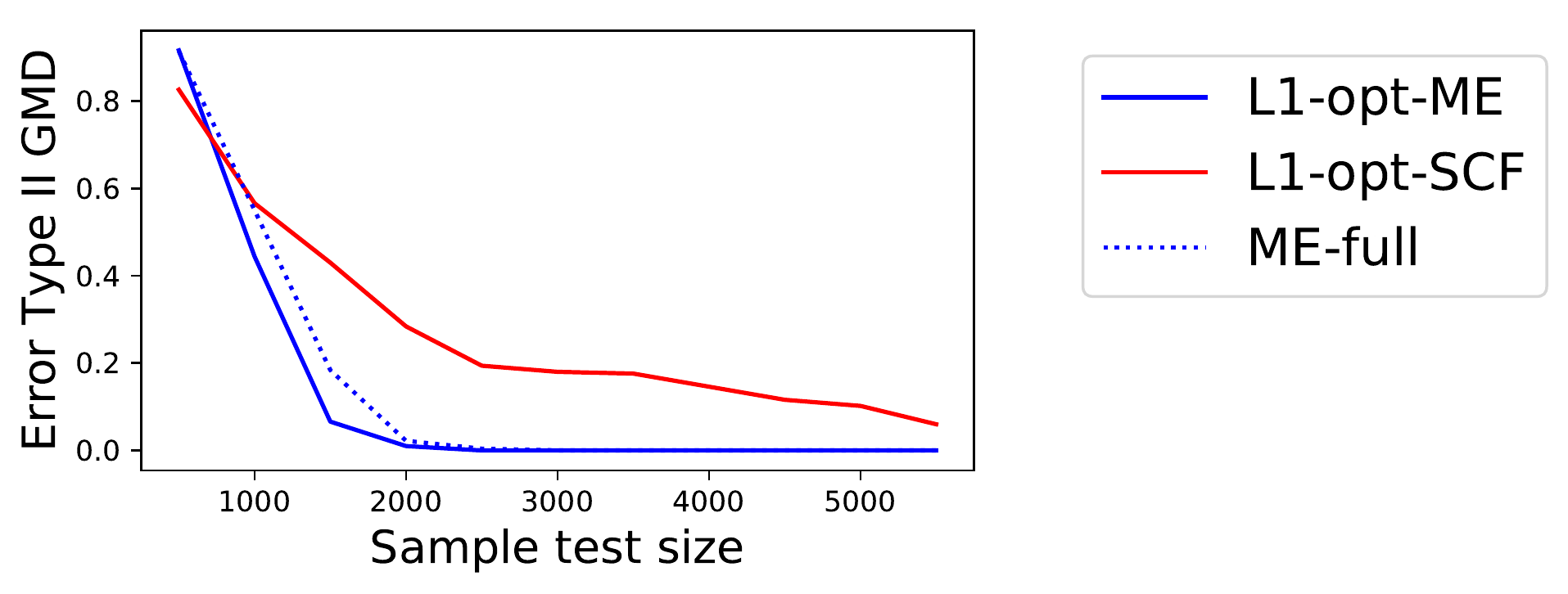}
\end{minipage}%
\end{figure}

In Figure \ref{fig:GMD}, we consider the following GMD problem: $P\sim\mathcal{N}\left(0,I_d\right)$, 
$Q\sim\mathcal{N}\left(\left(0.3,0,..,0\right)^{T},I_d\right)$ with $d=100$.
The figure shows that when the problem of GMD is more difficult, we can see that \textbf{L1-opt-ME} performs the best. 

\subsection{Informative features}
\label{sec:info}
We show that the optimization of the proxy
$\widehat{\lambda}_{t}^{tr}\left(\theta\right)$ for the test power in the
$\ell_1$ case is informative for revealing the difference of the two
samples in the ME test as in \cite{test} with the $\ell_2$ version.
We consider the Gaussian Mean Difference (GMD) problem (see Table \ref{tab:synthetic_pb}), where both $P$ and $Q$ are two-dimensional normal distributions with different means. We use $J = 2$ test locations $T_1$ and $T_2$ , where $T_1$ is fixed to the location indicated by the black triangle in Figure \ref{info}. The contour plot shows $T_2 \rightarrow \widehat{\lambda}_{t}^{tr}\left(T_1,T_2\right)$.

\begin{figure}[!h]
    \begin{minipage}{.265\linewidth}
    \caption{\textbf{Illustrating interpretable features}, replicating
    in the $\ell_1$ case the figure of \cite{test}. A contour plot of $ \widehat{\lambda}_{t}^{tr}\left(T_1,T_2\right)$ as a function of $T_2$,
    when $J=2$,
    and $T_1$ is fixed. The red and black dots represent the samples from
    the $P$ and $Q$ distributions, and the big black triangle the position of
    $T_1$.
    \label{info}}
    \end{minipage}%
    \hfill%
    \begin{subfigure}{0.35\linewidth}
	\hspace*{-.09\linewidth}%
        \includegraphics[width=1.2\linewidth]{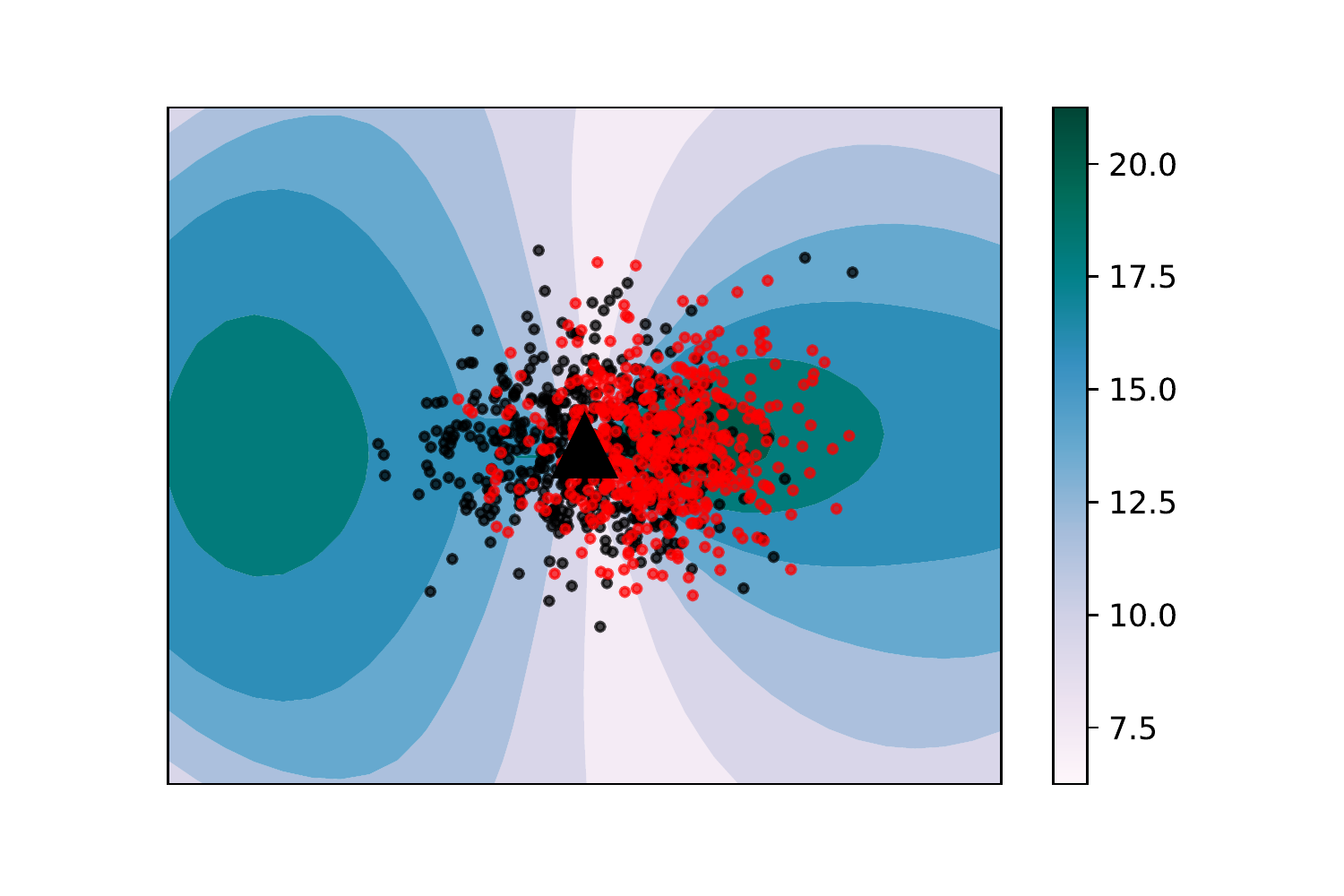}
        \caption{$T_1$ is centered\label{infoa}}
    \end{subfigure}%
    \begin{subfigure}{0.35\linewidth}
	\hspace*{-.04\linewidth}%
        \includegraphics[width=1.2\linewidth]{informative_4.pdf}
        \caption{$T_1$ lives in the left region\label{infob}}
    \end{subfigure}
\end{figure}

Figure \ref{infoa} suggests that $\widehat{\lambda}_{t}^{tr}\left(T_1,T_2\right)$ is maximized when $T_2$ is placed in either of the two regions that
captures the difference of the two samples i.e., the region in which the probability masses of $P$ and
$Q$ have less overlap. In Figure \ref{infob}, we consider placing $T_1$ in one of the two key regions. In this case, the contour plot shows that $T_2$ should be placed in the other region to maximize $\widehat{\lambda}_{t}^{tr}\left(T_1,T_2\right)$, implying that placing multiple test locations in the same neighborhood does not increase the discriminability.
The two modes on the left and right suggest two ways to place the test location in a region that
reveals the difference. The non-convexity of the $\widehat{\lambda}_{t}^{tr}\left(T_1,T_2\right)$ is an indication of many informative ways to 
detect differences of $P$ and $Q$, rather than a drawback. A convex objective would not capture this multimodality.

\subsection{Real problem: 20 newsgroups text dataset}
\label{sec:20group}

\begin{table}[!t]
\center
\begin{tabular}{lccccccccc}
P  & Q  &   L1-opt-ME & L1-grid-ME & L1-opt-SCF & L1-grid-SCF & ME-full&SCF-full\\ 
\hline
sci (1187) & sci (1187) & \textbf{0.00} & \textbf{0.00} & 0.004 &
\textbf{0.00} & 1 &0.002\\
sci (1187) & comp (292) & \textbf{0.00} & 0.496 & \textbf{0.00} & 0.170&
\textbf{0.00} & 0.634 \\
sci (1187) & alt (240) & \textbf{0.00} & 0.370 & \textbf{0.00} & 0.064 &
\textbf{0.00} & 0.510 
\end{tabular}

\caption{Type-I errors and Type-II errors of various the L1-tests in the problem of distinguishing the newsgroups text dataset. $\alpha= 0.01$. $J = 2$. The number in brackets denotes the test sample size of each samples.\label{tab:text}}
\label{Tab:text}
\end{table}

In this experiment we use the 20 newsgroups text dataset from \cite{Lang95} which comprises around 18000 newsgroups posts on 20 topics. We consider 3
categories which are: "comp", "sci", and "alt"
. The first category is about components in hardware systems, the second is about sciences and spaces, 
and the last is about religion. To perform the tests we need to embed
these documents in a metric space. For this, we use the TF-IDF matrix by
group of two categories with a $df\geq 30$, which lead to embed the
documents in  spaces of 3\,000 dimensions approximately. Then we perform
the two-sample tests on the embedded documents.
We compare the distribution of "sci" documents  with
others, as well as with itself to evaluate the level of the tests.
The number of samples of each category is not the same, hence to perform the tests from \cite{test}, we take
randomly $n_{\text{min}}$ samples for both distributions without
replacement (where $n_{\text{min}}$ is in fact the number of samples of
the distributions compared to the sci distribution). We set the number of location $J=2$.

Type-I errors and type-II errors are summarized in Table \ref{Tab:text} The two first columns indicates the categories of the papers in the two samples. This task represents a case in which $H_0$ holds. In this case all the tests are conservative except the \textbf{ME-full} test which totally rejecting the null hypothesis.
In the other problems, we show the Type-II errors of our tests.
The $\ell_1$ optimized tests perform very well, which shows that the
locations learned are indeed discriminant. The $\ell_1$ approaches bring
a clear gain in statistical control compared to their $\ell_2$
counterparts.

\subsection{Real problem: fast-food distribution}

\begin{table*}[!ht]
\centering\small
\begin{tabular}{lccccccc}
Problem &  \hspace*{-3.5ex}L1-opt-ME\hspace*{-.5ex} &
\hspace*{-.5ex}L1-grid-ME\hspace*{-.5ex}  &
\hspace*{-.5ex}L1-opt-SCF\hspace*{-.5ex} &
\hspace*{-.5ex}L1-grid-SCF\hspace*{-.5ex} &
\hspace*{-.5ex}ME-full\hspace*{-.5ex} &
\hspace*{-.5ex}SCF-full\hspace*{-.5ex} &
\hspace*{-.5ex}MMD-quad\hspace*{-1ex} \\
\hline
McDo vs McDo (2002)    &   0.010 &    0.000 &    0.000 &     0.000 &    0.012 &     0.000 &     0.000 \\
\end{tabular}
\caption{\textbf{Fast food dataset:} Type-I errors for
distinguishing the distribution of fast food restaurants. $\alpha= 0.01$.
$J = 3$. The number in brackets denotes the sample size of the
distribution on the right. We consider MMD-quad as the gold standard.}
\label{Tab:exp3result-level}
\end{table*}
Table \ref{Tab:exp3result-level} summarizes Type-I errors observed on the Mac Donald's
vs Mac Donald's problem. It shows that the optimized tests based on mean
embeddings stay roughly at the specified level $\alpha=0.01$ when $H_0$
hold, and others are more conservative.

\begin{figure}[b!]
\begin{minipage}{.3\linewidth}%
\caption{\textbf{Fast food data:} Visualizing interpretable locations for
    differences in Mc Donald's vs
    Burger King and Mc Donald's vs Wendy's. The lines correspond to the
    distribution of the locations chosen for the $T_J$ features by the
    L1-opt-ME procedure. The distributions are estimated with a kernel
    density estimate. The lines represent 
    the contours probabilities 80\% and 90\%.
\label{fig:usa}}
\end{minipage}%
\begin{minipage}{.7\linewidth}%
   \includegraphics[width=\linewidth]{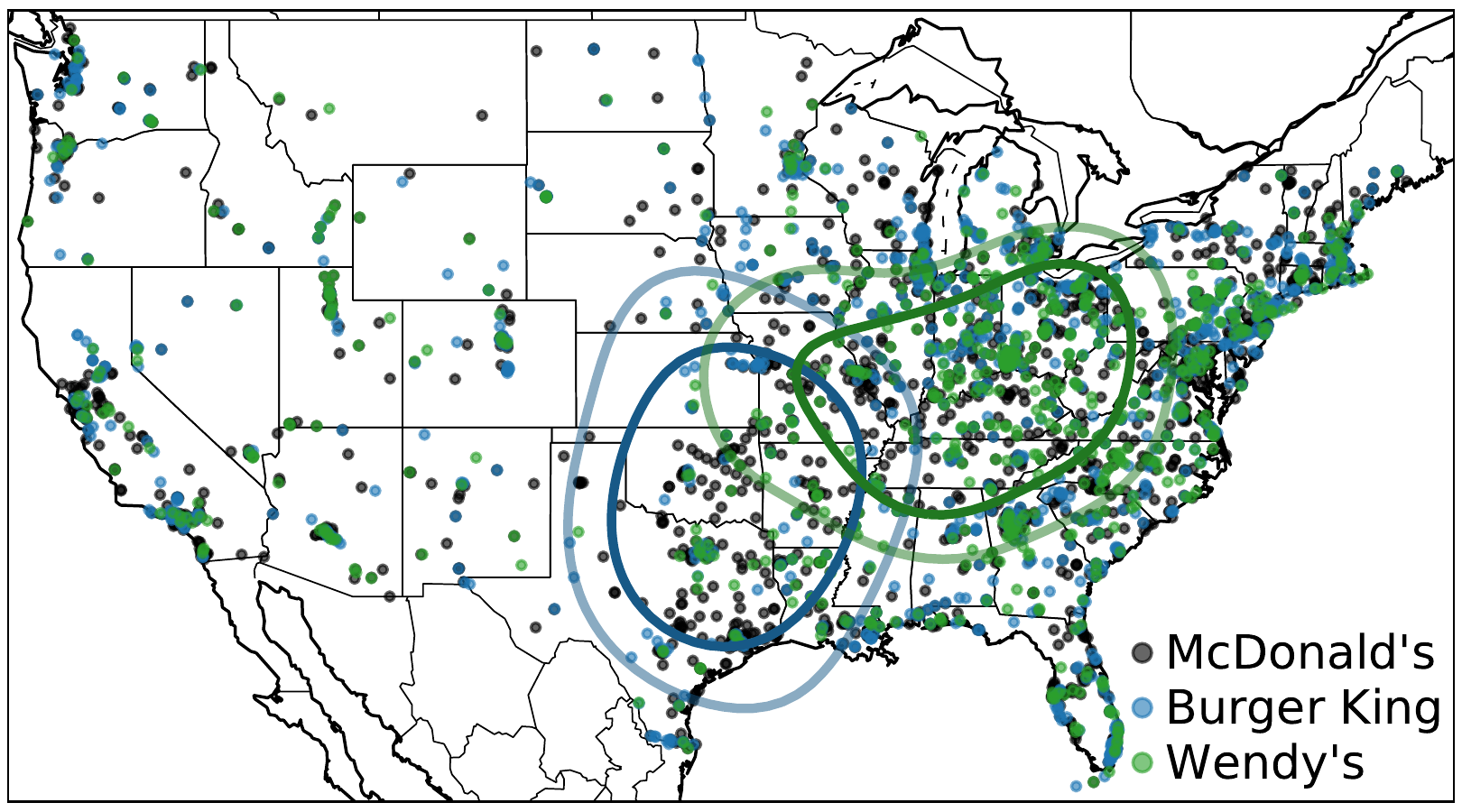}%
\end{minipage}
\end{figure}

\begin{figure}
\begin{minipage}{.3\linewidth}
    \caption{\textbf{Mc Donald's vs Burger King}\label{fig:burger_king}}
\end{minipage}
\hfill%
\begin{minipage}{.6\linewidth}
    \includegraphics[width=\linewidth]{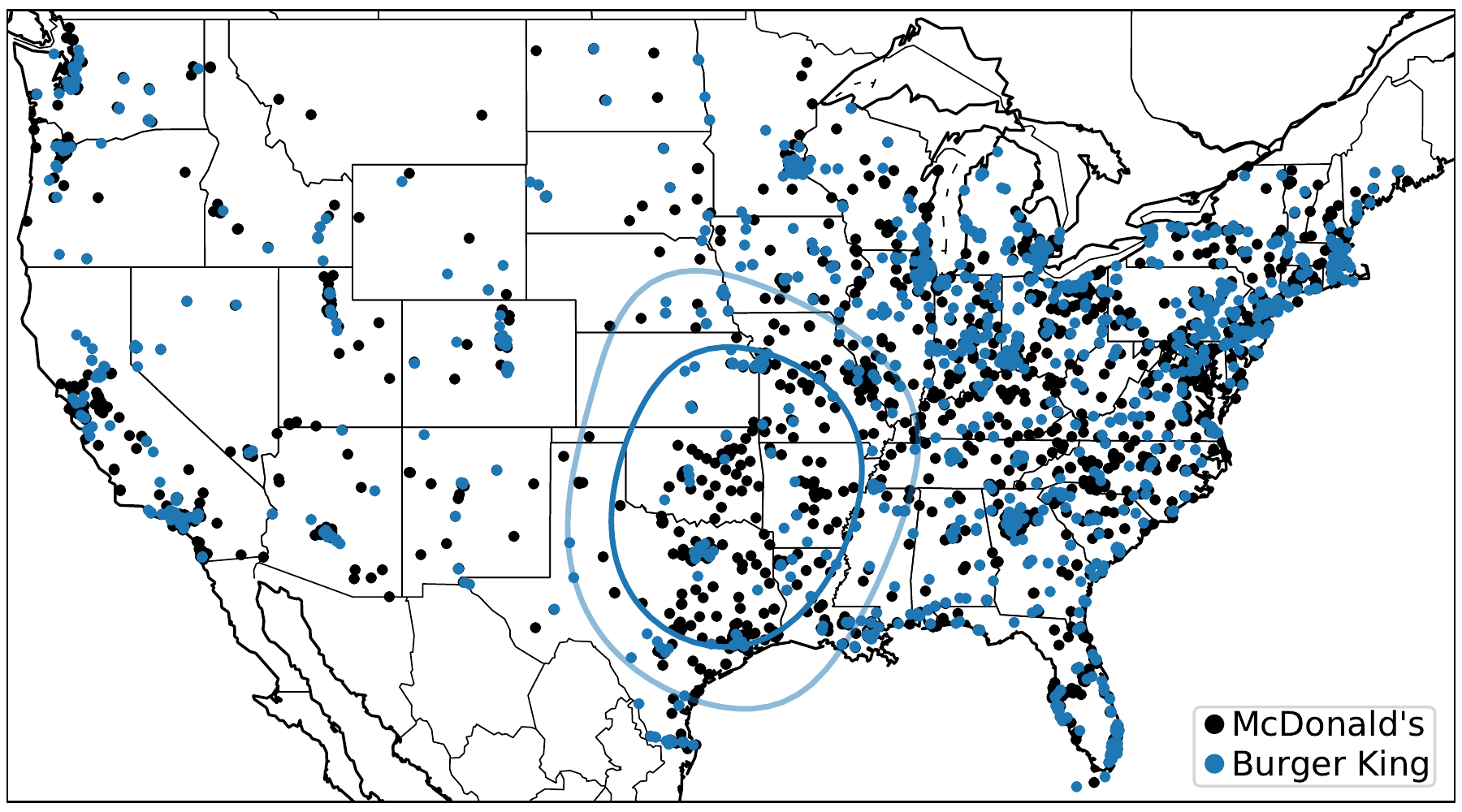}
\end{minipage}
\end{figure}

\begin{figure}
\begin{minipage}{.3\linewidth}
    \caption{\textbf{Mc Donald's vs Taco Bell}\label{fig:taco_bell}}
\end{minipage}
\hfill%
\begin{minipage}{.6\linewidth}
    \includegraphics[width=\linewidth]{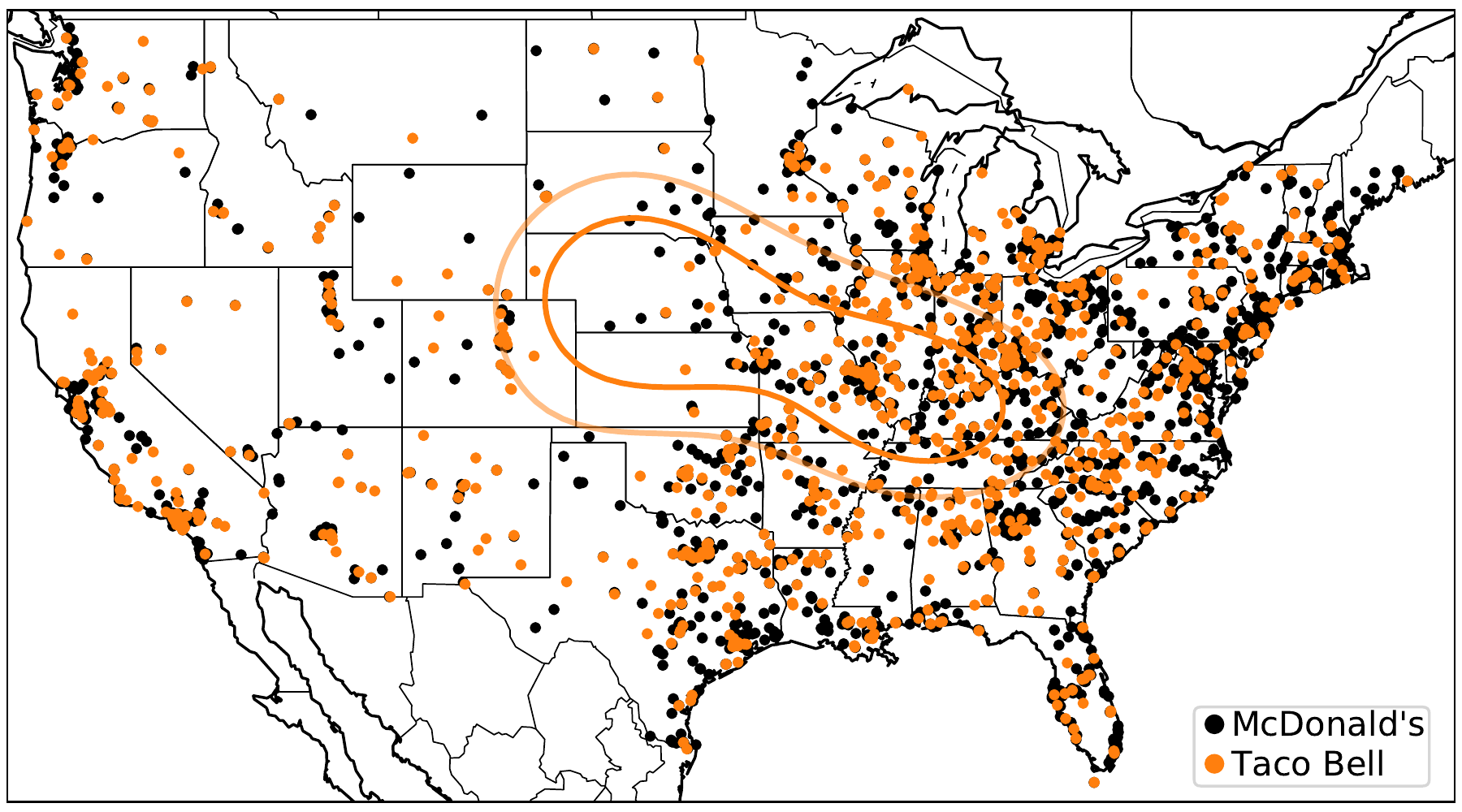}
\end{minipage}
\end{figure}

Figures \ref{fig:burger_king}, 
\ref{fig:taco_bell}, \ref{fig:wendys}, \ref{fig:arbys}, \ref{fig:kfc} give the
distributions of the data (restaurant locations) and of the $T_J$ for
each of the problems that we consider.

\begin{figure}[p]
\begin{minipage}{.3\linewidth}
    \caption{\textbf{Mc Donald's vs Wendy's}\label{fig:wendys}}
\end{minipage}
\hfill%
\begin{minipage}{.6\linewidth}
    \includegraphics[width=\linewidth]{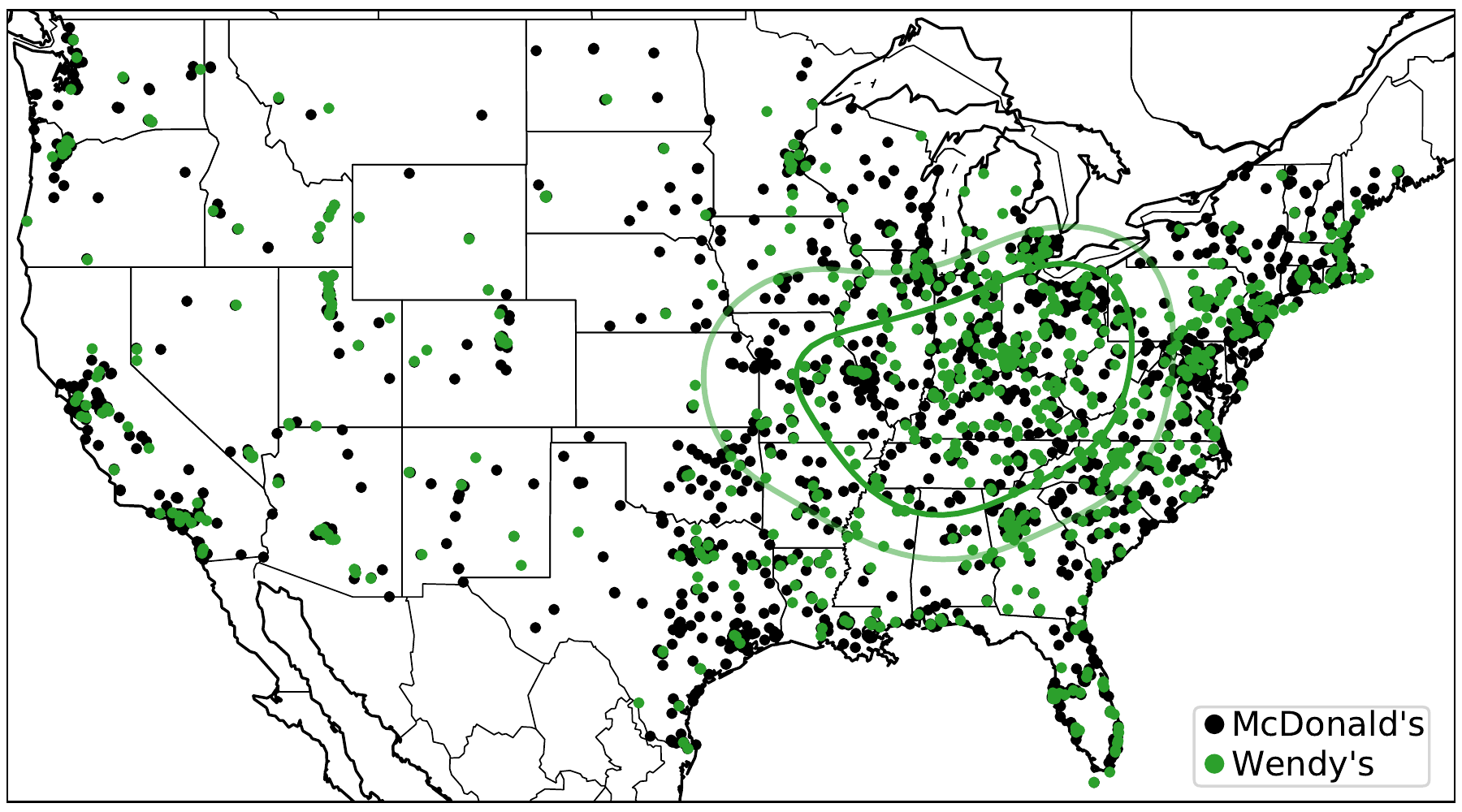}
\end{minipage}
\end{figure}

\begin{figure}[p]
\begin{minipage}{.3\linewidth}
    \caption{\textbf{Mc Donald's vs Arby's}\label{fig:arbys}}
\end{minipage}
\hfill%
\begin{minipage}{.6\linewidth}
    \includegraphics[width=\linewidth]{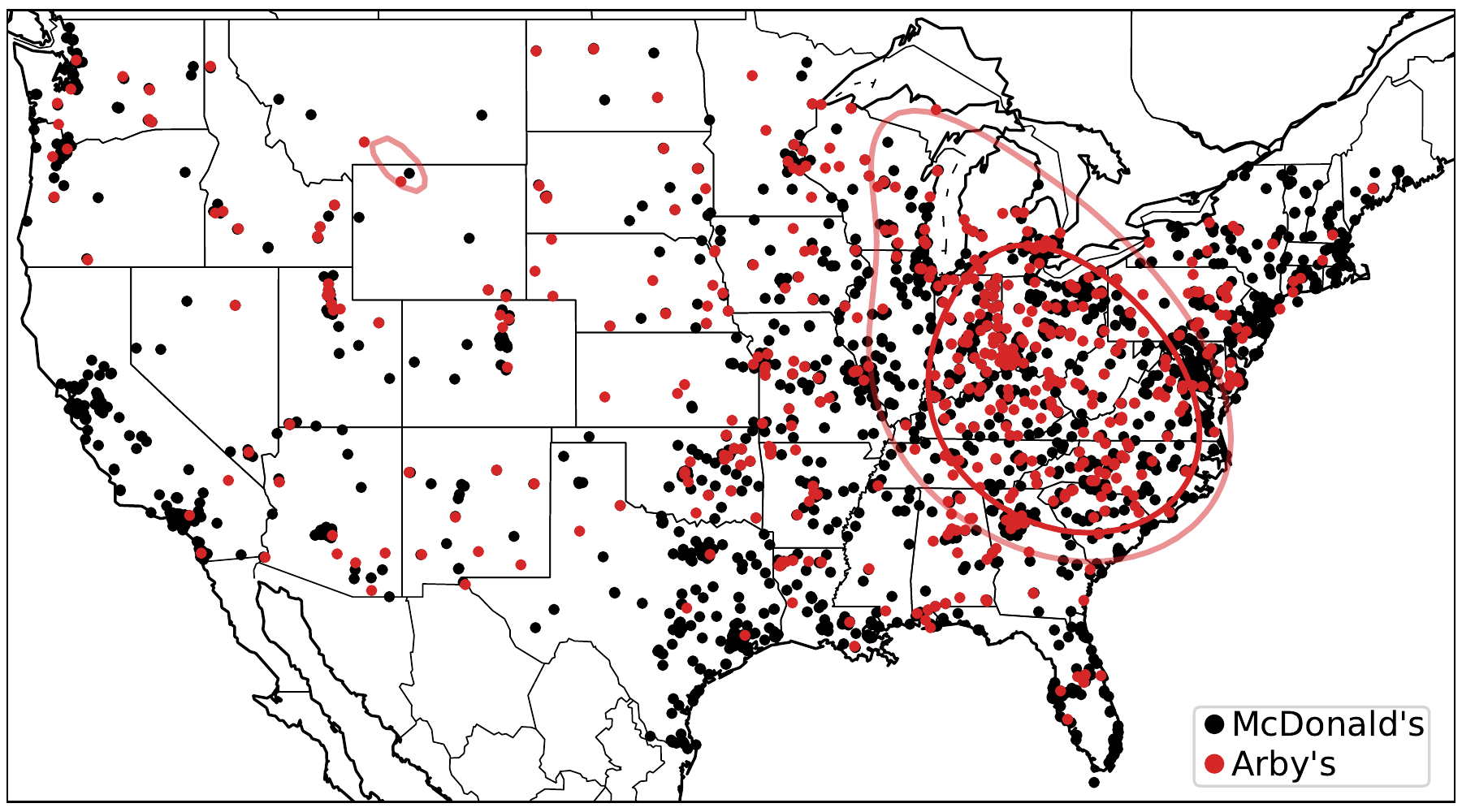}
\end{minipage}
\end{figure}

\begin{figure}[p]
\begin{minipage}{.3\linewidth}
    \caption{\textbf{Mc Donald's vs KFC}\label{fig:kfc}}
\end{minipage}
\hfill%
\begin{minipage}{.6\linewidth}
    \includegraphics[width=\linewidth]{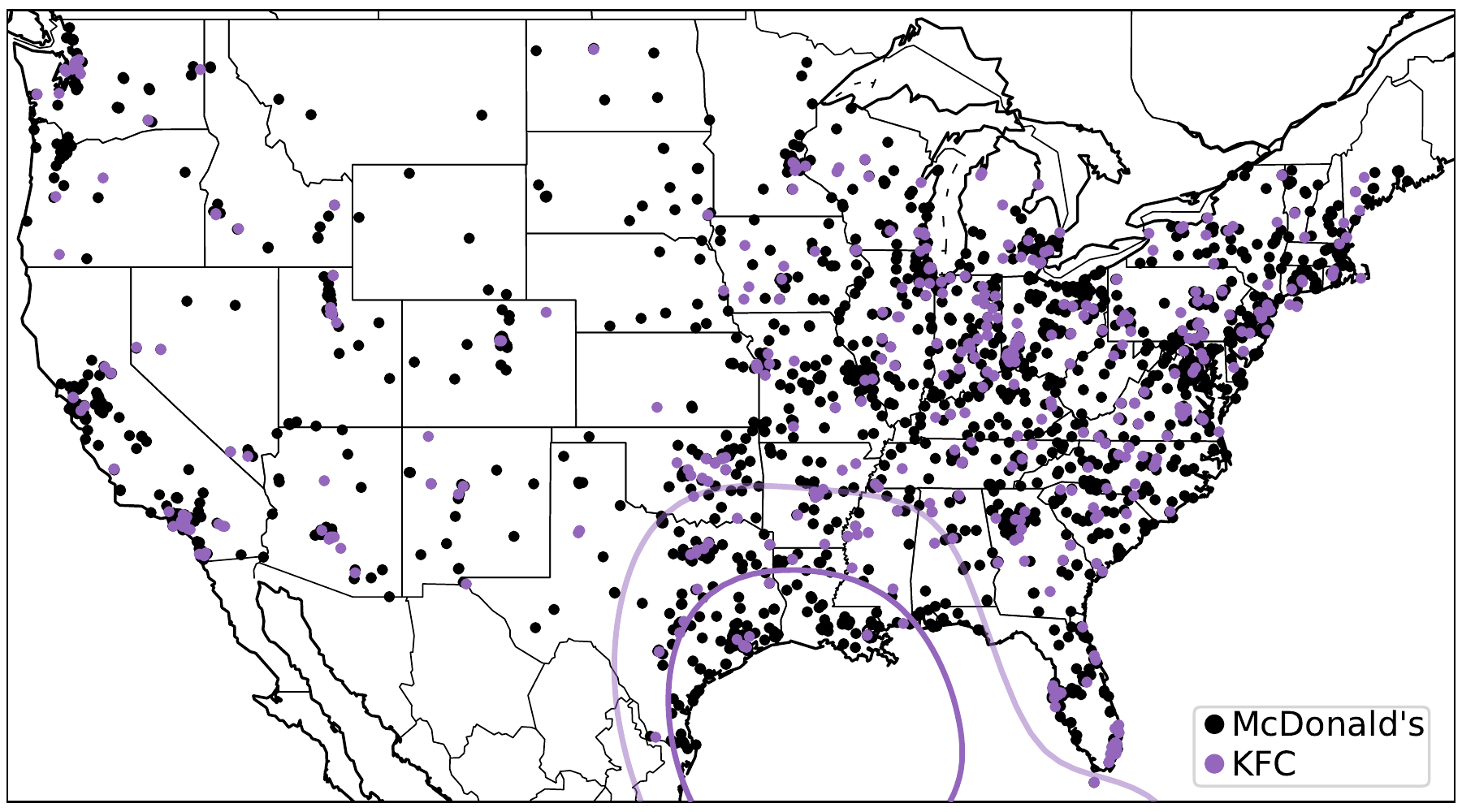}
\end{minipage}
\end{figure}

}{}

\end{document}